\documentclass[letterpaper,10pt]{article}

\usepackage[margin=1in]{geometry}
\usepackage[round,compress]{natbib}

\usepackage{parskip}

\usepackage{amsmath}
\usepackage{amsxtra,amsfonts,amscd,amssymb,bm,bbm,mathrsfs}
\usepackage{nicefrac}       %
\usepackage{amsthm,thmtools,thm-restate}
\usepackage{mathtools}
\usepackage{mathabx}
\newcommand\labelAndRemember[2]
  {\expandafter\gdef\csname labeled:#1\endcsname{#2}%
   \label{#1}#2}
\newcommand\recallLabel[1]
   {\csname labeled:#1\endcsname\tag{\ref{#1}}}
\newcommand\labelr[2]
  {\expandafter\gdef\csname labeled:#1\endcsname{#2}%
   \label{#1}#2}
\newcommand\recall[1]
   {\csname labeled:#1\endcsname}

\usepackage{colortbl}

\usepackage[algo2e,linesnumbered,vlined,ruled]{algorithm2e}
\usepackage{algorithm}
\usepackage[noend]{algorithmic}
\renewcommand{\algorithmicrequire}{\textbf{Input:}}
\renewcommand{\algorithmicensure}{\textbf{Output:}}

\usepackage{hyperref,url}
\hypersetup{
  colorlinks,
  citecolor=blue!70!black,
  linkcolor=blue!70!black,
  urlcolor=blue!70!black}
\usepackage[capitalise,noabbrev]{cleveref}
\crefformat{equation}{#2(#1)#3}
\Crefformat{equation}{#2(#1)#3}

\usepackage[titletoc,title]{appendix}
\usepackage{cancel}
\usepackage{enumerate}
\usepackage{enumitem}
\usepackage{comment}

\usepackage{booktabs}       %
\usepackage{multirow}
\usepackage{multicol}
\usepackage{makecell}
\usepackage{array}
\newcolumntype{H}{>{\setbox0=\hbox\bgroup}c<{\egroup}@{}}
\newcolumntype{Z}{>{\setbox0=\hbox\bgroup}c<{\egroup}@{\hspace*{-\tabcolsep}}}
\usepackage{color}
\usepackage{colortbl}
\definecolor{light-gray}{gray}{0.9}
\usepackage{hhline}
\usepackage{tablefootnote}
\usepackage{makecell}

\usepackage[normalem]{ulem}
\usepackage{exscale}
\usepackage{tikz}
\usepackage{float}
\usepackage{graphicx,graphics}
\graphicspath{ {./fig/} }
\allowdisplaybreaks

\usepackage{apptools}

\newtheorem{theorem}{Theorem}
\AtAppendix{\counterwithin{theorem}{section}}
\newtheorem{lemma}[theorem]{Lemma}
\AtAppendix{\counterwithin{lemma}{section}}
\newtheorem{corollary}[theorem]{Corollary}
\AtAppendix{\counterwithin{corollary}{section}}
\newtheorem{proposition}[theorem]{Proposition}
\AtAppendix{\counterwithin{proposition}{section}}

\newtheorem{definition}[theorem]{Definition}
\AtAppendix{\counterwithin{definition}{section}}

\theoremstyle{definition}
\newtheorem{remark}[theorem]{Remark}
\AtAppendix{\counterwithin{remark}{section}}

\AtAppendix{\counterwithin{example}{section}}

\renewcommand{\hat}{\widehat}
\renewcommand{\tilde}{\widetilde}
\renewcommand{\epsilon}{\varepsilon}

\usepackage{pifont}

\newcommand{\id}{I}

\newcommand{\nS}{|\mathcal{S}|}
\newcommand{\nA}{|\mathcal{A}|}

\newcommand{\eps}{\varepsilon}

\newcommand{\II}{\mathbb{I}}
\newcommand{\EE}{\mathbb{E}}
\newcommand{\PP}{\mathbb{P}}
\newcommand{\QQ}{\mathbb{Q}}

\newcounter{cnt}
\foreach \num in {1,2,...,26}{%
  \stepcounter{cnt}%
  \expandafter\xdef \csname c\Alph{cnt}\endcsname {\noexpand\mathcal{\Alph{cnt}}}%
  \expandafter\xdef \csname b\Alph{cnt}\endcsname {\noexpand\mathbb{\Alph{cnt}}}%
}

\newcommand{\spa}{\mathrm{span}}

\DeclareMathOperator*{\argmax}{arg\,max}

\newcommand{\diag}{\operatorname{diag}}

\newcommand{\geqsim}{\gtrsim}

\newcommand{\T}{\top}  
\newcommand{\iprod}[2]{\left\langle #1, #2 \right\rangle}
\newcommand{\nrm}[1]{\left\|#1\right\|}

\newcommand{\abs}[1]{\left|#1\right|}

\newcommand{\cond}[2]{\mathbb{E}\left[\left.#1\right|#2\right]}

\newcommand{\bigO}[1]{\mathcal{O}\!\left(#1\right)}
\newcommand{\tbO}[1]{\tilde{\mathcal{O}}\!\left(#1\right)}

\newcommand{\Om}[1]{\Omega\!\left(#1\right)}
\newcommand{\ThO}[1]{\Theta\!\left(#1\right)}
\newcommand{\ceil}[1]{\left\lceil #1\right\rceil}
\newcommand{\floor}[1]{\left\lfloor #1\right\rfloor}
\DeclarePairedDelimiterX{\ddiv}[2]{(}{)}{%
  #1\;\delimsize\|\;#2%
}
\newcommand{\KL}{\operatorname{KL}\ddiv}
\newcommand{\chis}{\chi^2\ddiv}

\newcommand{\clog}{\iota}

\newcommand{\poly}{\operatorname{poly}}

\newcommand{\piout}{\pi^{\mathrm{out}}}

\newcommand{\norm}[1]{\left\|{#1}\right\|} %
\newcommand{\lone}[1]{\norm{#1}_1} %

\renewcommand{\cO}{\mathcal{O}}
\newcommand{\tO}{\widetilde{\cO}}
\newcommand{\wt}{\widetilde}

\newcommand{\otau}{\overline{\tau}}
\renewcommand{\II}{\mathbbm{1}}

\newcommand{\fA}{\mathfrak{A}}
\newcommand{\fB}{\mathfrak{B}}

\newcommand{\bd}{\mathbf{q}}
\newcommand{\unif}{\mathrm{Unif}}
\newcommand{\Unif}{\mathrm{Unif}}

\newcommand{\Bs}{\B^\star}

\newcommand{\bds}{\bd^\star}

\newcommand{\otheta}{{\bar{\theta}}}

\newcommand{\doac}{\mathrm{do}}

\newcommand{\BB}{\mathbf{B}}
\newcommand{\bq}{\mathbf{q}}
\newcommand{\bx}{\mathbf{x}}

\newcommand{\QAh}{\mathcal{U}_{A,h}}

\newcommand{\Uone}{\cU_1}
\newcommand{\Uh}{{\cU_h}}
\newcommand{\Uhp}{{\cU_{h+1}}}

\newcommand{\UAh}{\mathcal{U}_{A,h}}

\newcommand{\oUh}{\overline{\cU}_h}

\newcommand{\nUh}{\abs{\Uh}}

\newcommand{\nUA}{U_A}
\newcommand{\nUAh}{\abs{\UAh}}

\newcommand{\dec}{\operatorname{dec}}
\newcommand{\psc}{\operatorname{psc}}

\newcommand{\odec}{\overline{\operatorname{dec}}}

\newcommand{\dH}{D_{\mathrm{H}}}
\newcommand{\dTV}{D_{\mathrm{TV}}}

\renewcommand{\DH}[1]{D_{\mathrm{H}}^2\left(#1\right)}

\newcommand{\DTV}[1]{D_{\mathrm{TV}}\left(#1\right)}

\newcommand{\co}{\operatorname{co}}

\newcommand{\etod}{\textsc{E2D-TA}}

\newcommand{\mops}{\textsc{MOPS}}
\newcommand{\omle}{\textsc{OMLE}}

\newcommand{\eetod}{\textsc{Explorative E2D}}

\newcommand{\barpi}{\bar \pi}

\newcommand{\M}{\mathbb{M}}
\renewcommand{\O}{\mathbb{O}}
\renewcommand{\T}{\mathbb{T}}

\newcommand{\colspan}[1]{ {\rm colspan}(#1)}

\newcommand{\Test}{\mathfrak{T}}
\newcommand{\dPSR}{d_{\sf PSR}}

\newcommand{\logNt}{\log\cN_{\Theta}}
\newcommand{\Nt}{\cN_{\Theta}}

\newcommand{\nrmpi}[1]{\left\|#1\right\|_{\Pi}}

\newcommand{\tmu}{\tilde{\mu}}
\newcommand{\nrmop}[1]{\left\|#1\right\|_{*}}
\newcommand{\nrmst}[1]{\nrm{#1}_{*}}

\newcommand{\Vs}{V_{\star}}
\newcommand{\ths}{{\theta^{\star}}}

\newcommand{\Bpara}{B-representation}

\newcommand{\tauhm}{\tau_{h-1}}

\newcommand{\thp}{t_{h+1}}

\newcommand{\arev}{\alpha_{\sf rev}}

\newcommand{\nrmpip}[1]{\nrm{#1}_{\Pi'}}

\newcommand{\cBHh}{\cB_{H:h}}
\newcommand{\cBHhp}{\cB_{H:h+1}}

\newcommand{\efunc}{f}

\newcommand{\cEs}{\cE^{\star}}

\newcommand{\stab}{\Lambda_{\sf B}}

\newcommand{\dum}{\rm dum}

\newcommand{\rb}{R_{\sf B}}

\newcommand{\linv}{+}

\renewcommand{\sp}{s_\oplus}
\newcommand{\sq}{s_{\ominus}}

\renewcommand{\ss}{{\star}}
\newcommand{\oba}{{\sf good}}
\newcommand{\obb}{{\sf bad}}

\newcommand{\oN}{\overline{N}}

\newcommand{\as}{\a^{\star}}

\newcommand{\Ac}{\cA_c}

\newcommand{\acs}{a^\star}
\newcommand{\hs}{h^\star}

\newcommand{\Stree}{\cS_{\sf tree}}
\newcommand{\Sl}{\cS_{\sf leaf}}
\renewcommand{\ss}{s^\star}
\newcommand{\aleft}{{\sf left}}
\newcommand{\aright}{{\sf right}}
\newcommand{\await}{{\sf wait}}

\newcommand{\taut}{\tau^{(t)}}

\newcommand{\acrev}{{\sf reveal}}
\newcommand{\rev}{{\sf rev}}

\newcommand{\Arev}{\cA_{\rev}}

\newcommand{\Regret}{\mathrm{\mathbf{Regret}}}
\newcommand{\Reg}{\mathrm{\mathbf{Regret}}}

\newcommand{\corr}{{\sf correct}}

\newcommand{\otaut}{\overline{\tau}^{(t)}}
\newcommand{\oP}{\overline{\PP}}
\newcommand{\oE}{\overline{\EE}}
\newcommand{\oaug}{o^{\sf aug}}
\newcommand{\Ereach}{E_{\sf reach}}
\newcommand{\Erev}{E_{\sf rev}}
\newcommand{\Ecor}{E_{\sf correct}}
\newcommand{\odum}{{\sf dummy}}
\newcommand{\brcond}[2]{\left[\left.#1\right|#2\right]}

\newcommand{\supp}{\operatorname{supp}}
\renewcommand{\nrmop}[1]{\nrm{#1}_{*\to 1}}
\newcommand{\nrmstst}[1]{\nrm{#1}_{*\to *}}

\newcommand{\olock}{{\sf lock}}

\newcommand{\Mhm}{\M_{h,m}}
\newcommand{\Mhone}{\M_{h,1}}
\newcommand{\Mhmp}{\M_{h,m+1}}

\newcommand{\Apass}{\cA_{{\sf code}, \hs}}
\newcommand{\revs}{a_{\rev}^\star}
\newcommand{\Atr}{\cA_{\sf tr}}

\newcommand{\ep}{e_{\oplus}}
\newcommand{\eq}{e_{\ominus}}
\newcommand{\termin}{{\sf terminal}}

\newcommand{\nO}{|\cO|}

\newcommand{\loneone}[1]{\nrm{#1}_{1\to1}}

\newcommand{\IIc}[1]{{\II\paren{#1}}}
\newcommand{\Nc}[1]{N\paren{#1}}

\newcommand{\stoptime}{\mathsf{T}}
\newcommand{\cexit}{\mathsf{exit}}

\def\edec{{\rm edec}}

\renewcommand{\arev}{\alpha}
\newcommand{\arevm}{\alpha_{m}}
\newcommand{\arevmp}{\alpha_{m+1}}
\newcommand{\arevone}{\alpha_{1}}

\newcommand{\Emu}{\EE_{\mu\sim{\sf unif}}}
\newcommand{\Emumu}{\EE_{\mu,\mu'\sim{\sf unif}}}

\newcommand{\mup}{\mu_{\oplus}}
\newcommand{\muq}{\mu_{\ominus}}

\newcommand{\paren}[1]{{\left( #1 \right)}}
\newcommand{\brac}[1]{{\left[ #1 \right]}}
\newcommand{\set}[1]{{\left\{ #1 \right\}}}
\newcommand{\cdbrac}[2]{{\left[\left. #1 \right|#2\right]}}

\newcommand{\defeq}{\mathrel{\mathop:}=}
\newcommand{\vect}[1]{\ensuremath{\mathbf{#1}}}
\newcommand{\mat}[1]{\ensuremath{\mathbf{#1}}}

\newcommand{\rank}{\mathrm{rank}}

\newcommand{\E}{\mathbb{E}}

\renewcommand{\P}{\mathbb{P}}

\newcommand{\Z}{\mathbb{Z}}
\newcommand{\N}{\mathbb{N}}
\newcommand{\R}{\mathbb{R}}

\newcommand{\B}{\mat{B}}

\newcommand{\e}{\vect{e}}

\renewcommand{\a}{\vect{a}}
\renewcommand{\o}{\vect{o}}

\newcommand{\ctO}{\tilde{\mathcal O}}

\newenvironment{proof-sketch}{\noindent{\bf Proof Sketch}
  \hspace*{1em}}{\qed\bigskip\\}
\newenvironment{proof-idea}{\noindent{\bf Proof Idea}
  \hspace*{1em}}{\qed\bigskip\\}
\newenvironment{proof-of-lemma}[1][{}]{\noindent{\bf Proof of Lemma {#1}}
  \hspace*{1em}}{\qed\bigskip\\}
\newenvironment{proof-of-proposition}[1][{}]{\noindent{\bf
    Proof of Proposition {#1}}
  \hspace*{1em}}{\qed\bigskip\\}
\newenvironment{proof-of-theorem}[1][{}]{\noindent{\bf Proof of Theorem {#1}}
  \hspace*{1em}}{\qed\bigskip\\}
\newenvironment{inner-proof}{\noindent{\bf Proof}\hspace{1em}}{
  $\bigtriangledown$\medskip\\}
\newenvironment{proof-attempt}{\noindent{\bf Proof Attempt}
  \hspace*{1em}}{\qed\bigskip\\}

\newenvironment{proofof}[1][{}]{\noindent{\bf Proof of \cref{#1}}
  \hspace*{1em}}{\qed\bigskip\\}

\newenvironment{proof-of}[1][{}]{\noindent{\bf #1}
  \hspace*{1em}}{\qed\bigskip\\}

\usetikzlibrary{matrix,decorations.pathreplacing,calc}

\pgfkeys{tikz/mymatrixenv/.style={decoration=brace,every left delimiter/.style={xshift=3pt},every right delimiter/.style={xshift=-3pt}}}
\pgfkeys{tikz/mymatrix/.style={matrix of math nodes,left delimiter=[,right delimiter={]},inner sep=2pt,column sep=1em,row sep=0.5em,nodes={inner sep=0pt}}}
\pgfkeys{tikz/mymatrixbrace/.style={decorate,thick}}

\title{Lower Bounds for Learning in Revealing POMDPs}

\date{\today}
\author{
  Fan Chen\thanks{Peking University. Email: \texttt{chern@pku.edu.cn}}
  \and
  Huan Wang\thanks{Salesforce Research. Email: \texttt{\{huan.wang,cxiong,yu.bai\}@salesforce.com}}
  \and
  Caiming Xiong\footnotemark[2]
  \and
  Song Mei\thanks{UC Berkeley. Email: \texttt{songmei@berkeley.edu}}
  \and
  Yu Bai\footnotemark[2]
}

\def\shownotes{0}  %
\ifnum\shownotes=1
\newcommand{\authnote}[2]{{\scriptsize $\ll$\textsf{#1 notes: #2}$\gg$}}
\else
\newcommand{\authnote}[2]{}
\fi

\begin{document}
\maketitle

\begin{abstract}
This paper studies the fundamental limits of reinforcement learning (RL) in the challenging \emph{partially observable} setting. While it is well-established that learning in Partially Observable Markov Decision Processes (POMDPs) requires exponentially many samples in the worst case, a surge of recent work shows that polynomial sample complexities are achievable under the \emph{revealing condition}---A natural condition that requires the observables to reveal some information about the unobserved latent states. However, the fundamental limits for learning in revealing POMDPs are much less understood, with existing lower bounds being rather preliminary and having substantial gaps from the current best upper bounds.

We establish strong PAC and regret lower bounds for learning in revealing POMDPs. Our lower bounds scale polynomially in all relevant problem parameters in a multiplicative fashion, and achieve significantly smaller gaps against the current best upper bounds, providing a solid starting point for future studies. In particular, for \emph{multi-step} revealing POMDPs, we show that (1) the latent state-space dependence is at least $\Omega(S^{1.5})$ in the PAC sample complexity, which is notably harder than the $\widetilde{\Theta}(S)$ scaling for fully-observable MDPs; (2) Any polynomial sublinear regret is at least $\Omega(T^{2/3})$, suggesting its fundamental difference from the \emph{single-step} case where $\widetilde{\mathcal{O}}(\sqrt{T})$ regret is achievable. Technically, our hard instance construction adapts techniques in \emph{distribution testing}, which is new to the RL literature and may be of independent interest.

\end{abstract}
\section{Introduction}

\begin{table*}[t]
\renewcommand{\arraystretch}{1.4}
\centering
\caption{ \small A summary of lower bounds and current best upper bounds for learning revealing POMDPs, with our contributions highlighted in gray cells. The rates presented here only focus on the dependence in $S, O, A, \arev^{-1}$, and $T$ (or $\epsilon^{-1}$), and omit ${\rm poly}(H)$ and all polylog factors. We also assume $O \ge \Omega(SA)$ (in our upper bounds) and $A^H \gg \poly(H, S, O, A^m, \arev^{-1}, T)$ to simplify the presentation. For regret lower bounds, we additional ignore the min with $T$ (due to the trivial $O(T)$ regret upper bound). $^*$Obtained by an explore-then-exploit conversion.
}
\vskip0.1cm
{
\small
\label{table:POMDP}
\begin{tabular}{|c|c|c|c|c|}
\hline
\multirow{2}{*}{ \textbf{Problem} } & \multicolumn{2}{c|}{\textbf{PAC sample complexity}} & \multicolumn{2}{c|}{\textbf{Regret}} \\
\cline{2-5}
 & \textbf{Upper bound} & \textbf{Lower bound} & \textbf{Upper bound} & \textbf{Lower bound} \\
\hhline{|-----|}
\rule{0pt}{15pt} \multirow{2}{*}{\makecell{ $1$-step \\ $\arev$-revealing }} & 
 $\ctO\Big(\frac{S^2OA}{\arev^2\epsilon^2}\Big)$ & 
\cellcolor{light-gray} $\Omega\Big(\frac{SO^{1/2}A}{\arev^2\epsilon^2}\Big)$ & \cellcolor{light-gray} 
 $\ctO\Big(\sqrt{\frac{S^2O^2A}{\arev^2}\cdot T}\Big)$  & \cellcolor{light-gray} $\Omega\Big(\sqrt{\frac{SO^{1/2}A}{\arev^2}\cdot T} \Big)$ \\ 
 \arrayrulecolor{light-gray}\cline{3-5}\arrayrulecolor{black}
&\citep{chen2022partially}&\cellcolor{light-gray} 
 (\cref{thm:1-step-demo})&\cellcolor{light-gray}  (\cref{thm:regret-upper})&\cellcolor{light-gray}  (\cref{cor:1-step-regret}) \\
\hline
\rule{0pt}{15pt} \multirow{2}{*}{\makecell{ $m$-step ($m\ge 2$) \\ $\arev$-revealing }} & $\ctO\Big(\frac{S^2OA^m}{\arev^2\epsilon^2}\Big)$ & \cellcolor{light-gray} $\Omega\Big(\frac{(S^{3/2}+SA)O^{1/2}A^{m-1}}{\arev^2\epsilon^2}\Big)$ & $\ctO\Big(\paren{\frac{S^2OA^{m}}{\arev^2}}^{1/3} T^{2/3}\Big)$& \cellcolor{light-gray} $\Omega\Big(\paren{\frac{SO^{1/2}A^{m}}{\arev^2}}^{1/3} T^{2/3} \Big)$\\  \arrayrulecolor{light-gray}\cline{3-3}\cline{5-5}\arrayrulecolor{black}
&\citep{chen2022partially}&\cellcolor{light-gray} (\cref{thm:multi-step-pac-demo})&\citep{chen2022partially}$^*$&\cellcolor{light-gray} (\cref{thm:no-regret-demo})\\
\hline
\end{tabular}
}
\end{table*} 

Partial observability---where the agent can only observe partial information about the true underlying state of the system---is ubiquitous in real-world applications of Reinforcement Learning (RL) and constitutes a central challenge to RL~\citep{kaelbling1998planning,sutton2018reinforcement}. It is known that learning in the standard model of Partially Observable Markov Decision Processes (POMDPs) is much more challenging than its fully observable counterpart---Finding a near-optimal policy in long-horizon POMDPs requires a number of samples at least exponential in the horizon length in the worst-case~\citep{krishnamurthy2016pac}. Such an exponential hardness originates from the fact that the agent may not observe any useful information about the true underlying state of the system, without further restrictions on the structure of the POMDP. This is in stark contrast to learning fully observable (tabular) MDPs where polynomially many samples are necessary and sufficient without further assumptions~\citep{kearns2002near,jaksch2010near,azar2017minimax,jin2018q,zhang2020almost,domingues2021episodic}. 

Towards circumventing this hardness result, recent work seeks additional structural conditions that permit sample-efficient learning. One natural proposal is the \emph{revealing condition}~\citep{jin2020sample,liu2022partially}, which at a high level requires the observables (observations and actions) to reveal some information about the underlying latent state, thus ruling out the aforementioned worst-case situation where the observables are completely uninformative. Concretely, the \emph{single-step} revealing condition~\citep{jin2020sample} requires the (immediate) emission probabilities of the latent states to be well-conditioned, in the sense that different states are probabilistically distinguishable from their emissions. The \emph{multi-step} revealing condition~\citep{liu2022partially} generalizes the single-step case by requiring the well conditioning of the multi-step \emph{emission-action} probabilities---the probabilities of observing a \emph{sequence} of observations in the next $m\ge 2$ steps, conditioned on taking a specific \emph{sequence} of actions at the current latent state.

Sample-efficient algorithms for learning single-step and multi-step revealing POMDPs are initially designed by~\citet{jin2020sample} and~\citet{liu2022partially}, and subsequently developed in a surge of recent work~\citep{cai2022reinforcement,wang2022embed,uehara2022provably,zhan2022pac,chen2022partially,liu2022optimistic,zhong2022posterior}. For finding an $\eps$ near-optimal policy in $m$-step revealing POMDPs, these results obtain PAC sample complexities (required episodes of play) that scale polynomially with the number of states, observations, action sequences (of length $m$), the horizon, $(1/\arev)$ where $\arev>0$ is the \emph{revealing constant}, and $(1/\eps)$, with the current best rate given by~\citet{chen2022partially}.

Despite this progress, the fundamental limit for learning in revealing POMDPs remains rather poorly understood. First, lower bounds for revealing POMDPs are currently scarce, with existing lower bounds either being rather preliminary in its rates~\citep{liu2022partially}, or following by direct reduction from fully observable settings, which does not exhibit the challenge of partial observability (cf.~\cref{section:known} for detailed discussions). Such lower bounds leave open many fundamental questions, such as the dependence on $\arev$ in the optimal PAC sample complexity: the current best lower bound scales in $\arev^{-1}$ while the current best upper bound requires $\arev^{-2}$. Second, the current best upper bounds for learning revealing POMDPs are mostly obtained by general-purpose algorithms not specially tailored to POMDPs~\citep{chen2022partially,liu2022optimistic,zhong2022posterior}. These algorithms admit unified analysis frameworks for a large number of RL problems including revealing POMDPs, and it is unclear whether these analyses (and the resulting upper bounds) unveil fundamental limits of revealing POMDPs.

This paper establishes strong sample complexity lower bounds for learning revealing POMDPs. Our contributions can be summarized as follows.
\begin{itemize}[leftmargin=1em, topsep=0pt, itemsep=0pt]
\item We establish PAC lower bounds for learning both single-step (\cref{section:1-step-pac}) and multi-step (\cref{section:m-step-pac}) revealing POMDPs. Our lower bounds are the first to scale with all relevant problem parameters in a multiplicative fashion, and settles several open questions about the fundamental limits for learning revealing POMDPs. 
Notably, our PAC lower bound for the multi-step case scales as $\Omega(S^{1.5})$, where $S$ is the size of the latent state-space, which is notably harder than fully observable MDPs where $\wt{\Theta}(S)$ is the minimax optimal scaling.
Further, our lower bounds exhibit rather mild gaps from the current best upper bounds, which could serve as a starting point for further fine-grained studies.
\item We establish regret lower bounds for the same settings. Perhaps surprisingly, we show an $\Omega(T^{2/3})$ regret lower bound for multi-step revealing POMDPs (\cref{section:regret}). Our construction unveils some new insights about the multi-step case, and suggests its fundamental difference from the single-step case in which $\tO(\sqrt{T})$ regret is achievable.
\item Technically, our lower bounds are obtained by embedding \emph{uniformity testing} problems into revealing POMDPs, in particular into an \emph{$m$-step revealing combination lock} which is the core of our hard instance constructions (\cref{section:proof-overview}). The proof further uses information-theoretic techniques such as Ingster's method for bounding certain divergences, which are new to the RL literature.
\item We discuss some additional interesting implications to RL theory in general, in particular to the Decision-Estimation Coefficients (DEC) framework (\cref{section:dec-implications}).
\end{itemize}

We illustrate our main results against the current best upper bounds in~\cref{table:POMDP}.

\subsection{Related work}

\paragraph{Hardness of learning general POMDPs}
It is well-established that learning a near-optimal policy in POMDPs is computationally hard in the worst case \cite{papadimitriou1987complexity, mossel2005learning}. %
With regard to learning, \citet{krishnamurthy2016pac, jin2020sample} used the combination lock hard instance to show that learning episodic POMDPs requires a sample size at least exponential in the horizon $H$. \citet{kearns1999approximate, even2005reinforcement} developed algorithms for learning episodic POMDPs that admit sample complexity scaling with $A^H$. A similar sample complexity can also be obtained by bounding the Bellman rank~\citep{jiang2017contextual,du2021bilinear,jin2021bellman} or coverability~\citep{xie2022role}.

\paragraph{Revealing POMDPs}
\citet{jin2020sample} proposed the single-step revealing condition in under-complete POMDPs and
showed that it is a sufficient condition for sample-efficient learning of POMDPs by designing a spectral type learning algorithm. \citet{liu2022partially, liu2022sample} proposed the multi-step revealing condition to the over-complete POMDPs and developed the optimistic maximum likelihood estimation (OMLE) algorithm for efficient learning. \citet{cai2022reinforcement, wang2022embed} extended these results to efficient learning of linear POMDPs under variants of the revealing condition. \citet{golowich2022planning, golowich2022learning} showed that approximate planning under the observable condition, a variant of the revealing condition, admits quasi-polynomial time algorithms. 

The only existing lower bound for learning revealing POMDPs is provided by \citet{liu2022partially}, which modified the combination lock hard instance \cite{krishnamurthy2016pac} to construct an $m$-step $1$-revealing POMDP and show an $\Omega(A^{m-1})$ sample complexity lower bound for learning a $1/2$-optimal policy. Our lower bound improves substantially over theirs using a much more sophisticated hard instance construction that integrates the combination lock with the tree hard instance for learning MDPs \citep{domingues2021episodic} and the hard instance for uniformity testing \citep{paninski2008coincidence, canonne2020survey}. Similar to the lower bound for uniformity testing, the proof of our lower bound builds on Ingster's method~\citep{ingster2012nonparametric}.

\paragraph{Other structural conditions}
Other conditions that enable sample-efficient learning of POMDPs include reactiveness~\citep{jiang2017contextual}, decodablity~\citep{efroni2022provable}, structured latent MDPs~\citep{kwon2021rl}, learning short-memory policies~\citep{uehara2022provably}, deterministic transitions~\citep{uehara2022computationally}, and regular predictive state representations (PSRs)~\citep{zhan2022pac}. \citet{chen2022partially, liu2022optimistic, zhong2022posterior} propose unified structural conditions for PSRs, which encompasses most existing tractable classes including revealing POMDPs, decodable POMDPs, and regular PSRs.

\section{Preliminaries}
\label{section:prelim}

\paragraph{POMDPs} An episodic Partially Observable Markov Decision Process (POMDP) is specified by a tuple $M=\{H,\cS,\cO,\cA,\{\T_h\}_{h \in [H]},\{\O_h\}_{h \in [H]},\{r_{h}\}_{h \in [H]},\mu_1 \}$, where 
$H\in\Z_{\ge 1}$ is the horizon length; 
$(\cS,\cO,\cA)$ are the spaces of (latent) states, observations, and actions with cardinality $(S,O,A)$ respectively; 
$\O_h(\cdot|\cdot):\cS\to\Delta(\cO)$ is the emission dynamics at step $h$ (which we identify as an emission matrix $\O_h\in\R^{\cO\times \cS}$); 
$\T_h(\cdot|\cdot,\cdot):\cS\times\cA\to\Delta(\cS)$ is the transition dynamics over the latent states (which we identify as a transition matrix $\T_h\in\R^{\cS\times (\cS\times\cA)}$); 
$r_h(\cdot,\cdot):\cO\times\cA\to[0,1]$ is the (possibly random) reward function; 
$\mu_1=\T_0(\cdot)\in\Delta(\cS)$ specifies the distribution of initial state. 
At each step $h\in[H]$, given latent state $s_h$ (which the agent does not observe), the system emits observation $o_h\sim \O_h(\cdot|s_h)$, receives action $a_h\in\cA$ from the agent, emits reward $r_h(o_h,a_h)$, and then transits to the next latent state $s_{h+1}\sim \T_h(\cdot|s_h, a_h)$ in a Markovian fashion. 

We use $\tau=(o_1,a_1,\dots,o_H,a_H)=(o_{1:H}, a_{1:H})$ to denote a full history of observations and actions observed by the agent, and $\tau_h=(o_{1:h}, a_{1:h})$ to denote a partial history up to step $h\in[H]$. A policy is given by a collection of distributions over actions $\pi=\set{\pi_h(\cdot|\tau_{h-1},o_h)\in\Delta(\cA)}_{h,\tau_{h-1},o_h}$, where $\pi_h(\cdot | \tau_{h-1},o_h)$ specifies the distribution of $a_h$ given the history $(\tau_{h-1},o_h)$. We denote $\Pi$ as the set of all policies. The value function of any policy $\pi$ is denoted as $V_M(\pi)=\E_M^\pi[\sum_{h=1}^H r_h(o_h, a_h)]$, where $\E_M^\pi$ specifies the law of $(o_{1:H}, a_{1:H})$ under model $M$ and policy $\pi$. The optimal value function of model $M$ is denoted as $V^\star_M=\max_{\pi\in\Pi} V_M(\pi)$. Without loss of generality, we assume that the total rewards are bounded by one, i.e. $\sum_{h\in[H]} r_h(o_h, a_h)\le 1$ for any $(o_{1:H},a_{1:H})\in(\cO\times\cA)^H$.

\paragraph{Learning goals}
We consider learning POMDPs from bandit feedback (exploration setting) where the agent plays with a fixed (unknown) POMDP model $M$ for $T \in \N_+$ episodes. In each episode, the agent plays some policy $\pi^{(t)}$, and observes the trajectory $\tau^{(t)}$ and the rewards $r^{(t)}_{1:H}$. 

We consider the two standard learning goals of PAC learning and no-regret learning. In PAC learning, the goal is to output a near-optimal policy $\hat{\pi}$ so that $V^\star_M - V_M(\hat{\pi})\le \eps$ within as few episodes of play as possible. In no-regret learning, the goal is to minimize the regret
\begin{align*}
\textstyle
\Reg(T) \defeq \sum_{t=1}^T \paren{V^\star_M - V_M\paren{\pi^{(t)}}},
\end{align*}
and an algorithm is called no-regret if $\Reg(T)=o(T)$ is sublinear in $T$. It is known that no-regret learning is no easier than PAC learning, as any no-regret algorithm can be turned to a PAC learning algorithm by the standard online-to-batch conversion (e.g.~\citet{jin2018q}) that outputs the average policy $\hat{\pi}\defeq \frac{1}{T}\sum_{t=1}^T \pi^{(t)}$ after $T$ episodes of play.

\subsection{Revealing POMDPs}

We consider revealing POMDPs~\citep{jin2020sample,liu2022partially}, a structured subclass of POMDPs that is known to be sample-efficiently learnable.
For any $m\ge 1$, define the \emph{$m$-step emission-action matrix} $\Mhm\in\R^{\cO^m\cA^{m-1}\times \cS}$ of a POMDP $M$ at step $h\in[H-m+1]$ as
\begin{align}
    [\Mhm]_{(\o,\a), s}
    \defeq  \P_M(o_{h:h+m-1} = \o | s_h = s, a_{h:h+m-2} = \a). \label{eqn:def-m-emi}
\end{align}
In the special case where $m=1$ (the \emph{single-step} case), we have $\Mhone=\O_h \in \R^{\cO \times \cS}$, i.e. the emission-action matrix reduces to the emission matrix. For $m\ge 2$, the $m$-step emission-action matrix $\Mhm$ generalizes the emission matrix by encoding the \emph{emission-action probabilities}, i.e. probabilities of observing any observation sequence $\o\in\cO^m$, starting from any latent state $s\in\cS$ and taking any action sequence $\a\in\cA^{m-1}$ in the next $m-1$ steps.

A POMDP is called $m$-step revealing if its emission-action matrices $\{\Mhm\}_{h \in [H - m + 1]}$ admit \emph{generalized left inverses} with bounded operator norm.

\begin{definition}[$m$-step $\arev$-revealing POMDPs]
\label{definition:m-step-revealing}
For $m\ge 1$ and $\arev>0$, a POMDP model $M$ is called \emph{$m$-step revealing}, if there exists matrices $\Mhm^+\in\R^{\cS\times\cO^m\cA^{m-1}}$ satisfying $\Mhm^+\Mhm \T_{h-1}=\T_{h-1}$ (generalized left inverse of $\Mhm$) for any $h\in[H-m+1]$. 
Furthermore, the POMDP model $M$ is called $m$-step $\alpha$-revealing if each $\Mhm^+$ further admits $(*\to 1)$-operator norm bounded by $\alpha^{-1}$:
\begin{align}\label{eqn:alpha-revealing}
    \| \Mhm^+ \|_{* \to 1} \defeq \max_{\nrmst{\bx}\le 1} \| \Mhm^+\bx \|_1 \le \arev^{-1},
\end{align}
where for any vector $\bx=(\bx(\o,\a))_{\o\in\cO^m,\a\in\cA^{m-1}}$, we denote its star-norm by
\[\textstyle
\nrmst{\bx}\defeq \Big[\sum_{\a\in\cA^{m-1}} \Big(\sum_{\o\in\cO^m} \abs{\bx(\o,\a)} \Big)^2 \Big]^{1/2}. 
\]
Let $\arevm(M)$---the \emph{$m$-step revealing constant} of model $M$---denote the maximum possible $\arev>0$ such that \cref{eqn:alpha-revealing} holds, so that $M$ is $m$-step $\arev$-revealing iff $\arevm(M) \ge \arev$.
\end{definition}

In~\cref{definition:m-step-revealing}, the existence of a generalized left inverse requires the matrix $\Mhm$ to have full rank in the column space of $\T_{h-1}$, 
which ensures that different states reachable from the previous step are information-theoretically distinguishable from the next $m$ observations and $m-1$ actions. The revealing condition---as a quantitative version of this full rank condition---ensures that states can be probabilistically ``revealed'' from the observables, and enables sample-efficient learning~\citep{liu2022partially}.

The choice of the particular norm in~\cref{eqn:alpha-revealing} is not important when only polynomial learnability is of consideration, due to the equivalence between norms. Our choice of the $(*\to 1)$-norm is different from existing work~\citep{liu2022partially,liu2022optimistic,chen2022partially}; however, it enables a tighter gap between our lower bounds and existing upper bounds.

\paragraph{Single-step vs. multi-step}
We highlight that when $m=1$, the emission-action matrix $\Mhone=\O_h$ does not involve the effect of actions. This turns out to make it qualitatively different from the \emph{multi-step} cases where $m\ge 2$, which will be reflected in our results. 

Additionally, we show that any $m$-step $\arev$-revealing POMDP is also $(m+1)$-step $\arev$-revealing, but not vice versa (proof in~\cref{appendix:proof-m-step-mp1-step}; this result is intuitive yet we were unable to find it in the literature). Therefore, as $m$ increases, the class of $m$-step revealing POMDPs becomes strictly larger and thus no easier to learn.
\begin{proposition}[$m$-step revealing $\subsetneq$ $(m+1)$-step revealing]
\label{prop:m-step-mp1-step}
For any $m\ge 1$ and any POMDP $M$ with horizon $H\ge m+1$, we have $\arevmp(M)\ge \arevm(M)$. Consequently, any $m$-step $\arev$-revealing POMDP is also an $(m+1)$-step $\arev$-revealing POMDP. Conversely, there exists an $(m+1)$-step revealing POMDP that is not an $m$-step revealing POMDP.
\end{proposition}

\subsection{Known upper and lower bounds}
\label{section:known}

\paragraph{Upper bounds}
Learning revealing POMDPs is known to admit polynomial sample complexity upper bounds~\citep{liu2022partially,liu2022optimistic,chen2022partially}. The current best PAC sample complexity for learning revealing POMDPs is given in the following result, which follows directly by adapting the results of~\citet{chen2022partially,chen2022unified} to our definition of the revealing condition (cf.~\cref{appendix:proof-PAC-upper}).
\begin{theorem}[PAC upper bound for revealing POMDPs~\citep{chen2022partially}]
\label{thm:PAC-upper}
There exists algorithms (\omle{}, \eetod{} \& \mops{}) that can find an $\epsilon$-optimal policy of any $m$-step $\arev$-revealing POMDP w.h.p. within 
\begin{align}
\label{eqn:PAC-upper}
    T\le \tbO{\frac{S^2OA^m(1+SA/O)H^3}{\arev^2\epsilon^2}} 
\end{align}
episodes of play.
\end{theorem}

\paragraph{Lower bounds}
Existing lower bounds for learning revealing POMDPs are scarce and preliminary. The only existing PAC lower bound %
for $m$-step $\alpha$-revealing POMDPs is 
\[\textstyle
\Omega(\min\set{\frac{1}{\arev H}, A^{H-1}} +A^{m-1})
\]
given by~\citet[Theorem 6 \& 9]{liu2022partially} for learning an $\eps=\Theta(1)$-optimal policy, which does not scale with either the model parameters $S,O$ or $(1/\eps)$ for small $\eps$.

In addition, revealing POMDPs subsume two fully observable models as special cases: (fully observable) MDPs with $H$ steps, $\min\set{S,O}$ states, and $A$ actions (with $\arev=1$); and contextual bandits with $O$ contexts and $A$ actions. By standard PAC lower bounds~\citep{dann2015sample,lattimore2020bandit,domingues2021episodic} in both settings\footnote{With total reward scaled to $[0,1]$.}, this implies an
\begin{align*}
\Omega\paren{ (H\min\set{S,O}A + OA)/\eps^2 }
\end{align*}
PAC lower bound for $m$-step $\arev$-revealing POMDPs for any $m\ge 1$ and $\arev\le 1$. 

Both lower bounds above exhibit substantial gaps from the upper bound~\eqref{eqn:PAC-upper}. Indeed, the upper bound scales \emph{multiplicatively} in $S,A^m,O,\arev^{-1}$ and $1/\eps^2$, whereas the lower bounds combined are far smaller than this multiplicative scaling.

\section{PAC lower bounds}

We establish PAC lower bounds for both single-step (\cref{section:1-step-pac}) and multi-step (\cref{section:m-step-pac}) revealing POMDPs. We first state and discuss our results, and then provide a proof overview for the multi-step case in~\cref{section:proof-overview}.

\subsection{Single-step revealing POMDPs}
\label{section:1-step-pac}

We begin by establishing the PAC lower bound for the single-step case. The proof can be found in~\cref{appdx:1-step-pac}.

\begin{theorem}[PAC lower bound for single-step revealing POMDPs]
\label{thm:1-step-demo}
For any $O\geq S\geq 5$, $A\geq 3$, $H\geq 4\log_2 S$, $\arev\in(0,\frac1{5H}]$, $\epsilon\in(0,0.01]$, there exists a family $\cM$ of single-step revealing POMDPs with $\nS\leq S$, $\nO\leq O$, $\nA=A$, and $\arevone(M)\geq \arev$ for all $M\in\cM$, such that for any algorithm $\fA$ that interacts with the environment for $T$ episodes and returns a $\piout$ such that 
$V_M^\star-V_M(\piout)<\epsilon$ with probability at least $3/4$ for all $M\in\cM$, we must have
\begin{align}
\label{eqn:1-step-pac}
    T\geq c \cdot \min\set{
        \frac{SO^{1/2}AH}{\arev^2\epsilon^2}, \frac{SA^{H/2}H}{\epsilon^2}
    },
\end{align}
where $c>0$ is an absolute constant. 
\end{theorem}

The lower bound in \cref{thm:1-step-demo} (and subsequent lower bounds) involves the minimum over two terms, where the second term ``caps'' the lower bound by an exponential scaling\footnote{A $\tO({\rm poly}(S,O,H)A^H/\eps^2)$ PAC upper bound is indeed achievable for \emph{any} POMDP (not necessarily revealing)~\citep{even2005reinforcement}; see also the discussions in~\citet{uehara2022provably}.} in $H$ and is less important. The main term $\Omega(S\sqrt{O}AH/(\arev^2\eps^2))$ scales polynomially in $1/\arev^2$, $1/\eps^2$, and $(S,O,A)$ in a \emph{multiplicative} fashion. This is the first such result for revealing POMDPs and improves substantially over existing lower bounds (cf.~\cref{section:known}).

\paragraph{Implications}
\cref{thm:1-step-demo} shows that, the multiplicative dependence on $(S,A,O,1/\arev,1/\eps)$ in the the current best PAC upper bound $\tO(S^2OA(1+SA/O)/(\arev^2\eps^2))$ (\cref{thm:PAC-upper}; ignoring $H$) is indeed necessary, and settles several open questions about learning revealing POMDPs:
\begin{itemize}[leftmargin=1em, topsep=0pt, itemsep=0pt]
\item It settles the optimal dependence on $\arev$ to be $\Theta(\arev^{-2})$ (combining our lower bound with the $\cO(\arev^{-2})$ upper bound), whereas the previous best lower bound on $\alpha$ is $\Omega(\arev^{-1})$~\citep{liu2022partially}.
\item For joint dependence on $(\arev,\eps)$, it shows that $1/(\arev^2\eps^2)$ samples are necessary. This rules out possibilities for better rates---such as the $\tO(\max\{1/\arev^2, 1/\eps^2\})$ upper bound for single-step revealing POMDPs with \emph{deterministic transitions}~\citep{jin2020sample}---in the general case. %
\item It necessitates a ${\rm poly}(O)$ factor as multiplicative upon the other parameters (most importantly $1/(\alpha^2\eps^2)$) in the sample complexity, which confirms that large observation spaces do impact learning in a strong sense. 
\end{itemize}

Finally, compared with the current best PAC upper bound, the lower bound $\Omega(SO^{1/2}A/(\arev^2\eps^2))$ captures all the parameters and is a $S\sqrt{O}$-factor away in the rich-observation regime where $O\ge \Omega(SA)$. This provides a solid starting point for future studies. %

\subsection{Multi-step revealing POMDPs}
\label{section:m-step-pac}

Using similar hard instance constructions (more details in~\cref{section:proof-overview}), we establish the PAC lower bound for the multi-step case with $m\ge 2$ (proof in~\cref{appdx:multi-step-pac}).
\begin{theorem}[PAC lower bound for multi-step revealing POMDPs]
\label{thm:multi-step-pac-demo}
For any $m\geq 2$, $O\geq S\geq 10$, $A\geq 3$, $H\geq 8\log_2 S+2m$, $\arev\in(0,0.1]$, $\epsilon\in(0,0.01]$, there exists a family $\cM$ of $m$-step revealing POMDPs with $\nS\leq S$, $\nO\leq O$, $\nA=A$, and $\arevm(M)\geq \arev$ for all $M\in\cM$, such that any algorithm $\fA$ that interacts with the environment and returns a $\piout$ such that 
$V_M^\star-V_M (\piout) <\epsilon$ with probability at least $3/4$ for all $M\in\cM$, we must have
\begin{align*}
    T\geq c_m \cdot \min\set{
        \frac{(S^{1.5}\vee SA)O^{1/2}A^{m-1}H}{\arev^2\epsilon^2}, \frac{SA^{H/2}H}{\epsilon^2}
    },
\end{align*}
where $c_m=c_0/m$ for some absolute constant $c_0>0$. %
\end{theorem}

The main difference in the multi-step case (\cref{thm:multi-step-pac-demo}) is in its higher $A$ dependence $\Omega(A^{m-1})$, which suggests that the $A^m$ dependence in the upper bound (\cref{thm:PAC-upper}) is morally unimprovable. Also, the $S^{1.5}$ scaling in~\cref{thm:multi-step-pac-demo} is higher than \cref{thm:1-step-demo}, which makes the result qualitatively stronger than the single-step case even aside from the $A$-dependence. This happens since the hard instance here is actually a strengthening---instead of a direct adaptation---of the single-step case, by leveraging the nature of multi-step revealing; see~\cref{section:actual} for a discussion.

Again, compared with the current best PAC upper bound $S^2OA^m(1+SA/O)/(\arev^2\eps^2)$ (\cref{thm:PAC-upper}), the lower bound in~\cref{thm:multi-step-pac-demo} has an $\sqrt{SO}A\wedge S\sqrt{O}$ gap from the current best upper bound. We believe that the $\sqrt{SO}$ factor in this gap is unimprovable from the lower bound side under the current hard instance; see~\cref{section:close-gap} for a discussion.

\paragraph{$\sqrt{O}$ dependence} Our lower bounds for both the single-step and the multi-step cases scale as $\sqrt{O}$ in its $O$-dependence. Such a scaling comes from the complexity of the \emph{uniformity testing} task of size $\cO(O)$, embedded in the revealing POMDP hard instances, whose sample complexity is $\Theta(\sqrt{O}/\eps^2)$ \citep{paninski2008coincidence,diakonikolas2014testing,canonne2020survey}. The construction of the hard instances will be described in detail in~\cref{section:proof-overview}.

\section{Regret lower bound for multi-step case}
\label{section:regret}

We now turn to establishing regret lower bounds. We show that surprisingly, for $m$-step revealing POMDPs with any $m\ge 2$, a non-trivial polynomial regret (neither linear in $T$ nor exponential in $H$) has to be at least $\Omega(T^{2/3})$. The proof can be found in~\cref{appdx:no-regret}.

\begin{theorem}[$\Omega(T^{2/3})$ regret lower bound for multi-step revealing POMDPs]
\label{thm:no-regret-demo}
For any $m\geq 2, O\geq S\geq 8$, $A\geq 3$, $H\geq 8\log_2 S+2m$, $\arev\in(0,0.1]$, $T\geq 1$, there exists a family $\cM$ of $m$-step revealing POMDPs with $\nS\leq S$, $\nO\leq O$, $\nA=A$, and $\arevm(M)\geq \arev$ for all $M\in\cM$, such that for any algorithm $\fA$, it holds that
\begin{align*}
    & \max_{M\in\cM} \EE^{\fA}_M\brac{\Regret}\geq %
    c_m \cdot  \min\set{
        \paren{\frac{SO^{1/2}A^mH}{\arev^2}}^{1/3} T^{2/3}, \sqrt{SA^{H/2}HT}, T
    },
\end{align*}
where $c_m=c_0/m$ for some absolute constant $c_0 > 0$. %
\end{theorem}
Currently, the best sublinear regret (polynomial in other problem parameters) is indeed $T^{2/3}$ by a standard explore-then-exploit style conversion from the PAC result~\citep{chen2022partially}. \cref{thm:no-regret-demo} rules out possibilities for obtaining an improvement (e.g. to $\sqrt{T}$) by showing that $T^{2/3}$ is rather a fundamental limit.

\paragraph{Proof intuition}
The hard instance used in~\cref{thm:no-regret-demo} is the same as one of the PAC hard instances (see~\cref{section:proof-overview}). However,~\cref{thm:no-regret-demo} relies on a key new observation 
that leads to the $\Omega(T^{2/3})$ regret lower bound. Specifically, for multi-step revealing POMDPs, we can design a hard instance such that 
the following two kinds of action sequences (of length $m-1$) are \emph{disjoint}:
\begin{itemize}[topsep=0pt, itemsep=0pt]
\item \emph{Revealing} action sequences, which yield observations that reveal information about the true latent state;
\item \emph{High-reward} action sequences.
\end{itemize}
The multi-step revealing condition (\cref{definition:m-step-revealing}) permits such constructions. Intuitively, this is since its requirement that $\Mhm\in\R^{\cO^m\cA^{m-1}\times\cS}$ admits a generalized left inverse is fairly liberal, and can be achieved by carefully designing the emission-action probabilities over a \emph{subset} of action sequences. In other words, the multi-step revealing condition allows only \emph{some} action sequences to be revealing, such as the ones that receive rather suboptimal rewards.

Such a hard instance forbids an efficient exploration-exploitation tradeoff, as exploration (taking revealing actions) and exploitation (taking high-reward actions) cannot be simultaneously done. Consequently, the best thing to do is simply an explore-then-exploit type algorithm\footnote{Alternatively, a bandit-style algorithm that does not take revealing actions but instead attempts to identify the optimal policy directly by brute-force trying, which corresponds to the $\sqrt{A^HT}$ term in~\cref{thm:no-regret-demo}.} whose regret is typically $\Theta(T^{2/3})$~\citep{lattimore2020bandit}.

\paragraph{Difference from the single-step case} 
\cref{thm:no-regret-demo} demonstrates a fundamental difference between the multi-step and single-step settings, as single-step revealing POMDPs are known to admit $\tO(\sqrt{T})$ regret upper bounds~\citep{liu2022partially}. %
Intuitively, the difference is that in single-step revealing POMDPs, the agent does not need to take specific actions to acquire information about the latent state, so that information acquisition (exploration) and taking high-reward actions (exploitation) \emph{can} always be achieved simultaneously. %

\paragraph{Towards $\sqrt{T}$ regret under stronger assumptions}
It is natural to ask whether the $\Omega(T^{2/3})$ lower bound can be circumvented by suitably strengthening the multi-step revealing condition (yet still weaker than single-step revealing). Based on our intuitions above, a possible direction is to additionally require that \emph{all} action sequences (of length $m-1$) must reveal information about the latent state. We leave this as a question for future work.

\begin{figure*}[t]
\centering
\includegraphics[width=0.96\textwidth]{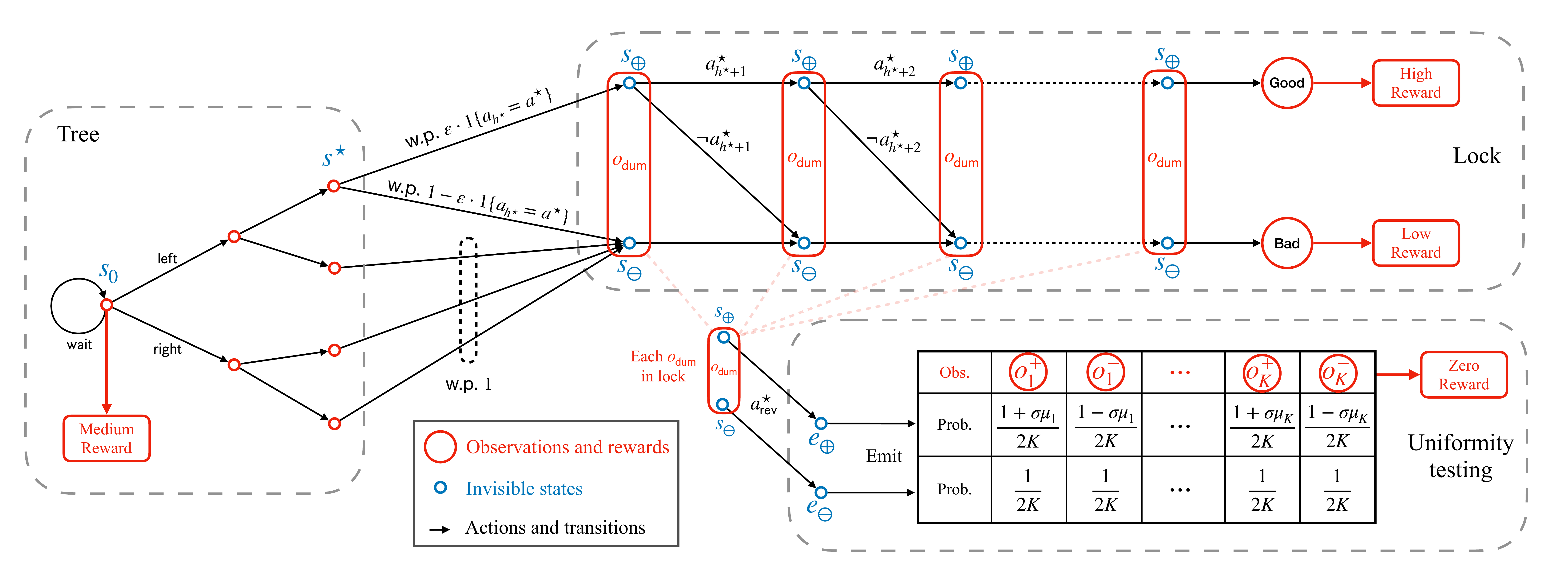}
\vspace{-1em}
  \caption{
  Schematic plot of a simplified version of our hard instance for $2$-step revealing POMDPs. The instance consists of three components: tree, lock, and uniformity testing. In the tree, all transitions are deterministic and fully observable, and the agent fully controls how to transit from $s_0$ to a leaf node. The tree transits stochastically to the lock if any action is taken at any leaf of the tree, but there is a unique (unknown) state $s^\star$, step $h^\star$, and action $a^\star$ at which the agent to transit to $s_\oplus$ with positive probability. In the lock, the agent cannot observe the latent states $\{\sp, \sq\}$, and they need to enter the correct password $\a^\star$ to stay at $s_\oplus$ to eventually receive a high reward. The agent may also take the revealing action $\revs$ at any $o_{\sf dum}$ to transit to the uniformity testing component, in which they will receive an observation that slightly reveals whether the previous latent state is $s_\oplus$ or $s_\ominus$. See \cref{section:proof-overview-construction} for a more detailed description.
  }
  \label{fig:tree-lock}
\end{figure*}

\section{Proof overview}
\label{section:proof-overview}

We now provide a technical overview of the hard instance constructions and the lower bound proofs. We present a simplified version of the multi-step revealing hard instance in~\cref{appdx:no-regret} that is used for proving both the PAC and the regret lower bounds (\cref{thm:multi-step-pac-demo} \&~\ref{thm:no-regret-demo}). For simplicity, we describe our construction in the 2-step case ($m=2$); a schematic plot of the resulting POMDP is given in~\cref{fig:tree-lock}.

\subsection{Construction of hard instance}\label{section:proof-overview-construction}

A main challenge for obtaining our lower bounds---compared with existing lower bounds in fully observable settings---is to characterize the difficulty of partial observability, i.e. the dependence on $O$ and $\arev^{-1}$.

\paragraph{2-step revealing combination lock}
To reflect this difficulty, the basic component we design is a ``2-step revealing combination lock'' (cf. the ``Lock'' part in~\cref{fig:tree-lock}), which is a modification of the non-revealing combination lock of \citet{liu2022partially, jin2020sample}. This lock consists of two hidden states $\sp,\sq$ and an (unknown) sequence of ``correct'' actions (i.e. the ``password'') $\as_{h^\star+1:H}$. The only way to stay at $\sp$ is to take the correct action $\as_{h}$ at each step $h$, and only state $\sp$ at step $H$ gives a high reward. Therefore, the task of learning the optimal policy is equivalent to identifying the correct action $\as_{h}$ at each step. We make the hidden states $\sp,\sq$ non-observable (emit dummy observations $o_{\sf dum}$), so that a naive strategy for the agent is to guess the sequence $\as$ from scratch, which incurs an $\exp(\Om{H})$ sample complexity.

A central ingredient of our design is a unique (known) \emph{revealing action} $\revs$ at each step that is always distinct from the correct action. Taking $\revs$ will transit from latent state $\sp$ to $\ep$ which then emits an observation from distribution $\mup\in\Delta(\cO)$, and similarly from $\sq$ to $\eq$ which then emits an observation from distribution $\muq\in\Delta(\cO)$. After this (single) emission, the system deterministically transits to an absorbing terminal state with reward $0$.

\paragraph{Uniformity testing}
We adapt techniques from the uniformity testing~\citep{canonne2020survey,canonne2022topics} literature to pick $\{\mup,\muq\}$ that are as hard to distinguish as possible, yet ensuring that the POMDP still satisfies the $\arev$-revealing condition. Concretely, picking $\muq=\Unif(\cO)$ to be the uniform distribution over $\cO$\footnote{Technically, we pick $\mup,\muq$ to be uniformity testing hard instances on \emph{subset} of $\cO$ with size $2K=\Theta(O)$. Here we use the full set $\cO$ for simplicity of presentation.}, it is known that testing $\muq$ from a nearby $\mup$ with $\DTV{\mup,\muq}\asymp \sigma$ requires $\Theta(\sqrt{O}/\sigma^2)$ samples~\citep{paninski2008coincidence}. Further, the worst-case prior for $\mup$ takes form $\mup=\Unif(\cO) + \sigma \mu/O$, where $\mu\sim \Unif(\{(+1,-1), (-1, +1)\}^{O/2})$. We adopt such choices of $\muq$ and $\mup$ in our hard instance (cf. the ``Uniformity testing'' part in~\cref{fig:tree-lock}), which can also ensure that the POMDP is $\ThO{\sigma^{-1}}$-revealing.

\paragraph{Tree MDP; rewards}
To additionally exhibit an $HSA$ factor in the lower bound, we further embed a fully observable \emph{tree MDP}~\citep{domingues2021episodic} before the combination lock. The tree is a balanced binary tree with $S$ leaf nodes, with deterministic transitions (so that which leaf node to arrive at is fully determined by the action sequence) and full observability. All leaf nodes of the tree will transit to the combination lock (i.e. one of $\set{\sp,\sq}$). However, there exists a unique $(\hs,\ss,\as)$ such that only taking $a_{\hs}=\acs$ at $s_{\hs}=\ss$ and step $\hs$ has a probability $\epsilon$ of transiting to $\sp$; all other choices at leaf nodes transit to $\sq$ with probability one (cf. the ``Tree'' part in~\cref{fig:tree-lock}).

We further design the reward function so that the agent must identify the underlying parameters $(\hs,\ss,\acs)$ correctly to learn a $\Theta(\epsilon)$ near-optimal policy.

\subsection{Calculation of lower bound}
Base on our construction, to learn an $\eps$ near-optimal policy in this hard instance, the agent has to identify $(\hs,\ss,\acs)$, which can only be achieved by trying all ``entrances'' $(s,a,h)$ and testing between 
\begin{align*}
    H_0: &~ \PP(s_{h+1}=\sp|s_h=s,a_h=a)=0,\\
    H_1: &~ \PP(s_{h+1}=\sp|s_h=s,a_h=a)=\epsilon.
\end{align*}
for each entrance. As we have illustrated, to achieve this, the agent has to either (1) guess the password $\as$ from scratch (using $\Om{A^{H-h}/\epsilon^2}$ samples), or (2) take $\revs$ and perform uniformity testing using the observations. The latter task turns out to be equivalent to testing between
\begin{align*}
    H_0' = \muq, \quad H_1' = \epsilon\mup+(1-\epsilon)\muq,
\end{align*}
where $\muq$ is the uniform distribution over $2K=\Theta(O)$ elements, and $\mup$ is drawn from the worst-case prior for uniformity testing. Distinguishing between $H_0'$ and $H_1'$ is a uniformity testing task with parameter $\sigma\eps$, which requires $n \ge \Omega(\sqrt{O}/(\epsilon\sigma)^2)$ samples~\citep{paninski2008coincidence}.

With careful information-theoretic arguments, the arguments above will result in a PAC lower bound  
\begin{align*}
\ThO{SAH}\times\Omega\Big(\min \Big\{ \frac{\sqrt{O}}{\sigma^2\epsilon^2}, \frac{A^{\ThO{H}}}{\epsilon^2} \Big\}\Big),
\end{align*}
for learning 2-step $\Theta(\sigma^{-1})$-revealing POMDPs. This rate is similar as (though slightly worse than) our actual PAC lower bound (\cref{thm:multi-step-pac-demo}). The same hard instance further yields a $\Omega(T^{2/3})$ regret lower bound (though slightly worse rate than \cref{thm:no-regret-demo}); see a calculation in~\cref{appdx:heuristic_regret}.

We remark that the above calculations are heuristic; rigorizing these arguments relies on information-theoretic arguments---in our case Ingster's method~\citep{ingster2012nonparametric} (cf.~\cref{appendix:ingster-method} \&~\cref{lemma:1-step-pac-alternative-weak} as an example)---for bounding the divergences between distributions induced by an arbitrary algorithm on different hard instances.

\subsection{Remark on actual constructions}
\label{section:actual}
The above 2-step hard instance is a simplification of the actual ones used in the proofs of~\cref{thm:multi-step-pac-demo} \&~\ref{thm:no-regret-demo} in several aspects. The actual constructions are slightly more sophisticated, with the following additional ingredients:
\begin{itemize}[topsep=0pt, itemsep=0pt, leftmargin=1em]
\item For the $m$-step case, to obtain a lower bound that scales with $A^{m}$, we modify the construction above so that the agent can take $\revs$ only once per $(m-1)$-steps, and replace $\revs$ by a set $|\Arev|=\Theta(A)$ of revealing actions, which collectively lead to an $A^{m-1}\times A=A^m$ factor.
\item We further obtain an extra $\sqrt{S}$ factor in~\cref{thm:multi-step-pac-demo} by replacing the single combination lock with $\Theta(S)$ parallel locks that \emph{share the same password} but \emph{differ in their emission probabilities}. We show that learning in this setting is least as hard as uniformity testing over $\Theta(SO)$ elements, which leads to the extra $\sqrt{S}$ factor.
\end{itemize}

\section{Discussions}

\subsection{Regret for single-step case}
\label{sec:1-step-regret}

As we have discussed, single-step revealing POMDPs cannot possibly admit a $\Omega(T^{2/3})$ regret lower bound like the multi-step case, as a $\tO(\sqrt{T})$ upper bound is achievable. Nevertheless, we obtain a matching $\Omega(\sqrt{T})$ regret lower bound by a direct reduction from the PAC lower bound (\cref{thm:1-step-demo}) using Markov's inequality and standard online-to-batch conversion, which we state as follows.

\begin{corollary}[Regret lower bound for single-step revealing POMDPs]
\label{cor:1-step-regret}
Under the same setting as~\cref{thm:1-step-demo}, the same family $\cM$ of single-step $\arev$-revealing POMDPs there satisfy that for any algorithm $\fA$,
\begin{align}
\label{eqn:1-step-regret}
\begin{aligned}
\max_{M\in\cM} \EE^{\fA}_M\brac{\Regret} \ge c_0\cdot \min\Big\{\sqrt{\frac{SO^{1/2}AH}{\arev^2} T}, \sqrt{SA^{H/2}HT} , T \Big\},
\end{aligned}
\end{align}
where $c_0>0$ is an absolute constant.
\end{corollary}

To contrast this lower bound, the current best regret upper bound for single-step revealing POMDPs is $\tO(\sqrt{S^3O^3A^2(1+SA/O)\arev^{-4}\cdot T} \times {\rm poly}(H))$~\citep{liu2022optimistic}\footnote{Converted from their result whose revealing constant is defined in $(2\to 2)$-norm.}, which is at least a $\sqrt{S^2O^{2.5}A\arev^{-2}}$-factor larger than the main term in~\cref{eqn:1-step-regret}. Here we present a much sharper regret upper bound, reducing this gap to $\sqrt{SO^{1.5}}$ and importantly settling the dependence on $\arev$. %
\begin{theorem}[Regret upper bound for single-step revealing POMDPs]
\label{thm:regret-upper}
There exists algorithms (\omle{}, \etod{}, and \mops{}) that can interact with any single-step $\arev$-revealing POMDP $M$ and achieve regret
\begin{align}
\label{eqn:regret-upper}\textstyle
    \Regret\le \ctO \Big(\sqrt{\frac{S^2O^2A(1+SA/O)H^3}{\arev^2}\cdot T}\Big)
\end{align}
with high probability. %
\end{theorem}
We establish~\cref{thm:regret-upper} on a broader class of sequential decision problems termed as \emph{strongly B-stable PSRs} (cf.~\cref{appdx:strong-b-stable}), which include single-step revealing POMDPs as a special case. The proof is largely parallel to the analysis of PAC learning for B-stable PSRs~\citep{chen2022partially}, and can be found in~\cref{appdx:1-step-regret}.

\subsection{Implications on the DEC approach}
\label{section:dec-implications}

The Decision-Estimation Coefficient (DEC)~\citep{foster2021statistical} offers another potential approach for establishing sample complexity lower bounds for any general RL problem. However, here we demonstrate that for revealing POMDPs, any lower bound given by the DEC will necessarily be strictly weaker than our lower bounds. 

For example, for PAC learning, the Explorative DEC (EDEC) of $m$-step revealing POMDPs is known to admit an \emph{upper bound} $\edec_\gamma \le \tO(SA^mH^2\arev^{-2}/\gamma)$ (\citet{chen2022partially}; see also~\cref{prop:rev-to-psr}), and consequently any PAC lower bound obtained by \emph{lower bounding} the EDEC is at most $\Omega(SA^mH^2\arev^{-2}/\eps^2)$~\citep{chen2022unified}. Such a lower bound would be necessarily smaller than our~\cref{thm:multi-step-pac-demo} by at least a factor of $\sqrt{O}(1\vee \sqrt{S}/A)$, and importantly does not scale polynomially in $O$.

Our lower bounds have additional interesting implications on the DEC theory in that, while algorithms such as the E2D achieve sample complexity upper bounds in terms of the DEC and log covering number for the \emph{model class}~\citep{foster2021statistical,chen2022unified}, without further assumptions, this log covering number cannot be replaced by that of either the \emph{value class} or the \emph{policy class}, giving negative answers to the corresponding questions left open in~\citet{foster2021statistical} (cf.~\cref{appdx:discussion-dec} for a detailed discussion).

\subsection{Towards closing the gaps}
\label{section:close-gap}

Finally, as an important open question, our lower bounds still have mild gaps from the current best upper bounds, importantly in the $(S, O)$ dependence. For example, for multi-step revealing POMDPs, the (first term in the) PAC lower bound $\Omega(S^{1.5}\sqrt{O}A^{m-1}/(\arev^2\eps^2))$ (\cref{thm:multi-step-pac-demo}) still has a $\sqrt{SO}A$ gap from the upper bound (\cref{thm:PAC-upper}). While we believe that the $A$ factor is an analysis artifact that may be removed, the remaining $\sqrt{SO}$ factor \emph{cannot} be obtained in the lower bound if we stick to the current family of hard instances---There exists an algorithm \emph{specially tailored to this family} that achieves an $\tO(S^{1.5}\sqrt{O}A^m/(\arev^2\eps^2))$ upper bound, by brute-force enumeration in the tree and uniformity testing in the combination lock (Appendix~\ref{appdx:discussion-algorithm}). 

Closing this $\sqrt{SO}$ gap may require either stronger lower bounds with alternative hard instances---e.g. by embedding other problems in distribution testing~\citep{canonne2020survey}---or sharper upper bounds, which we leave as future work.

\section{Conclusion}
This paper establishes sample complexity lower bounds for partially observable reinforcement learning in the important tractable class of revealing POMDPs. Our lower bounds are the first to scale polynomially in the number of states, actions, observations, and the revealing constant in a multiplicative fashion, and suggest rather mild gaps between the lower bounds and current best upper bounds. Our work provides a strong foundation for future fine-grained studies and opens up many interesting questions, such as closing the gaps (from either side), or strengthening the multi-step revealing assumption meaningfully to allow a $\sqrt{T}$ regret.

\bibliography{note}
\bibliographystyle{plainnat}

\appendix

\section{Technical tools}

\begin{lemma}\label{lemma:balance-eps}
    For positive real numbers $A,B,T,\epsilon_0>0$, it holds that
    \begin{align*}
        \sup_{\epsilon\in(0,\epsilon_0]} \paren{ \epsilon T \wedge \frac{A}{\epsilon^2} \wedge \frac{B}{\epsilon}} \geq A^{1/3}T^{2/3} \wedge \sqrt{BT} \wedge \epsilon_0 T.
    \end{align*}
\end{lemma}

\begin{proof}[Proof of \cref{lemma:balance-eps}]
    Suppose that $R>0$ is such that $R\geq \epsilon T \wedge \frac{A}{\epsilon^2} \wedge \frac{B}{\epsilon}$ for all $\epsilon\in(0,\epsilon_0]$. Then for each $\epsilon\in(0,\epsilon_0]$, either $\epsilon\leq \frac{R}{T}$, or $\epsilon\geq \sqrt{\frac{A}{R}}$, or $\epsilon\geq \frac{B}{R}$.
    Thus,
    $$
    (0,\epsilon_0]\subseteq (0,\frac{R}{T}]\cup [\sqrt{\frac{A}{R}},+\infty) \cup [\frac{B}{R},+\infty).
    $$
    Therefore, either $\frac{R}{T}\geq \epsilon_0$, or $\sqrt{\frac{A}{R}}\leq \frac{R}{T}$, or $\frac{B}{R}\leq \frac{R}{T}$. Combining these three cases together, we obtain $$
    R\geq \epsilon_0 T \wedge A^{1/3}T^{2/3} \wedge \sqrt{BT}.
    $$
\end{proof}

\begin{lemma}\label{lemma:martingale-eqn}
    Suppose that $(R_t)_{t \ge 1}$ is a sequence of positive random variables adapted to filtration $(\cF_t)_{t\ge 1}$ and $\stoptime$ is a stopping time (i.e. for $t\geq 1$, $R_t$ is $\cF_{t}$-measurable and the event $\set{\stoptime\leq t}\in\cF_t$). Then it holds that
    \begin{align*}
        \EE\brac{\prod_{t=1}^\stoptime R_t \times \prod_{t=1}^\stoptime \EE[R_t|\cF_{t-1}]^{-1}}=1.
    \end{align*}
    Equivalently,
    \begin{align*}
        \EE\brac{\prod_{t=1}^\stoptime R_t \times \exp\paren{ -\sum_{t=1}^\stoptime \log \EE[R_t|\cF_{t-1}] }}=1.
    \end{align*}
\end{lemma}

\cref{lemma:martingale-eqn} follows immediately from iteratively applications of the tower properties.

\begin{lemma}\label{lemma:MGF-abs}
    Suppose that random variable $X$ is $\sigma$-sub-Gaussian, i.e. $\EE\brac{\exp(tX)}\leq \exp\paren{\frac{\sigma^2t^2}{2}}$ for any $t \in \R$. Then for all $t\geq 0$, we have
    $$
\EE\brac{\exp(t\abs{X})}\leq \exp\paren{\max\set{\sigma^2 t^2, \frac43\sigma t}}.
    $$
\end{lemma}

\begin{proof}[Proof of \cref{lemma:MGF-abs}]
    For any $x\geq1$, we have
    $$
    \EE\brac{\exp(t\abs{X})} \leq \EE\brac{\exp(xt\abs{X})}^{\frac1x} \leq \paren{\EE\brac{\exp(xtX)}+\EE\brac{\exp(-xtX)}}^{\frac1x}\leq 2^{\frac1x}\exp\paren{\frac{\sigma^2t^2x}{2}}
    = \exp\paren{\frac{\sigma^2t^2x}{2}+\frac{\log 2}x}.
    $$
    We consider two cases: 1. If $\sigma t\geq \sqrt{2\log 2}$, then by taking $x=1$ in the above inequality, we have $\EE\brac{\exp(t\abs{X})} \leq\exp(\sigma^2t^2) $. 2. If $\sigma t< \sqrt{2\log 2}$, then by taking $x=\frac{\sqrt{2\log 2}}{\sigma t}>1$ in the above inequality, we have $\EE\brac{\exp(t\abs{X})} \leq\exp(\sqrt{2\log 2}\sigma t) \leq\exp(\frac43\sigma t)$. Combining these two cases completes the proof.
\end{proof}

For probability distributions $\PP$ and $\QQ$ on a measurable space $(\cX,\cF)$ with a base measure $\mu$, we define the TV distance and the Hellinger distance between $\PP,\QQ$ as
\begin{align*}
    \DTV{\PP,\QQ}&=\sup_{A\in\cF}\abs{\PP(A)-\QQ(A)}=\frac12 \int_{\cX} \abs{\frac{d\PP}{d\mu}(x)-\frac{d\QQ}{d\mu}(x)}d\mu(x),\\
    \DH{\PP,\QQ}&= \int_{\cX} \paren{\sqrt{\frac{d\PP}{d\mu}}-\sqrt{\frac{d\QQ}{d\mu}}}^2d\mu.
\end{align*}
When $\PP\ll\QQ$, we can also define the KL-divergence and the $\chi^2$-divergence between $\PP,\QQ$ as
\begin{align*}
    \KL{\PP}{\QQ}=\EE_{\PP}\brac{\log{\frac{d\PP}{d\QQ}}}, \qquad \chis{\PP}{\QQ}=\EE_{\QQ}\brac{\paren{\frac{d\PP}{d\QQ}}^2}-1. 
\end{align*}
\begin{lemma}\label{lemma:change-alg}
    Suppose $\PP,\QQ,\PP',\QQ'$ are four probability measures on $(\cX,\cF)$, and $\Omega$ is an event such that $\PP|_{\Omega}=\PP'|_{\Omega}$, $\QQ|_{\Omega}=\QQ'|_{\Omega}$. Then it holds that 
    \begin{align*}
        \DTV{\PP',\QQ'}\geq \DTV{\PP, \QQ}-\PP(\Omega^c).
    \end{align*}
\end{lemma}
\begin{proof}[Proof of \cref{lemma:change-alg}] Let $\mu$ be a base measure on $(\cX,\cF)$ such that $\PP, \PP', \QQ, \QQ'$ have densities with respect to $\mu$ (for example, $\mu = (\PP + \PP' + \QQ + \QQ')/4$). For notation simplicity, we use $\PP(x)$ to stand for $d \PP(x)/d \mu(x)$ and use $dx$ to stand for $\mu(dx)$. Then we have
\begin{align*}
    2\DTV{\PP',\QQ'}
    =&
    \int_{\cX} \abs{\PP'(x)-\QQ'(x)} dx
    =
    \int_{\Omega} \abs{\PP'(x)-\QQ'(x)}dx+\int_{\Omega^c} \abs{\PP'(x)-\QQ'(x)}dx\\
    \geq &
    \int_{\Omega} \abs{\PP'(x)-\QQ'(x)}dx+ \abs{\PP'(\Omega^c)-\QQ'(\Omega^c)}\\
    = &
    \int_{\Omega} \abs{\PP(x)-\QQ(x)}dx+ \abs{\PP(\Omega^c)-\QQ(\Omega^c)}\\
    \geq &
    \int_{\Omega} \abs{\PP(x)-\QQ(x)}dx+ \PP(\Omega^c)+\QQ(\Omega^c)-2\PP(\Omega^c)\\
    \geq&
    \int_{\Omega} \abs{\PP(x)-\QQ(x)}dx+\int_{\Omega^c} \abs{\PP(x)-\QQ(x)}dx
    -2\PP(\Omega^c)\\
    =&
    2\DTV{\PP, \QQ}-2\PP(\Omega^c).
\end{align*}
This completes the proof. 
\end{proof}

\begin{lemma}[Divergence inequalities, see e.g. \citet{sason2016f}]\label{lemma:TV-to-chi}
For two probability measures $\PP,\QQ$ on $(\cX,\cF)$, it holds that
\begin{align*}
    2\DTV{\PP,\QQ}^2\leq \KL{\PP}{\QQ}\leq \log\paren{1+\chis{\PP}{\QQ}}.
\end{align*}
\end{lemma}

\begin{lemma}[Hellinger conditioning lemma, see e.g.~{\citet[Lemma A.1]{chen2022partially}}]\label{lemma:Hellinger-cond}
    For any pair of random variables $(X,Y)$, it holds that
    \begin{align*}
        \EE_{X\sim\PP_X}\brac{\DH{\PP_{Y|X}, \QQ_{Y|X}}}\leq 2\DH{\PP_{X,Y}, \QQ_{X,Y}}.
    \end{align*}
\end{lemma}

\section{Basics of predictive state representations and B-stability}
\label{appdx:psr}

The following notations for predictive state representations (PSRs) and the B-stability condition are extracted from \citep{chen2022partially}.

\paragraph{Sequential decision processes with observations} 
An episodic sequential decision process is specified by a tuple $\set{H,\cO,\cA,\PP,\{r_{h}\}_{h \in [H]}}$, where $H\in\Z_{\ge 1}$ is the horizon length; $\cO$ is the observation space; $\cA$ is the action space; $\PP$ specifies the transition dynamics, such that the initial observation follows $o_1\sim \PP_0(\cdot) \in \Delta(\mathcal{O})$, and given the \emph{history} $\tau_h\defeq (o_1,a_1,\cdots,o_h,a_h)$ up to step $h$, the observation follows $o_{h+1}\sim\PP(\cdot|\tau_{h})$;
$r_h:\cO\times\cA\to[0,1]$ is the reward function at $h$-th step, which we assume is a known deterministic function of $(o_h,a_h)$.

In an episodic sequential decision process, a policy $\pi = \{\pi_h: (\cO\times\cA)^{h-1}\times\cO\to\Delta(\cA) \}_{h \in [H]}$ is a collection of $H$ functions. At step $h\in[H]$, an agent running policy $\pi$ observes the observation $o_h$ and takes action $a_{h}\sim \pi_h(\cdot|\tau_{h-1}, o_h)\in\Delta(\cA)$ based on the history $(\tau_{h-1},o_h)=(o_1,a_1,\dots,o_{h-1},a_{h-1},o_h)$. The agent then receives their reward $r_{h}(o_h,a_h)$, and the environment generates the next observation $o_{h+1}\sim\PP(\cdot|\tau_h)$ based on $\tau_h=(o_1,a_1,\cdots,o_h,a_h)$ (if $h < H $). The episode terminates immediately after $a_H$ is taken. 

For any $\tau_h=(o_1,a_1,\cdots,o_h,a_h)$, we write 
\[
\begin{aligned}
\PP(\tau_h)\defeq&~ \PP(o_{1:h}|a_{1:h})=\prod_{h'\le h} \PP(o_{h'}|\tau_{h'-1}), \\
\pi(\tau_h)\defeq&~  \prod_{h'\le h} \pi_{h'}(a_{h'}|\tau_{h'-1}, o_{h'}), \\
\PP^{\pi}(\tau_h)\defeq&~ \PP(\tau_h)\times \pi(\tau_h).
\end{aligned}
\]
Then $\PP^{\pi}(\tau_h)$ is the probability of observing $\tau_h$ (for the first $h$ steps) when executing $\pi$.

\paragraph{PSR, core test sets, and predictive states} 
A \emph{test} $t$ is a sequence of future observations and actions (i.e. $t\in\Test:=\bigcup_{W \in\Z_{\ge 1}}\cO^W \times\cA^{W -1}$). For some test $t_h=(o_{h:h+W-1},a_{h:h+W-2})$ with length $W\ge 1$, we define the probability of test $t_h$ being successful conditioned on (reachable) history $\tau_{h-1}$ as $\PP(t_h|\tau_{h-1})\defeq \PP(o_{h:h+W-1}|\tau_{h-1};\doac(a_{h:h+W-2}))$,
i.e., the probability of observing $o_{h:h+W-1}$ if the agent deterministically executes actions $a_{h:h+W-2}$, conditioned on history $\tau_{h-1}$. We follow the convention that, if $\P^\pi(\tau_{h-1}) = 0$ for any $\pi$, then $\P(t | \tau_{h-1}) = 0$. 

\begin{definition}[PSR, core test sets, and predictive states]\label{def:core-test}
For any $h\in[H]$, we say a set $\Uh\subset\Test$ is a \emph{core test set} at step $h$ if the following holds: For any $W\in\Z_{\ge 1}$, any possible future (i.e., test) $t_h=(o_{h:h+W-1},a_{h:h+W-2})\in\cO^W\times\cA^{W-1}$, there exists a vector $b_{t_h,h}\in\mathbb{R}^{\Uh}$ such that 
\begin{align}
	\label{eqn:psr-def}
	\PP(t_h|\tau_{h-1})=\langle b_{t_h,h},[\PP(t|\tau_{h-1})]_{t\in\Uh}\rangle, \qquad \forall \tau_{h-1} \in \cT^{h-1}:=(\cO \times \cA)^{h-1}. 
\end{align}
We refer to the vector $\bq(\tau_{h-1})\defeq [\PP(t|\tau_{h-1})]_{t\in\Uh}$ as the \emph{predictive state} at step $h$ (with convention $\bq(\tau_{h-1})=0$ if $\tau_{h-1}$ is not reachable), and $\bq_0\defeq [\PP(t)]_{t\in\cU_{1}}$ as the initial predictive state. A (linear) PSR is a sequential decision process equipped with a core test set $\{ \Uh \}_{h \in [H]}$.
\end{definition}

Define $\QAh\defeq \{\a:(\o,\a)\in\Uh~\textrm{for some}~\o\in\bigcup_{W \in\mathbb{N}^+}\cO^{W} \}$ as the set of ``core actions'' (possibly including an empty sequence) in $\Uh$, with $\nUA\defeq \max_{h\in[H]}\nUAh$. Further define $\cU_{H+1}\defeq \set{o_{\dum}}$ for notational simplicity. The core test sets $(\Uh)_{h\in[H]}$ are assumed to be known and the same within a PSR model class.

\begin{definition}[PSR rank]
\label{def:rank}
Given a PSR, its \emph{PSR rank} is defined as $\dPSR:=\max_{h\in [H]} \rank(D_h)$, where $D_h:=\left[ \bd(\tau_{h} ) \right]_{\tau_{h}\in \cT^h}\in\R^{\Uhp\times \cT^{h}}$ is the matrix formed by predictive states at step $h\in[H]$.
\end{definition}
For POMDP, it is clear that $\dPSR\leq S$, regardless of the core test sets.

\paragraph{\Bpara}
\citep{chen2022partially} introduced the notion of \emph{\Bpara} of PSR, which plays a fundamental role in their general structural condition and their analysis.
\begin{definition}[\Bpara]
\label{def:Bpara}
A {\Bpara} of a PSR with core test set $(\Uh)_{h\in[H]}$ is a set of matrices $\{ (\BB_h(o_h,a_h)\in\R^{\Uhp\times\Uh})_{h,o_h,a_h}, \bq_0 \in \R^{\Uone} \}$ such that for any $0\leq h\leq H$, policy $\pi$, history $\tau_{h} = (o_{1:h}, a_{1:h})\in \cT^{h}$, and core test $t_{h+1} = (o_{h+1:h+W}, a_{h+1:h+W-1}) \in \Uhp$, the quantity $\PP(\tau_{h},t_{h+1})$, i.e.
the probability of observing $o_{1:h+W}$ upon taking actions $a_{1:h+W-1}$, admits the decomposition 
\begin{align}\label{eqn:psr-op-prod-h}
    \PP(\tau_{h},t_{h+1}) = \PP(o_{1:h+W} | \doac(a_{1:h+W-1}))
    = \e_{t_{h+1}}^\top \cdot \BB_{h:1}(\tau_{h}) \cdot \bq_0,
\end{align}
where $\e_{t_{h+1}} \in \R^{\Uhp}$ is the indicator vector of $t_{h+1} \in \Uhp$, and 
$$
\BB_{h:1}(\tau_{h}) \defeq \BB_{h}(o_{h},a_{h}) \BB_{h-1}(o_{h-1},a_{h-1}) \cdots \BB_{1}(o_{1},a_{1}).
$$
\end{definition}

Based on the B-representations of PSRs, \cite{chen2022partially} proposed the following structural condition for sample-efficient learning in PSRs.
\begin{definition}[B-stability \citep{chen2022partially}]
\label{def:B-stable}
A PSR is \emph{B-stable with parameter $\stab\ge 1$} (henceforth also \emph{$\stab$-stable}) if it admits a {\Bpara} such that for all step $h\in[H]$, policy $\pi$, and $x\in\R^{\Uh}$, we have
\begin{equation}\label{eqn:B-stable}
\sum_{\tau_{h:H}=(o_h,a_h,\cdots,o_H,a_H)} \pi(\tau_{h:H})\times \abs{\B_{H}(o_H,a_H)\cdots \B_h(o_h,a_h) x} \le \stab\max\set{\nrmst{x},\nrmpip{x}},
\end{equation}
where for any vector $x = (x(t))_{t \in \Uh}$, we denote its $(1, 2)$-norm by
\[\textstyle
\nrmst{x}\defeq \big(\sum_{\a \in \UAh } \big(\sum_{\o: (\o, \a) \in \Uh} \vert x(\o,\a) \vert \big)^2 \big)^{1/2},
\]
and its $\Pi'$-norm by
\[\textstyle
\nrmpip{x}\defeq \max_{\barpi}\sum_{t\in \oUh}\barpi(t)\abs{x(t)},
\]
where $\oUh \defeq \{ t \in \Uh: \not\exists t' \in \Uh \text{ such that $t$ is a prefix of $t'$}\}$.
\end{definition}
Equivalently, \eqref{eqn:B-stable} can be written as $\nrmpi{\cB x}\leq \stab\max\set{\nrmst{x},\nrmpip{x}}$,
where for each step $h$, vector $x\in\R^{\Uh}$, we write
\begin{align}\label{eqn:B-op}
    \nrmpi{\cB_{H:h} x}\defeq \max_{\pi} \sum_{\tau_{h:H}} \pi(\tau_{h:H})\times \abs{\B_{H:h}(\tau_{h:H}) x}.
\end{align}

\citet{chen2022partially} showed that B-stability enables sample efficiency of PAC-learning, and we summarize the results in the following theorem.

\begin{theorem}[PAC upper bound for learning PSRs]
\label{thm:PSR-PAC-upper}
Suppose $\Theta$ is a PSR class with the same core test sets $\{ \Uh\}_{h \in [H]}$, and each $\theta\in\Theta$ admits a {\Bpara} that is $\stab$-stable and has PSR rank at most $d$. Then there exists algorithms (\omle/\eetod/\mops) that can find an $\epsilon$-optimal policy with probability at least $1-\delta$, within 
\begin{align}
\label{eqn:PSR-PAC-upper}
    T\le \tbO{\frac{\stab^2 d A U_AH^2\log(\Nt(1/T)/\delta)}{\epsilon^2}}
\end{align}
episodes of play, where $\cN_\Theta$ is the covering number of $\Theta$ (cf. \citet[Definition A.4]{chen2022partially}).
\end{theorem}

When $\Theta$ is a subclass of POMDPs, we have $\log \cN_\Theta(1/T)=\tbO{H(S^2A+SO)}$ \citep{chen2022partially}. Therefore, to deduce \cref{thm:PAC-upper} from the above general theorem, it remains to upper bound $\stab$ for $m$-step $\arev$-revealing POMDPs, which is done in \cref{appendix:proof-PAC-upper}.

\section{Proofs for Section~\ref{section:prelim}}

\subsection{Proof of Proposition~\ref{prop:m-step-mp1-step}}
\label{appendix:proof-m-step-mp1-step}

Fix any POMDP $M$, and we first show that $\arevmp(M)\geq \arevm(M)$. By the definition of $\arevmp(M)$ (\cref{definition:m-step-revealing}), it suffices to show the following result.
\begin{lemma}\label{lemma:rev-inc-step}
For any $h\in[H-m]$, and any choice of generalized left inverse $\Mhm^+$ (of $\Mhm$), the matrix $\Mhmp$ admits a generalized left inverse $\Mhmp^+$ such that
\begin{align*}
\nrmop{\Mhmp^+} \le \nrmop{\Mhm^+}.
\end{align*}
\end{lemma}
The converse part of \cref{prop:m-step-mp1-step} can be shown directly by examples. In particular, our construction in \cref{appdx:no-regret} readily provides such an example (see \cref{remark:mp-to-m-rev}). %

\begin{proof}[Proof of \cref{lemma:rev-inc-step}]
Fix an arbitrary action $\wt{a}\in\cA$. Consider the matrix $F_{\wt{a}}\in\R^{\cO^m\cA^{m-1}\times \cO^{m+1}\cA^m}$ defined as (the unique matrix associated with) the following linear operator:
\begin{align*}
\brac{F_{\wt{a}}\bx}(\o_{h:h+m-1},\a_{h:h+m-2}) \defeq \sum_{o\in \cO} \bx(\o_{h:h+m-1}o,\a_{h:h+m-2}\wt{a}), \quad \textrm{for all} \; \bx\in\R^{\cO^{m+1}\cA^m}.
\end{align*}
We first show that $F_{\wt{a}}\Mhmp=\Mhm$. Indeed,
\begin{align*}
& \quad \brac{F_{\wt{a}} \Mhmp}_{\o_{h:h+m-1}\a_{h:h+m-2}, s} = \sum_{o\in \cO} \brac{\Mhmp}_{(\o_{h:h+m-1}o)(\a_{h:h+m-2}\wt{a}), s} \\
& = \sum_{o\in \cO} \P\paren{ o_{h:h+m}=\o_{h:h+m-1}o|a_{h:h+m-1}=\a_{h:h+m-2}\wt{a}, s_h=s } \\
& = \P\paren{ o_{h:h+m-1}=\o_{h:h+m-1}|a_{h:h+m-2}=\a_{h:h+m-2}, s_h=s } = \brac{\Mhm}_{\o_{h:h+m-1}\a_{h:h+m-2}, s}
\end{align*}
for any $(\o_{h:h+m-1}\a_{h:h+m-2}, s)$, which verifies the claim. Therefore, for any generalized left inverse $\Mhm^+$, we can take
\begin{align*}
\Mhmp^+ \defeq \Mhm^+F_{\wt{a}}.
\end{align*}
This matrix satisfies $\Mhmp^+\Mhmp \T_{h-1}=\Mhm^+F_{\wt{a}}\Mhmp \T_{h-1}=\Mhm^+\Mhm \T_{h-1}=\T_{h-1}$ and is thus indeed a generalized left inverse of $\Mhmp$. Further,
\begin{align*}
\nrmop{\Mhmp^+} = \nrmop{\Mhm^+F_{\wt{a}}} \le \nrmop{\Mhm^+}\nrmstst{F_{\wt{a}}},
\end{align*}
so it remains to show that $\nrmstst{F_{\wt{a}}}\le 1$. To see this, note that for any $\bx\in\R^{\cO^{m+1}\cA^m}$ with $\nrmst{\bx}^2\le 1$, we have
\begin{align*}
& \quad \nrmst{F_{\wt{a}}\bx}^2 = \sum_{\a_{h:h+m-2}\in\cA^{m-1}} \paren{ \sum_{\o_{h:h+m-1}\in \cO^{m}} \abs{\sum_{o\in\cO} \bx(\o_{h:h+m-1}o,\a_{h:h+m-2}\wt{a})} }^2 \\
& \le \sum_{\a_{h:h+m-2}\in\cA^{m-1}} \paren{ \sum_{\o_{h:h+m}\in \cO^{m+1}} \abs{\bx(\o_{h:h+m},\a_{h:h+m-2}\wt{a})} }^2 \\
& \le \sum_{\a_{h:h+m-1}\in\cA^{m}} \paren{ \sum_{\o_{h:h+m}\in \cO^{m+1}} \abs{\bx(\o_{h:h+m},\a_{h:h+m-1})} }^2 = \nrmst{\bx}^2.
\end{align*}
This proves $\nrmstst{F_{\wt{a}}}\le 1$ and thus the desired result.
\end{proof}

\subsection{Proof of Theorem~\ref{thm:PAC-upper}}\label{appendix:proof-PAC-upper}

We will deduce \cref{thm:PAC-upper} from the general result (\cref{thm:PSR-PAC-upper}) of learning PSRs \citep{chen2022partially}.
To apply \cref{thm:PSR-PAC-upper}, we first invoke the following proposition, which basically states that any $m$-step $\arev$-revealing POMDP is B-stable with $\stab\le \arev^{-1}$.

\begin{proposition}\label{prop:rev-to-psr}
Any $m$-step $\arev$-revealing POMDP is a $\arev^{-1}$-stable PSR with core test set $\Uh=(\cO\times\cA)^{\min\set{m-1,H-h}}\times\cO$, i.e. it admits a $\stab\le \arev^{-1}$-stable \Bpara.
\end{proposition}
Therefore, for $\cM$ a class of $m$-step $\arev$-revealing POMDPs, $\cM$ is also a class of PSRs with common core test sets, such that each $M\in\cM$ is $\arev^{-1}$-stable, has PSR rank at most $S$ and $U_A=A^{m-1}$. Then, \cref{thm:PSR-PAC-upper} implies that an $\epsilon$-optimal policy of $\cM$ can be learned using \omle, \eetod, or \mops, with sample complexity
\begin{align*}
    \tbO{\frac{SA^mH^2\log(\cN_{\cM}(1/T)/\delta)}{\alpha^2\epsilon^2}},
\end{align*}
and we also have $\log\cN_{\cM}(1/T)=\tbO{H(S^2A+SO)}$ \citep{chen2022partially}. Combining these facts completes the proof of \cref{thm:PAC-upper}. \qed

\begin{proof}[Proof of \cref{prop:rev-to-psr}]
\citet[Appendix B.3.3]{chen2022partially} showed that any $m$-step $\arev$-revealing POMDP $M$ is a $\arev^{-1}$-stable PSR with core test set $\Uh=(\cO\times\cA)^{\min\set{m-1,H-h}}\times\cO$, and explicitly constructed the following \Bpara~for it:
when $h\leq H-m$, set
\begin{align}\label{eqn:Bpara-rev-1}
    \B_h(o,a)=\M_{h+1}\T_{h,a}\diag\paren{\O_h(o|\cdot)}\M_h^+, \qquad h\in[H-m],
\end{align}
and when $h>H-m$, take
\begin{align}\label{eqn:Bpara-rev-2}
    \BB_h(o_h,a_h)=\brac{\II\paren{t_h=(o_h,a_h,t_{h+1})}}_{(t_{h+1},t_h)\in\Uhp\times\Uh}\in\R^{\Uhp\times \Uh},
\end{align}
where $\II\paren{t_h=(o_h,a_h,t_{h+1})}$ is 1 if $t_h$ equals to $(o_h,a_h,\thp)$, and 0 otherwise. 

Then, by \citet[Lemma B.13]{chen2022partially}, for any $1\leq h\leq H$, $x\in\R^{\nUh}$, it holds that
\begin{align*}
    \nrmpi{\cB_{H:h} x} = \max_{\pi}\sum_{\tau_{h:H}} \nrm{ \B_{H}(o_{H},a_{H})\cdots\B_{h}(o_h,a_h) x }_1\times \pi(\tau_{h:H}) 
    \le \max\set{\lone{\M_h^+ x}, \nrmpip{x}} \leq \arev^{-1 }\max\set{\nrmst{x}, \nrmpip{x}}.
\end{align*}
Therefore, \Bpara~provided in \eqref{eqn:Bpara-rev-1} and \eqref{eqn:Bpara-rev-2} is indeed $\arev^{-1}$-stable, and hence completes the proof.
\end{proof}

\section{Basics of Ingster's method}\label{appendix:ingster-method}

In this section, we first introduce the basic notations frequently used in our analysis of hard instances, and then state Ingster's method for proving information-theoretic lower bounds \cite{ingster2012nonparametric}. Recall that we have introduced the formulation of sequential decision process in \cref{appdx:psr}.

\paragraph{Algorithms for sequential decision processes} An algorithm $\fA$ for sequential decision processes (with a fixed number of episodes $T$) is specified by a collection of $HT$ functions $\fA=\{\pi^{\fA}_{t,h} \}_{h \in [H], t \in [T]}$, where $\pi^{\fA}_{t,h}$ maps the tuple of all past histories and the current observation $(\tau^{(1)},\cdots,\tau^{(t-1)},\tau_{h-1}^{(t)},o_h^{(t)})$ to a distribution over actions $\Delta(\cA)$ from which we sample the next action $a_h^{(t)} \sim \pi^{\fA}_{t,h}(\cdot | \tau^{(1:t-1)},\tau_{h-1}^{(t)},o_h^{(t)})$. At the end of interaction, the algorithm output a $\piout\in\Pi$ by taking $\piout=\pi^{\fA}_{\sf output}(\tau^{1:T})$.

For any algorithm $\fA$ (with a fixed number of episodes $T$), we write $\PP^{\fA}_M$ to be the law of $(\tau^{(1)},\tau^{(2)},\cdots,\tau^{(T)})$ under the model $M$ and the algorithm $\fA$. We remark that although our formulation seems only to allow deterministic algorithms where each $\pi^{\fA}_{t,h}$ is a deterministic mapping to $\Delta(\cA)$, our formulation indeed allows randomized algorithms: any randomized algorithm can be written as a mixture of deterministic algorithm $\fB(\omega)$ parameterized by $\omega$ which satisfies a distribution $\omega \sim \zeta$; furthermore, for any $\fB(\omega)$ and $\zeta$, there exists a deterministic algorithm $\fA$ such that the marginal laws of $\tau^{1:T}$ induced by $\fB$ and $\fA$ are the same, i.e., $\EE_{\omega \sim \zeta}[\PP^{ \fB(\omega)}_M(\, \cdot \,)] = \PP^{\fA}_M(\, \cdot \,)$.

\newcommand{\TRUE}{\mathsf{TRUE}}
\newcommand{\FALSE}{\mathsf{FALSE}}

\paragraph{Algorithms with a random stopping time} Our analysis requires us to consider algorithms with a random stopping time. An algorithm $\fA$ with a random stopping time (with at most $T$ interaction) is specified by a collection of $HT$ functions $\{\pi^{\fA}_{t,h}\}_{h \in [H], t \in [T]}$ along with an exit criterion $\cexit$, where $\pi^{\fA}_{t,h}$ is the strategy at $t$-th episode and $h$-th step, and $\cexit$ is a deterministic function such that
\begin{align*}
\cexit(\tau^{(1)},\cdots,\tau^{(t)})\in\set{\TRUE,\FALSE}.
\end{align*}
Once $\cexit(\tau^{(1)},\cdots,\tau^{(\stoptime)})=\TRUE$ or $\stoptime=T$, the algorithm $\fA$ terminates at the end of the $\stoptime$-th episode. The random variable $\stoptime$ (induced by the exit criterion $\cexit$) is clearly a stopping time. We write $\PP^{\fA}_M$ to be the law of $(\tau^{(1)},\tau^{(2)},\cdots,\tau^{(\stoptime)})$ under the model $M$ and the algorithm $\fA$.

The following lemma and discussions hold for algorithms with or without a random stopping time.

\begin{lemma}[Ingster's method]\label{lemma:ingster}
    For a family of sequential decision processes $(\P_M)_{M\in\cM}$, a distribution $\zeta$ over $\cM$, a reference model $0\in\cM$, and an algorithm $\fA$ that interacts with the environment for $\stoptime$ episodes (where $\stoptime$ is  stopping time), it holds that
    \begin{align*}
    1+\chis{\EE_{M\sim\zeta}\brac{\PP_{M}^{\fA}}}{\PP_{0}^{\fA}}
    =
    \EE_{M,M'\sim_{\rm iid} \zeta}\EE_{\tau^{(1)},\cdots,\tau^{(\stoptime)}\sim\PP_{0}^{\fA}}\brac{
        \prod_{t=1}^{\stoptime}\frac{\PP_{M}(\taut)\PP_{M'}(\taut)}{\PP_{0}(\taut)^2}
    }.
    \end{align*}
\end{lemma}
\begin{proof}
We only need to consider the case $\fA$ has a random stopping time $\stoptime$. By our definition, $\PP_M^\fA$ is supported on the following set:
\begin{align*}
    \Omega_0\defeq \set{ \omega = \tau^{(1:\stoptime)}: \forall t < \stoptime, \cexit(\tau^{(1:t)}) =\FALSE, \text{ and either }\stoptime=T \text{ or } \cexit(\tau^{(1:\stoptime)}) = \TRUE}.
\end{align*}
For any $(\tau^{(1)},\cdots,\tau^{(\stoptime)})\in\Omega_0$, we have %
\begin{align}\label{eqn:rule-expand}
\begin{aligned}
    \PP_{M}^{\fA}(\tau^{(1)},\cdots,\tau^{(\stoptime)})
    =&~\prod_{t=1}^{\stoptime} \PP_M^{\fA}(\taut|\tau^{(1:t-1)})\\ 
    =&~ \prod_{t=1}^{\stoptime} \prod_{h=1}^H \PP_M(o_h^{(t)}|\taut_{1:h})\times\pi_{t,h}^{\fA}(a_h^{(t)}|\tau^{(1:t-1)},\taut_{1:h},o_h^{(t)})\\
    =&~\prod_{t=1}^{\stoptime} \PP_M(\taut)\times \prod_{t=1}^{\stoptime}\prod_{h=1}^H \pi_{t,h}^{\fA}(a_h^{(t)}|\tau^{(1:t-1)},\taut_{1:h},o_h^{(t)}). 
\end{aligned}
\end{align}

Therefore, by definition of $\chi^2$ divergence, we have
\begin{align*}
    1+\chis{\EE_{M\sim\zeta}\brac{\PP_{M}^{\fA}}}{\PP_{0}^{\fA}}
    =&~
    \EE_{\tau^{(1)},\cdots,\tau^{(\stoptime)}\sim\PP_{0}^{\fA}}\brac{\paren{
        \frac{\EE_{M\sim\zeta}\brac{\PP_{M}^{\fA}(\tau^{(1)},\cdots,\tau^{(\stoptime)})}}{\PP_{0}^{\fA}(\tau^{(1)},\cdots,\tau^{(\stoptime)})}
    }^2}\\
    =&~
    \EE_{M,M'\sim\zeta}\EE_{\tau^{(1)},\cdots,\tau^{(\stoptime)}\sim\PP_{0}^{\fA}}\brac{
        \frac{\PP_{M}^{\fA}(\tau^{(1)},\cdots,\tau^{(\stoptime)})\PP_{M'}^{\fA}(\tau^{(1)},\cdots,\tau^{(\stoptime)})}{\PP_{0}^{\fA}(\tau^{(1)},\cdots,\tau^{(\stoptime)})^2}
    }\\
    =&~
    \EE_{M,M'\sim\zeta}\EE_{\tau^{(1)},\cdots,\tau^{(\stoptime)}\sim\PP_{0}^{\fA}}\brac{
        \prod_{t=1}^{\stoptime}\frac{\PP_{M}(\taut)\PP_{M'}(\taut)}{\PP_{0}(\taut)^2}
    },
\end{align*}
where the last equality is due to \eqref{eqn:rule-expand}. This proves the lemma. 
\end{proof}

Therefore, in order to upper bound $\chis{\EE_{M\sim\zeta}\brac{\PP_{M}^{\fA}}}{\PP_{0}^{\fA}}$, we just need to upper bound the quantity 
\begin{align}\label{eqn:ingster-alg-basic-eq}
    \EE_{\tau^{(1)},\cdots,\tau^{(\stoptime)}\sim\PP_{0}^{\fA}}\brac{
        \prod_{t=1}^{\stoptime}\frac{\PP_{M}(\taut)\PP_{M'}(\taut)}{\PP_{0}(\taut)^2}
    }
    =
    \EE_{\tau^{(1)},\cdots,\tau^{(\stoptime)}\sim\PP_{0}^{\fA}}\brac{
        \prod_{t=1}^{\stoptime}\prod_{h=1}^H \frac{\PP_{M}(o_h^{(t)}|\taut_{h-1})\PP_{M'}(o_h^{(t)}|\taut_{h-1})}{\PP_{0}(o_h^{(t)}|\taut_{h-1})^2}
    }.
\end{align}
At this aim, we will leverage the following fact (which is due to \cref{lemma:martingale-eqn} and \eqref{eqn:ingster-alg-basic-eq}): 
\begin{align}\label{eqn:ingster-prod-to-i}
    \EE_{\tau^{(1)},\cdots,\tau^{(\stoptime)}\sim\PP_{0}^{\fA}}\brac{
        \prod_{t=1}^{\stoptime}\frac{\PP_{M}(\taut)\PP_{M'}(\taut)}{\PP_{0}(\taut)^2}\cdot
        \exp\paren{-\sum_{t=1}^\stoptime\sum_{h=1}^H \log I_{M,M'}(\taut_{h-1})}
    }=1,
\end{align}
where $I_{M,M'}(\tau_{h-1})$ is defined as
\begin{align}\label{eqn:ingster-i-func}
    I_{M,M'}(\tau_{h-1})\defeq \EE_{0}\cdbrac{\frac{\PP_{M}(o_h|\tau_{h-1})\PP_{M'}(o_h|\tau_{h-1})}{\PP_{0}(o_h|\tau_{h-1})^2}}{\tau_{h-1}}. 
\end{align}

\paragraph{Early stopped algorithm} Consider an algorithm $\fA$ that interacts with the environment for a fixed number of episodes $T$ and consider an exit criterion $\cexit$. We define the early stopped algorithm $\fA(\cexit)$, which executes the algorithm $\fA$ until $\cexit=\TRUE$ is satisfied (or $T$ is reached). Clearly, $\fA(\cexit)$ is an algorithm with a random stopping time. We have the following lemma regarding how much the TV distance $\DTV{ \EE_{M\sim\zeta}\brac{\PP_{M}^{\fA}}, \PP_{0}^{\fA}}$ is perturbed after changing the algorithm $\fA$ to its stopped version $\fA(\cexit)$. %
\begin{lemma}\label{lemma:alg-stop}
It holds that
\begin{align*}
    \DTV{ \EE_{M\sim\zeta}\brac{\PP_{M}^{\fA(\cexit)}}, \PP_{0}^{\fA(\cexit)}}
    \geq \DTV{ \EE_{M\sim\zeta}\brac{\PP_{M}^{\fA}}, \PP_{0}^{\fA}} - \PP_0^\fA(\exists t < T, \cexit(\tau^{(1:t)})=\TRUE).
\end{align*}
\end{lemma}
\begin{proof}
    We consider the event $\Omega=\{\omega = \tau^{(1:T)}: \forall t < T, \cexit(\tau^{(1:t)}) =\FALSE \}$. To prove this lemma, we only need to verify that $\PP_M^{\fA}|_{\Omega}=\PP_M^{\fA(\cexit)}|_{\Omega}$ and then apply \cref{lemma:change-alg}.

   Indeed, for $\omega = \tau^{(1:T)} \in \Omega$, we have that for all $t< T$,  $\cexit(\tau^{(1:t)})=\FALSE$. Then, by \eqref{eqn:rule-expand} we have
\begin{align*}
    \PP_{M}^{\fA(\cexit)}(\tau^{(1:T)})
    =&~\prod_{t=1}^{T} \PP_M(\taut)\times \prod_{t=1}^{T}\prod_{h=1}^H \pi_{t,h}^{\fA}(a_h^{(t)}|\tau^{(1:t-1)},\taut_{1:h},o_h^{(t)})
    =\PP_{M}^{\fA}(\tau^{(1:T)}),
\end{align*}
and thus $\PP_{M}^{\fA(\cexit)}(\omega)=\PP_{M}^{\fA}(\omega)$ for any $\omega\in\Omega$. Applying \cref{lemma:change-alg} proves the lemma. 
\end{proof}

\section{Proof of Theorem \ref{thm:1-step-demo}}\label{appdx:1-step-pac}

We first construct a family of hard instances in~\cref{appendix:construction-1-step}. We state the PAC lower bound of this family of hard instances in~\cref{prop:1-step-pac-prop}.~\cref{thm:1-step-demo} then follows from~\cref{prop:1-step-pac-prop} as a direct corollary. 

\subsection{Construction of hard instances and proof of Theorem~\ref{thm:1-step-demo}}
\label{appendix:construction-1-step}

We consider the following family of single-step revealing POMDPs $\cM$ that admits a tuple of hyperparameters $(\epsilon,\sigma,n,K,H)$. All POMDPs in $\cM$ have the same horizon length $H$, the state space $\cS$, the action space $\cA$, and the observation space $\cO$, defined as follows.
\begin{itemize}[topsep=0pt, itemsep=0pt]
    \item The state space $\cS=\Stree \bigsqcup \set{\sp, \sq}$, where $\Stree$ is a binary tree with level $n$ (so that $\abs{\Stree}=2^{n}-1$).  
    Let $s_0$ be the root of $\Stree$, and $\Sl$ be the set of leaves of $\Stree$, with $\abs{\Sl}=2^{n-1}$.%
    \item The observation space $\cO=\Stree \bigsqcup \set{o_1^+,o_1^-,\cdots,o_K^+,o_K^-}\bigsqcup\set{\oba,\obb}$.
    Note that here we slightly abuse notations, reusing $\Stree$ to denote both a set of states and the corresponding set of observations, in the sense that each state $s\in\Stree\subset \cS$ corresponds to a unique observation $o_s\in\Stree\subset \cO$, which we also denote as $s$ when it is clear from the context.%
    \item The action space $\cA=\set{0,1,\cdots,A-1}$.
\end{itemize}

\paragraph{Model parameters} 
Each non-null POMDP model $M=M_{\theta,\mu}\in\cM \setminus \{ M_0\}$ is specified by two parameters $(\theta, \mu)$. Here $\mu\in\set{-1,+1}^{K}$, and $\theta=(\hs,\ss,\acs,\as)$, where 
\begin{itemize}[topsep=0pt, itemsep=0pt]
    \item $\ss\in\Sl$, $\acs\in \Ac\defeq\set{1,\cdots,A-1}$.
    \item $\hs\in\set{n+1,\cdots,H-1}$.
    \item $\as=(\as_{\hs+1},\dots,\as_{H-1})\in\cA^{H-\hs-1}$ is an action sequence indexed by $\hs+1,\cdots,H-1$.%
\end{itemize}

For any POMDP $M_{\theta, \mu}$, its emmision and transition dynamics $\PP_{\theta,\mu}\defeq \PP_{M_{\theta, \mu}}$ are defined as follows.

\paragraph{Emission dynamics} 
\begin{itemize}[topsep=0pt, itemsep=0pt]
\item At states $s\in\Stree$, the agent always receives (the unique observation corresponding to) $s$ itself as the observation.
    \item At state $\sp$ and steps $h<H$, the emission dynamics is given by
    \begin{align*}
        \O_{h;\mu}(o_i^+|\sp)=\frac{1+\sigma\mu_{i}}{2K}, \qquad 
        \O_{h;\mu}(o_i^-|\sp)=\frac{1-\sigma\mu_{i}}{2K}, \qquad
        \forall i\in[K].
    \end{align*}
    \item At state $\sq$ and steps $h<H$, the observation is uniformly drawn from $\cO_o\defeq \set{o_1^+,o_1^-,\cdots,o_K^+,o_K^-}$:
    \begin{align*}
    \O_{h}(o_i^+|\sq) = \O_{h}(o_i^-|\sq) = \frac{1}{2K}, \qquad \forall i\in[K].
    \end{align*}
    Here we omit the subscript $\mu$ to emphasize that the dynamic does not depend on $\mu$.
    \item At step $H$, the emission dynamics at $\set{\sp,\sq}$ is given by
    \begin{align*}
        &\O_H(\oba|\sp)=\frac34, \qquad \O_H(\obb|\sp)=\frac14, \\
        &\O_H(\oba|\sq)=\frac14, \qquad \O_H(\obb|\sq)=\frac34.
    \end{align*}
\end{itemize}

\paragraph{Transition dynamics} In each episode, the agent always begins at $s_0$. 
\begin{itemize}[topsep=0pt, itemsep=0pt]
    \item At any node $s\in\Stree\setminus \Sl$, there are three types of available actions: $\await=0$, $\aleft=1$ and $\aright=2$, such that the agent can take $\await$ to stay at $s$, $\aleft$ to transit to the left child of $s$, and $\aright$ to transit to the right child of $s$.\footnote{
    For action $a\in\set{3,\cdots,A-1}$, $a$ has the same effect as $\await$.
    }
    \item At any $s\in\Sl$, the agent can take action $\await=0$ to stay at $s$ (i.e. $\PP(s|s,\await)=1$); otherwise, for $s\in\Sl$, $h\in[H-1]$, $a\neq \await$ (i.e. $a\in\Ac$),
    \begin{align*}
        \PP_{h;\theta}(\sp|s,a)&=\epsilon\cdot \II(h=\hs,s=\ss,a=\acs), \\
        \PP_{h;\theta}(\sq|s,a)&=1-\epsilon\cdot \II(h=\hs,s=\ss,a=\acs),
    \end{align*}
    where we use subscript $\theta$ to emphasize the dependence of the transition probability $\P_{h; \theta}$ on $\theta$. In words, at step $h$, state $s \in \Sl$, and after $a \in \Ac$ is taken, any leaf node will transit to one of $\set{\sp,\sq}$, and only taking $\acs$ at state $\ss$ and step $\hs$ can transit to the state $\sp$ with a small probability $\eps$; in any other case, the system will transit to the state $\sq$ with probability one.%
    
    \item At state $\sp$, we set
    \begin{align*}
        \PP_{h;\theta}(\sp|\sp,a)=\begin{cases}
            1, & a=\as_h,\\
            0, & a\neq \as_h,
        \end{cases},\qquad
        \PP_{h;\theta}(\sq|\sp,a)=\begin{cases}
            0, & a=\as_h,\\
            1, & a\neq \as_h.
        \end{cases}
    \end{align*}
    \item The state $\sq$ is an absorbing state, i.e. $\P_{h}(\sq|\sq,a)=1$ for all $a\in\cA$.
\end{itemize}

\paragraph{Reward} The reward function is known (and only depends on the observation): at the first $H-1$ steps, no reward is given; at step $H$, we set $r_H(\oba)=1$, $r_H(\obb)=0$, 
$r_H(s_0)=(1+\epsilon)/4$, and $r_H(o)=0$ for any other $o\in\cO$. 

\paragraph{Reference model}
We use $M_0$ (or simply $0$) to refer to the null model (reference model). The null model $M_0$ has transition and emission the same as any non-null model, except that the agent always arrives at $\sq$ by taking any action $a\neq\await$ at $s \in\Sl$ and $h \in [H - 1]$ (i.e., $\PP_{h;M_0}(\sq|s,a) = 1$ for any $s \in \Sl$, $a \in \Ac$, $h \in [H-1]$). In this model, $\sp$ is not reachable, and hence we do not need to specify the emission dynamics at $\sp$.

We present the PAC-learning sample complexity lower bound of the above POMDP model class $\cM$ in the following proposition, which we prove in~\cref{appendix:proof-1-step-pac-prop}.

\begin{proposition}\label{prop:1-step-pac-prop}
For given $\epsilon \in(0,0.1], \sigma\in(0,\frac1{2H}]$, $n\geq 1$, $K\geq 1$, $H\geq 4n$, the model class $\cM$ we construct above satisfies the following properties: 

\begin{enumerate}[topsep=0pt, itemsep=0pt]
\item $\nS=2^n+1$, $\nO=2^n+2K+1$, $\nA=A$.

\item For each $M\in\cM$ (including the null model $M_0$), $M$ is single-step revealing with $\arevone(M)^{-1}\leq 1+\frac2\sigma$.

\item $\log\abs{\cM}\leq K\log 2 + H\log A + \log(SAH)$. 

\item Suppose algorithm $\fA$ interacts with the environment for $T$ episodes and returns $\piout$ such that %
$$
\PP^{\fA}_M\paren{V_M^\star-V_M(\piout)<\frac{\epsilon}{8}}\geq \frac{3}{4}
$$
for any $M\in\cM$. Then it must hold that
\begin{align*}
    T\geq \frac{1}{20000}\min\set{
        \frac{\abs{\Sl}K^{1/2}AH}{\sigma^2\epsilon^2}, \frac{\abs{\Sl}A^{H/2}H}{\epsilon^2}
    },
\end{align*}
where we recall that $\abs{\Sl}=2^{n-1}$.
\end{enumerate}
\end{proposition}

\begin{proofof}[thm:1-step-demo]
In \cref{prop:1-step-pac-prop}, suitably choosing $\sigma, n, K$, and choosing a rescaled $\epsilon$, we obtain \cref{thm:1-step-demo}. More specifically, we can take $n\geq 1$ to be the largest integer such that $2^n\leq \min\set{S-1,(O-1)/2}$, and take $K=\floor{\frac{O-2^n-1}{2}}\geq \frac{O-1}{4}$, $\epsilon'=\epsilon/8$, and $\sigma=\frac{2}{\arev^{-1}-1}\leq \frac{1}{2H}$. Applying \cref{prop:1-step-pac-prop} to the parameters $(\epsilon',\sigma,n,K,H)$ completes the proof of \cref{thm:1-step-demo}.
\end{proofof}

\subsection{Proof of Proposition \ref{prop:1-step-pac-prop}}
\label{appendix:proof-1-step-pac-prop}

All propositions and lemmas stated in this section are proved in~\cref{appendix:proof-1-step-pac-rev}-\ref{appendix:proof-1-step-pac-MGF}.

Claim 1 follows directly by the counting the number of states, observations, and actions in construction of $\cM$. Claim 3 follows as we have $\abs{\cM}=\abs{\set{(h^\star,s^\star,a^\star,\a^\star)}}\times | \{\pm 1 \}^K | + 1 \le HSA\times A^H\times 2^K$. Taking logarithm yields the claim.

Claim 2 follows directly by the following proposition with proof in \cref{appendix:proof-1-step-pac-rev}.

\begin{proposition}
\label{prop:1-step-pac-rev}
 For any $M\in\cM$, $M$ is single-step revealing with $\arevone(M)^{-1} \leq \frac2\sigma+1$.   
\end{proposition}

We now prove Claim 4 (the sample complexity lower bound). We begin by using the following lemma to relate the PAC learning problem to a testing problem, using the structure of $\cM$. Intuitively, the lemma states that a near-optimal policy of any $M\neq 0$ cannot ``stay'' at $s_0$, whereas a near-optimal policy of model $M=0$ has to ``stay'' at $s_0$. The proof of the lemma is contained in \cref{appendix:proof-1-step-pac-to-testing}. 

\begin{lemma}[Relating policy suboptimality to the probability of staying]
\label{lemma:1-step-pac-value-func}
    For any $M\in\cM$ such that $M\neq 0$ and any policy $\pi$, it holds that
    \begin{align}
        V_M^\star-V_M(\pi)\geq \frac{\epsilon}{4}\PP^{\pi}_M\paren{o_H=s_0}.
    \end{align}
    On the other hand, for the reference model $0$ and any policy $\pi$, we have
    \begin{align}
        V_0^\star-V_0(\pi)\geq \frac{\epsilon}{4}\PP^{\pi}_0\paren{o_H\neq s_0}.
    \end{align}
\end{lemma}

Notice that the probability $\PP^{\pi}_M\paren{o_H=s_0}$ actually does not depend on the model $M\in\cM$, i.e.
$$
\PP^{\pi}_M\paren{o_H=s_0}=\PP^{\pi}_0\paren{o_H=s_0}.
$$
This is because once the agent leaves $s_0$, it will never come back (for any model $M\in\cM$). In the following, we define $w(\pi)\defeq \PP^{\pi}_0\paren{o_H=s_0}$. Note that $\piout$ is the output policy that depends on the observation histories $\tau^{1:T}$, and thus $w(\piout)$ is a deterministic function of the observation histories $\tau^{1:T}$.

By~\cref{lemma:1-step-pac-value-func} and our assumption that $\PP^{\fA}_M\paren{V_M^\star-V_M(\piout)<\frac{\epsilon}{8}}\geq \frac{3}{4}$ for any $M \in \cM$, we have
\begin{align*}
    \PP^{\fA}_0\paren{1-w(\piout)< \frac{1}{2}}\geq \frac34,
    \qquad
    \text{while } \qquad\PP^{\fA}_M\paren{w(\piout)< \frac12 }\geq \frac34, \qquad \forall M\neq 0.
\end{align*}
Now we consider $\mu\sim \Unif(\{ \pm 1\}^K)$ to be the uniform prior over the parameter $\mu$.
For any fixed $\theta$, we consider averaging the above quantity over the non-null models $M=(\theta,\mu)$ when $\mu\sim\Unif(\{ \pm 1\}^K)$, 
\begin{align*}
    \Emu\brac{\PP^{\fA}_{\theta,\mu}}\paren{w(\piout)< \frac12}
    =
    \Emu\brac{\PP^{\fA}_{\theta,\mu}\paren{w(\piout)< \frac12}}
    \geq \frac34.
\end{align*}
However, we also have
$$
\PP^{\fA}_0\paren{w(\piout)< \frac12}=1-\PP^{\fA}_0\paren{w(\piout)\geq \frac12}\leq 1-\PP^{\fA}_0\paren{w(\piout)> \frac12} \leq \frac14.
$$
Thus by the definition of TV distance we must have
\begin{align}\label{eqn:1-step-pac-TV-lower}
    \DTV{ \PP^{\fA}_0, \Emu\brac{\PP^{\fA}_{\theta,\mu}} } \geq 
    \abs{\PP^{\fA}_0\paren{w(\piout)< \frac12}-\Emu\brac{\PP^{\fA}_{\theta,\mu}}\paren{w(\piout)< \frac12}}
    \geq
    \frac12.
\end{align}

As the core of the proof, we now use \eqref{eqn:1-step-pac-TV-lower} to derive our lower bound on $T$.
Recall that $\PP^{\fA}_M$ is the law of $(\tau^{(1)},\tau^{(2)},\cdots,\tau^{(\stoptime)})$ induced by letting $\fA$ interact with the model $M$. For any event $E\subseteq (\cO\times\cA)^H$, we denote the visitation count of $E$ as
\begin{align*}
    N(E)\defeq \sum_{t=1}^T \II(\taut\in E). 
\end{align*}
Since $N(E)$ is a function of $\tau^{(1:T)}$, we can talk about its expectation under the distribution $\PP^{\fA}_M$ for any $M\in\cM$. We present the following lemma on the lower bound of the expected visitation count of some good events, whose proofs are contained in \cref{appendix:proof-1-step-pac-visitation-lower-bound}. %

\newcommand{\Erevh}{E_{\rev,h}}

\begin{lemma}\label{lemma:1-step-pac-alternative}
    Fix a $\theta=(\hs,\ss,\acs,\as)$. We consider events
    \begin{align*}
        E_{\rev,h}^\theta&\defeq \set{o_{\hs}=\ss, a_{\hs:h}=(\acs,\as_{\hs+1:h}) },~~~~~~~~~~ \forall h \in \{\hs + 1, \ldots,  H-2\},\\
        E_{\corr}^\theta&\defeq \set{o_{\hs}=\ss, a_{\hs:H-1}=(\acs,\as)}.
    \end{align*}
    Then for any algorithm $\fA$ with $\delta\defeq \DTV{\PP_{0}^{\fA},\Emu\brac{\PP_{\theta,\mu}^{\fA}}}>0$, we have %
    \begin{align*}
        \text{either }\sum_{h=\hs}^{H-2}\EE^{\fA}_0\brac{N(\Erevh^\theta)} \geq \frac{\delta^3\sqrt{K}}{54\epsilon^2\sigma^2}-\frac{H\delta}6, \quad \text{  or }\quad\EE^{\fA}_0\brac{N(\Ecor^\theta)}\geq \frac{\delta^3}{54\epsilon^2}-\frac\delta6.
    \end{align*}
\end{lemma}

Applying \cref{lemma:1-step-pac-alternative} for any parameter tuple $\theta=(\hs,\ss,\acs,\as)$ with $\delta=\frac12$, we obtain 
\begin{align}\label{eqn:1-step-pac-alternative}
    \text{either }\sum_{h=\hs}^{H-2}\EE^{\fA}_0\brac{\Nc{\Erevh^{(\hs,\ss,\acs,\as)}}} \geq \frac{\sqrt{K}}{1000\epsilon^2\sigma^2}, \quad \text{  or } \quad \EE^{\fA}_0\brac{\Nc{\Ecor^{(\hs,\ss,\acs,\as)}}}\geq \frac{1}{1000\epsilon^2},
\end{align}
by our choice that $\epsilon\in(0,0.1]$ and $\sigma\in(0,\frac{1}{2H}]$.

Fix a tuple $(\hs,\ss,\acs)$ with $\ss\in\Sl, \acs\in\Ac, \hs\in[n+1,\frac{H}{2}]$. By \eqref{eqn:1-step-pac-alternative}, we know that for all $\a\in\cA^{H-\hs-1}$, it holds that 
\begin{align}\label{eqn:1-step-pac-alternative-sum}
    \sum_{h=\hs}^{H-2}\EE^{\fA}_0\brac{N\paren{\Erevh^{(\hs,\ss,\acs,\a)}}}+A^{H-\hs-1}\cdot \EE^{\fA}_0\brac{N\paren{\Ecor^{(\hs,\ss,\acs,\a)}}}\geq \frac1{1000}\min\set{ \frac{\sqrt{K}}{\epsilon^2\sigma^2}, \frac{A^{H/2-1}}{\epsilon^2} }=:\omega.
\end{align}
Notice that by definition,
\begin{align*}
    \sum_{\a\in\cA^{H-\hs-1}}\EE^{\fA}_0\brac{N\paren{\Ecor^{(\hs,\ss,\acs,\a)}}}
    =&
    \sum_{\a\in\cA^{H-\hs-1}}\EE^{\fA}_0\brac{N\paren{o_{\hs}=\ss, a_{\hs:H-1}=(\acs,\a)}}\\
    =&
    \EE^{\fA}_0\brac{\sum_{\a\in\cA^{H-\hs-1}}N\paren{o_{\hs}=\ss, a_{\hs:H-1}=(\acs,\a)}}\\
    =&
    \EE^{\fA}_0\brac{N\paren{o_{\hs}=\ss, a_{\hs}=\acs}},
\end{align*}
and similarly for each $h\in[\hs,H-2]$, it holds
\begin{align*}
    \sum_{\a\in\cA^{H-\hs-1}}\EE^{\fA}_0\brac{N\paren{\Erevh^{(\hs,\ss,\acs,\a)}}}
    =&
    \sum_{\a\in\cA^{H-\hs-1}}\EE^{\fA}_0\brac{N\paren{o_{\hs}=\ss, a_{\hs:h}=(\acs,\a_{\hs+1:h})}}\\
    =&
    \sum_{\a_{\hs+1:h}\in\cA^{h-\hs}}\EE^{\fA}_0\brac{N\paren{o_{\hs}=\ss, a_{\hs:h}=(\acs,\a_{\hs+1:h})}}\cdot\sum_{\a_{h+1:H-1}\in\cA^{H-h-1}}1\\
    =&
    \EE^{\fA}_0\brac{N\paren{o_{\hs}=\ss, a_{\hs}=\acs}}\cdot A^{H-h-1}.
\end{align*}
Therefore, summing the bound~\cref{eqn:1-step-pac-alternative-sum} over all $\a\in\cA^{H-h^\star-1}$, we get
\begin{align*}
    A^{H-\hs-1}\omega
    =\sum_{\a\in\cA^{H-\hs-1}}\omega\leq& 
    \sum_{\a\in\cA^{H-\hs-1}} \brac{ \sum_{h=\hs}^{H-2}\EE^{\fA}_0\brac{N\paren{\Erevh^{(\hs,\ss,\acs,\a)}}}+A^{H-\hs-1}\cdot \EE^{\fA}_0\brac{N\paren{\Ecor^{(\hs,\ss,\acs,\a)}}} }\\
    =&
    \paren{\sum_{h=\hs}^{H-2} A^{H-h-1} + A^{H-\hs-1}}\EE^{\fA}_0\brac{N\paren{o_{\hs}=\ss, a_{\hs}=\acs}}\\
    \leq &
    3A^{H-\hs-1}\EE^{\fA}_0\brac{N\paren{o_{\hs}=\ss, a_{\hs}=\acs}},
\end{align*}
where the last inequality is due to $\sum_{h=\hs}^{H-2} A^{H-h-1} = \frac{A^{H-h}-A}{A-1}\leq 2A^{H-h-1}$ for $A\geq 3$.

Therefore, we have shown that $\EE^{\fA}_0\brac{N\paren{o_{\hs}=\ss, a_{\hs}=\acs}}\geq \frac{\omega}{3}$ for each $\ss\in\Sl, \acs\in\Ac, \hs\in[n+1,\frac{H}{2}]$. Taking summation over all such $(\hs,\ss,\acs)$, we derive that
\begin{align*}
    \abs{\Sl}\abs{\Ac} \paren{\floor{\frac{H}{2}}-n}\cdot\frac{\omega}{3} \leq \sum_{\ss\in\Sl}\sum_{\acs\in\Ac}\sum_{\hs=n+1}^{\floor{H/2}-1} \EE^{\fA}_0\brac{N\paren{o_{\hs}=\ss, a_{\hs}=\acs}}\leq T,
\end{align*}
where the second inequality is because events $\set{o_{\hs}=\ss, a_{\hs}=\acs}$ are disjoint. Plugging in $\abs{\Ac}=A-1, H\geq 4n$ and the definition of $\omega$ in \eqref{eqn:1-step-pac-alternative-sum} completes the proof of~\cref{prop:1-step-pac-prop}.
\qed

\subsection{Proof of Proposition \ref{prop:1-step-pac-rev}}
\label{appendix:proof-1-step-pac-rev}

We first consider the case $M=M_{\theta,\mu}$. At the step $h<H$, the emission matrix $\O_{h;\mu}$ can be written as (up to some permutation of rows and columns)
\begin{align*}
    \O_{h;\mu}=\begin{bmatrix}
        \frac{\II_{2K}+\sigma \tmu}{2K} & \frac{\II_{2K}}{2K} & 0_{2K \times \Stree}\\
        0_{\Stree \times 1}&0_{\Stree \times 1} & \id_{\Stree\times\Stree}\\
       0_{2 \times 1} & 0_{2 \times 1} & 0_{2 \times \Stree}
    \end{bmatrix}\in\R^{\cO\times\cS},
\end{align*}
where $\tmu=[\mu;-\mu]\in\set{-1,1}^{2K}$, and $\II=\II_{2K}$ is the column vector in $\R^{2K}$ with all entries being one. A simple calculation shows that 
$$
\brac{\frac{\II+\sigma \tmu}{2K},\frac{\II}{2K}}^{\dagger\top}=\brac{\frac{1}{\sigma}\tmu, \II-\frac{1}{\sigma}\tmu},
$$
whose $1$-norm is bounded by $\frac2\sigma+1$. Hence $\lone{\O_{h;\mu}^{\dagger}}\leq \frac2\sigma+1$.

Similarly, for $h=H$, $\O_H$ has the form (up to some permutation of rows and columns)
\begin{align*}
    \O_H=\begin{bmatrix}
        \frac34 & \frac14 & 0_{1 \times \Stree} \\
        \frac14 & \frac34 & 0_{1 \times \Stree} \\
       0_{\Stree \times 1} & 0_{\Stree \times 1} & I_{\Stree\times\Stree} \\
        0_{2K \times 1}& 0_{2K \times 1} & 0_{2K \times \Stree} \\
    \end{bmatrix}\in\R^{\cO\times\cS}.
\end{align*}
Notice that $\begin{bmatrix}
        \frac34 & \frac14 \\
        \frac14 & \frac34 
    \end{bmatrix}^{-1}=\begin{bmatrix}
        \frac32 & -\frac12 \\
        -\frac12 & \frac32 
    \end{bmatrix}$, and hence $\lone{\O_{H}^{\dagger}}\leq2$.

Finally, by~\cref{definition:m-step-revealing} and noting that $\Mhone=\O_h$ and taking the generalized left inverse $\Mhone^+=\O_h^\dagger$ to be the pseudo-inverse for all $h\in[H]$, this gives $(\arevone(M))^{-1}\le \max\set{\frac{1}{\sigma}+2, 2}=\frac{2}{\sigma}+1$.

We next consider the case $M=0$. In this case, $\sp$ is not reachable, and hence for each step $h$, we can consider the generalized left inverse of $\O_h$ given by
\begin{align*}
    \O_h^+\defeq \brac{ \IIc{ \O_h(o|s)>0 } }_{(s,o)}\in\R^{\cS\times\cO},
\end{align*}
with the convention that $\IIc{ \O_h(o|\sp)>0 }=0$ for all $o\in\cO$ as $\O_h(\cdot|\sp)$ is not defined. Then it is direct to verify $\O_h^+\O_h\e_s=\e_s$ for all state $s\neq \sp$ (because the supports $\supp(\O_h(\cdot|s))$ are disjoint by our construction). It is clear that $\loneone{\O_h^+}\leq 1$, and hence $(\arevone(M))^{-1}\leq 1$, which completes the proof.
\qed

\subsection{Proof of Lemma \ref{lemma:1-step-pac-value-func}}\label{appendix:proof-1-step-pac-to-testing}

By definition, for any model $M\in\cM$ and policy $\pi$,
\begin{align*}
    V_M(\pi)&=\EE^{\pi}_M\brac{r_H(o_H)}
    =\frac{1+\epsilon}{4}\PP_M^{\pi}(o_H=s_0)+\PP_M^{\pi}(o_H=\oba)\\
    &=\frac{1+\epsilon}{4}\PP_M^{\pi}(o_H=s_0)+\frac{3}{4}\PP_M^{\pi}(s_H=\sp)+\frac{1}{4}\PP_M^{\pi}(s_H=\sq),
\end{align*}
where we have used the following equality due to our construction:
\begin{align*}
\PP_M^{\pi}(o_H=\oba)
=&
\PP_M(o_H=\oba|s_H=\sp)\cdot \PP_M^{\pi}(s_H=\sp)
+\PP_M(o_H=\oba|s_H=\sq)\cdot \PP_M^{\pi}(s_H=\sq)\\
=&\frac{3}{4}\PP_M^{\pi}(s_H=\sp)+\frac{1}{4}\PP_M^{\pi}(s_H=\sq).
\end{align*}
We next prove the result for the case $M=0$ and $M\neq 0$ separately.

Case 1: $M=0$. In this case, $\sp$ is not reachable, and hence we have $V_0^\star=\max_{\pi}V_0(\pi)=\max\set{\frac{1+\epsilon}{4}, \frac{1}{4}}=\frac{1+\eps}{4}$, which is attained by staying at $s_0$. Thus, for any policy $\pi$,
\begin{align*}
    V_0^\star-V_0(\pi)
    =&\frac{1+\epsilon}{4}-\frac{1+\epsilon}{4}\PP_0^{\pi}(o_H=s_0)-\frac{1}{4}\PP_0^{\pi}(s_H=\sq)\\
    =&
    \frac{1+\epsilon}{4}\PP_0^{\pi}(o_H\neq s_0)-\frac{1}{4}\PP_0^{\pi}(s_H=\sq)\\
    =&\frac{1}{4}\paren{\PP_0^{\pi}(o_H\neq s_0)-\PP_0^{\pi}(s_H=\sq)}+\frac{\epsilon}{4}\PP_0^{\pi}(o_H\neq s_0)\\
    \geq& \frac{\epsilon}{4}\PP_0^{\pi}(o_H\neq s_0).
\end{align*}

Case 2: $M=(\theta,\mu)$ for some $\theta=(\hs,\ss,\acs,\as)$. In this case, $\sp$ is reachable only when $o_{\hs}=\ss$ and $a_{\hs}=\acs$, and
\begin{align*}
    \PP_M^{\pi}(s_H=\sp)=\PP_M^{\pi}(s_H=\sp|o_{\hs}=\ss,a_{\hs}=\acs)\PP_M^{\pi}(o_{\hs}=\ss,a_{\hs}=\acs)
    \leq \epsilon\PP_M^{\pi}(o_{\hs}=\ss,a_{\hs}=\acs)\leq \epsilon,
\end{align*}
where the equality can be attained when $\pi$ is any deterministic policy that ensure $o_{\hs}=\ss,a_{\hs}=\acs, a_{\hs+1:H-1}=\as$. Thus, in this case $V_M^\star=\max_{\pi}V_M(\pi)=\max\set{\frac{1+\eps}{4}, \frac{3\eps}{4}+\frac{1-\eps}{4}}=\frac{1+2\epsilon}{4}$, and
\begin{align*}
    V_M^\star-V_M(\pi)
    =&\frac{1+2\epsilon}{4}-\frac{1+\epsilon}{4}\PP_M^{\pi}(o_H=s_0)-\frac{3}{4}\PP_M^{\pi}(s_H=\sp)-\frac{1}{4}\PP_M^{\pi}(s_H=\sq)\\
    =&
    \frac{\epsilon}{4}\PP_M^{\pi}(o_H= s_0)+\frac{1+2\epsilon}{4}\PP_M^{\pi}(o_H\neq s_0)-\frac{3}{4}\PP_M^{\pi}(s_H=\sp)-\frac{1}{4}\PP_M^{\pi}(s_H=\sq)\\
    \geq& \frac{\epsilon}{4}\PP_M^{\pi}(o_H= s_0)+\frac{\epsilon}{2}\PP_M^{\pi}(o_H\neq s_0)-\frac{1}{2}\PP_M^{\pi}(s_H=\sp)\\
    \geq& \frac{\epsilon}{4}\PP_M^{\pi}(o_H= s_0),
\end{align*}
where the first inequality is because $\PP_M^{\pi}(s_H=\sp)+\PP_M^{\pi}(s_H=\sq)\leq \PP_M^{\pi}(o_H\neq s_0)$ by the inclusion of events.
\qed

\subsection{Proof of Lemma \ref{lemma:1-step-pac-alternative}}\label{appendix:proof-1-step-pac-visitation-lower-bound}

We first prove the following version of \cref{lemma:1-step-pac-alternative} with an additional condition that the visitation counts are almost surely bounded under $\PP_0^{\fA}$, and then prove \cref{lemma:1-step-pac-alternative} by reducing to this case using a truncation argument. %

\begin{lemma}\label{lemma:1-step-pac-alternative-weak}
Suppose that algorithm $\fA$ (with possibly random stopping time $\stoptime$) satisfies $\sum_{h} N(\Erevh^\theta)\leq \oN_o$ and $N(\Ecor^\theta)\leq\oN_r$ almost surely under $\PP_0^{\fA}$, for some fixed $\oN_o,\oN_r$. Then 
\begin{align*}
    \text{either }\oN_o \geq \frac{\delta^2\sqrt{K}}{4\epsilon^2\sigma^2}, \quad \text{  or } \oN_r\geq \frac{\delta^2}{4\epsilon^2},
\end{align*}
where $\delta=\DTV{\PP_{0}^{\fA},\Emu\brac{\PP_{\theta,\mu}^{\fA}}}$.
\end{lemma}
\begin{proof}[Proof of Lemma \ref{lemma:1-step-pac-alternative-weak}]
By \cref{lemma:ingster}, we have
\begin{align*}
    1+\chis{\Emu\brac{\PP_{\theta,\mu}^{\fA}}}{\PP_{0}^{\fA}}
    =
    \Emumu\EE_{\tau^{(1)},\cdots,\tau^{(\stoptime)}\sim\PP_{0}^{\fA}}\brac{
        \prod_{t=1}^{\stoptime}\frac{\PP_{\theta,\mu}(\taut)\PP_{\theta,\mu'}(\taut)}{\PP_{0}(\taut)^2}
    }.
\end{align*}
To upper bound the above quantity, we invoke the following lemma, which serves a key step for bounding the above ``$\chi^2$-inner product'' \citep[Section 3.1]{canonne2022topics} between $\P_{\theta, \mu}/\P_0$ and $\P_{\theta,\mu'}/\P_0$ (proof in~\cref{appendix:proof-1-step-pac-MGF}).

\begin{lemma}[Bound on the $\chi^2$-inner product]
\label{lemma:1-step-pac-MGF}
Under the conditions of \cref{lemma:1-step-pac-alternative-weak} (for a fixed $\theta$), it holds that for any $\mu,\mu'\in\set{-1,1}^K$,
\begin{align}\label{eqn:1-step-pac-MGF}
    \EE_0^\fA\brac{
        \prod_{t=1}^{\stoptime}\frac{\P_{\theta,\mu}(\taut)\P_{\theta,\mu'}(\taut)}{\P_{0}(\taut)^2}
    }
    \leq \exp\paren{ \oN_o\cdot \frac{C\sigma^2\epsilon^2}{K}\abs{\iprod{\mu}{\mu'}}+\frac43C\epsilon^2\oN_r }.
\end{align}
where $C\defeq (1+\sigma)^{2H}\leq e$ as $\sigma\leq \frac{1}{2H}$.
\end{lemma}
Now we assume that \cref{lemma:1-step-pac-MGF} holds and continue the proof of \cref{lemma:1-step-pac-alternative-weak}. Taking expectation of \eqref{eqn:1-step-pac-MGF} over $\mu,\mu'\sim \unif(\set{-1,+1}^K)$, we obtain
\begin{align*}
    1+\chis{\Emu\brac{\P_{\theta,\mu}^{\fA}}}{\P_{0}^{\fA}}
    =&
    \Emumu\EE_{\tau^{(1)},\cdots,\tau^{(\stoptime)}\sim\P_{0}^{\fA}}\brac{
        \prod_{t=1}^{\stoptime}\frac{\P_{\theta,\mu}(\taut)\P_{\theta,\mu'}(\taut)}{\P_{0}(\taut)^2}
    }\\
    \leq & \Emumu \brac{\exp\paren{ \oN_o\cdot \frac{C\sigma^2\epsilon^2}{K}\abs{\iprod{\mu}{\mu'}}+\frac43C\epsilon^2\oN_r }}.
\end{align*}
Notice that $\mu_i,\mu_i'$ are i.i.d. $\Unif(\{\pm 1\})$, and hence $\mu_1\mu_1',\cdots,\mu_K\mu_K'$ are i.i.d. $\Unif(\{\pm 1\})$. Then by Hoeffding's lemma, it holds that $\Emumu\brac{\exp\paren{x\sum_{i=1}^K\mu_i\mu_i'}}\leq \exp\paren{Kx^2/2}$ for all $x\in\R$, and thus by \cref{lemma:MGF-abs}, we have 
\begin{align*}
    \Emumu\brac{\exp\paren{\frac{C\oN_o\sigma^2\epsilon^2}{K}\abs{\iprod{\mu}{\mu'}}}}\leq \exp\paren{\max\set{\frac{C^2\sigma^4\epsilon^4\oN_o^2}{K}, \frac43\frac{C\sigma^2\epsilon^2\oN_o}{\sqrt{K}}}}.
\end{align*}
Therefore, combining the above inequalities with \cref{lemma:TV-to-chi}, we obtain
\begin{align*}
    2\delta^2=2 \DTV{\Emu\brac{\PP_{\theta,\mu}^{\fA}}, \PP_{0}^{\fA}}^2
    \leq \log\paren{1+\chis{\Emu\brac{\PP_{\theta,\mu}^{\fA}}}{\PP_{0}^{\fA}}}
    \leq \max\set{ \frac43\frac{\oN_oC\sigma^2\epsilon^2}{\sqrt{K}}, \frac{\oN_o^2C^2\sigma^4\epsilon^4}{K}} +\frac43C\epsilon^2\oN_r.
\end{align*}
Then, we either have $\oN_r \geq \frac{3\delta^2}{4C\epsilon^2}$, or it holds
$$
\max\set{ \frac43\frac{\oN_oC\sigma^2\epsilon^2}{\sqrt{K}}, \frac{\oN_o^2C^2\sigma^4\epsilon^4}{K}} \geq \delta^2,
$$
which implies that $\frac{\oN_oC\sigma^2\epsilon^2}{\sqrt{K}}\geq \min\set{\frac43,\frac34\delta^2}=\frac34\delta^2$ (as $\delta\leq 1$). Using the fact that $C\leq e$ completes the proof of \cref{lemma:1-step-pac-alternative-weak}.
\end{proof}

\begin{proof}[Proof of Lemma \ref{lemma:1-step-pac-alternative}]
We perform a truncation type argument to reduce Lemma \ref{lemma:1-step-pac-alternative} to Lemma \ref{lemma:1-step-pac-alternative-weak}.
Let us take $\oN_o=\ceil{6\delta^{-1}\E_0^\fA \brac{\sum_{h=\hs}^{H-2}N(\Erevh^\theta)}}$ and $\oN_r=\ceil{6\delta^{-1}\E_0^\fA N(\Ecor^\theta)}$.
By Markov's inequality, we have
\begin{align*}
    \PP_{0}^{\fA}\paren{\sum_{h=\hs}^{H-2}N(\Erevh^\theta)\geq \oN_o}\leq \frac\delta6, \qquad 
    \PP_{0}^{\fA}\paren{N(\Ecor^\theta)\geq \oN_r}\leq \frac\delta6.
\end{align*}
Therefore, we can consider the following exit criterion $\cexit$ for the algorithm $\fA$:
\[
{\cexit(\tau^{(1:T')})}=\TRUE \quad \textrm{iff} \quad \sum_{t=1}^{T'} \sum_{h=\hs}^{H-2} {{\mathbb{I}}\paren{\taut\in\Erevh^\theta}\geq \oN_o \text{ or } \sum_{t=1}^{T'} {{\mathbb{I}}\paren{\taut\in \Ecor^\theta}}\geq \oN_r.}
\]
The criterion $\cexit$ induces a stopping time $T_{\cexit}$, and we have
\begin{align*}
    \PP_0^\fA(\exists t < T, \cexit(\tau^{(1:t)})=\TRUE)
    \leq \PP_{0}^{\fA}\paren{\sum_{h=\hs}^{H-2}N(\Erevh^\theta)\geq \oN_o \text{ or } N(\Ecor^\theta)\geq \oN_r}\leq \frac\delta6+\frac\delta6\leq \frac\delta3.
\end{align*}
Therefore, we can consider the early stopped algorithm $\fA(\cexit)$ with exit criterion $\cexit$ (cf. \cref{appendix:ingster-method}), and by \cref{lemma:alg-stop} we have
\begin{align*}
\DTV{\PP_{0}^{\fA(\cexit)},\Emu\brac{\PP_{\theta,\mu}^{\fA(\cexit)}}}
\geq \DTV{\PP_{0}^{\fA},\Emu\brac{\PP_{\theta,\mu}^{\fA}}} - \PP_0^\fA(\exists t < T, \cexit(\tau^{(1:t)})=\TRUE) \geq \frac{2\delta}{3}.
\end{align*}

Notice that by our definition of $\cexit$ and stopping time $T_{\cexit}$, in the execution of $\fA(\cexit)$, we also have
\begin{align*}
    \sum_{t=1}^{T_{\cexit}-1} \sum_{h=\hs}^{H-2} \IIc{\taut\in\Erevh^\theta}< \oN_o, \qquad \sum_{t=1}^{T_{\cexit}-1} \IIc{\taut\in \Ecor^\theta}< \oN_r.
\end{align*}
Therefore, algorithm $\fA(\cexit)$ ensures that
\begin{align*}
    \sum_{h=\hs}^{H-2}N(\Erevh^\theta)=\sum_{t=1}^{T_{\cexit}} \sum_{h=\hs}^{H-2} \IIc{\taut\in\Erevh^\theta} \leq \oN_o+H-1, \qquad N(\Ecor^\theta)=\sum_{t=1}^{T_{\cexit}} \IIc{\taut\in \Ecor^\theta}\leq \oN_r.
\end{align*}
Applying \cref{lemma:1-step-pac-alternative-weak} to the algorithm $\fA(\cexit)$ (and $\delta'=\frac23\delta$), we can obtain
\begin{align*}
    \text{either }~  \frac{\delta^2\sqrt{K}}{9\epsilon^2\sigma^2}
    \leq \oN_o+H-1 \leq 6\delta^{-1}\E_0^\fA \brac{\sum_{h=\hs}^{H-2}N(\Erevh^\theta)} + H, \qquad
    \text{  or }~ \frac{\delta^2}{9\epsilon^2} \leq \oN_r \leq 6\delta^{-1}\E_0^\fA \brac{N(\Ecor^\theta)} +1,
\end{align*}
and rearranging gives the desired result.
\end{proof}

\subsection{Proof of Lemma \ref{lemma:1-step-pac-MGF}}\label{appendix:proof-1-step-pac-MGF}

Throughout the proof, the parameters $\theta,\mu,\mu'$ are fixed.

By our discussions in \cref{appendix:ingster-method}, 
using~\cref{eqn:ingster-prod-to-i}, we have
\begin{align}\label{eqn:1-step-pac-ingster-prod-to-i}
    \EE_{\tau^{(1)},\cdots,\tau^{(\stoptime)}\sim\PP_{0}^{\fA}}\brac{
        \prod_{t=1}^{\stoptime}\frac{\PP_{M}(\taut)\PP_{M'}(\taut)}{\PP_{0}(\taut)^2}\cdot
        \exp\paren{-\sum_{t=1}^\stoptime\sum_{h=1}^H \log I(\taut_{h-1})}
    }=1,
\end{align}
where for any partial trajectory $\tau_l$ up to step $l\in[H]$, $I(\tau_l)$  is defined as
\begin{align*}
    I(\tau_l)\defeq \EE_0\brcond{\frac{\PP_{\theta,\mu}(o_{l+1}|\tau_l)\PP_{\theta,\mu'}(o_{l+1}|\tau_l)}{\PP_{0}(o_{l+1}|\tau_l)^2}}{\tau_l}.
\end{align*}

Notice that the model $\PP_{\theta,\mu}$ and $\PP_0$ are different only at the transition from $s_{\hs}=\ss,a_{\hs}=\acs$ to $\sp$ and the transition dynamic at state $\sp$. Therefore,
for any (reachable) trajectory $\tau_l=(o_1,a_1,\cdots,o_l,a_l)$, $\PP_{\theta,\mu}(o_{l+1}=\cdot|\tau_l)\neq \PP_{0}(o_{l+1}=\cdot|\tau_l)$ only if $o_{\hs}=\ss,a_{\hs}=\acs$. In other words, $I(\tau_l)=1$ if $\tau_l\not\in \set{o_{\hs}=\ss,a_{\hs}=\acs}$. %

We next compute $I(\tau_l)$ for $\tau_l\in \set{o_{\hs}=\ss,a_{\hs}=\acs}$. By our construction, we have
\begin{equation}\label{eqn:1-step-pac-MGF-proof-eq1}
\begin{aligned}
    \PP_{\theta,\mu}(o_{l+1}=o|\tau_l)
    =&~
    \PP_{\theta,\mu}(o_{l+1}=o|s_{l+1}=\sp)\cdot \PP_{\theta,\mu}(s_{l+1}=\sp|\tau_l)\\
    &~+\PP_{\theta,\mu}(o_{l+1}=o|s_{l+1}=\sq)\cdot \PP_{\theta,\mu}(s_{l+1}=\sq|\tau_l)\\
    =&~
    \paren{\O_{l;\mu}(o|\sp)-\O_l(o|\sq)}\cdot \PP_{\theta,\mu}(s_{l+1}=\sp|\tau_l)
    + \O_l(o|\sq).
\end{aligned}
\end{equation}
Notice that if $\tau_l\not \in E_{\rev,l}$, then $s_{l+1}$ must be $\sq$, and hence $\PP_{\theta,\mu}(o_{l+1}=\cdot|\tau_l)=\O_h(\cdot|\sq)=\PP_{0}(o_{l+1}=\cdot|\tau_l)$ which implies that $I(\tau_l) = 1$.

We next consider the case $\tau_l\in E_{\rev,l}$, i.e. $a_{\hs+1:l}=\as_{\hs+1:l}$:
\begin{align*}
    \PP_{\theta,\mu}(s_{l+1}=\sp|\tau_l)
    =&~
    \PP_{\theta,\mu}(s_{l+1}=\sp|o_{\hs}=\ss,a_{\hs}=\acs, o_{\hs+1:l}, a_{\hs+1:l})\\
    =&~
    \frac{ 
        \PP_{\theta,\mu}(o_{\hs+1:l},s_{l+1}=\sp|o_{\hs}=\ss,a_{\hs}=\acs, a_{\hs+1:l}) 
    }{
        \PP_{\theta,\mu}(o_{\hs+1:l}|o_{\hs}=\ss,a_{\hs}=\acs, a_{\hs+1:l}) 
    }\\
    =&~
    \frac{ 
        \epsilon\cdot \PP_{\theta,\mu}(o_{\hs+1:l}|s_{\hs+1}=\sp, a_{\hs+1:l}) 
    }{
        \epsilon\cdot \PP_{\theta,\mu}(o_{\hs+1:l}|s_{\hs+1}=\sp, a_{\hs+1:l}) + (1-\epsilon)\cdot \PP_{\theta,\mu}(o_{\hs+1:l}|s_{\hs+1}=\sq, a_{\hs+1:l}) 
    }\\
    =&~
    \frac{ 
        \epsilon
    }{
        \epsilon + (1-\epsilon)\cdot \frac{\PP_{\theta,\mu}(o_{\hs+1:l}|s_{\hs+1}=\sq, a_{\hs+1:l})}{\PP_{\theta,\mu}(o_{\hs+1:l}|s_{\hs+1}=\sp, a_{\hs+1:l}) } 
    },
\end{align*}
where the third equality is because $\PP_{\theta,\mu}(s_{\hs+1}=\sp|o_{\hs}=\ss,a_{\hs}=\acs)=\epsilon$. Notice that
\begin{align*}
    \beta_{\tau_l}\defeq \frac{\PP_{\theta,\mu}(o_{\hs+1:l}|s_{\hs+1}=\sp, a_{\hs+1:l})}{\PP_{\theta,\mu}(o_{\hs+1:l}|s_{\hs+1}=\sq, a_{\hs+1:l}) }
    =&~
    \prod_{h=\hs+1}^l \frac{\O_{h;\mu}(o_{h}|\sp)}{\O_{h}(o_{h}|\sq)}
    \leq (1+\sigma)^{l-\hs},
\end{align*}
where the inequality holds by our construction of $\O$, as long as $\tau_l$ is reachable (i.e. $o_{\hs+1:l}\in\cO^{l-\hs}$). Thus, for 
$$
c_{\tau_l}\defeq \PP_{\theta,\mu}(s_{l+1}=\sp|\tau_l)=\frac{\beta_{\tau_l}}{\epsilon \beta_{\tau_l} + 1-\epsilon},
$$
we have $c_{\tau_l}\leq (1+\sigma)^H=\sqrt{C}$. Notice that by \cref{eqn:1-step-pac-MGF-proof-eq1} and the equation above we have
\begin{align*}
    &\text{when }l<H-1, \qquad \P_{\theta,\mu}(o_{l+1}=o_i^+|\tau_l)=\frac{1+c_{\tau_l}\epsilon\sigma\mu_{i}}{2K}, \qquad 
    \P_{\theta,\mu}(o_{l+1}=o_i^-|\tau_l)=\frac{1-c_{\tau_l}\epsilon\sigma\mu_{i}}{2K} \qquad \forall i\in[K],\\
    &\text{when }l=H-1, \qquad \P_{\theta,\mu}(o_{H}=\oba|\tau_{H-1})=\frac{1+2c_{\tau_{H-1}}\epsilon}{4}, \qquad 
    \P_{\theta,\mu}(o_{H}=\obb|\tau_{H-1})=\frac{3-2c_{\tau_{H-1}}\epsilon}{4}.
\end{align*}
On the other hand, when $l<H-1$, $\P_{\theta,\mu}(o_{l+1}=\cdot|\tau_l)=\Unif(\set{o_1^+,o_1^-,\cdots,o_K^+,o_K^-})$. Hence, 
\begin{align*}
    I(\tau_l)
    =&
    \EE_0\brcond{\frac{\PP_{\theta,\mu}(o_{l+1}|\tau_l)\PP_{\theta,\mu'}(o_{l+1}|\tau_l)}{\PP_{0}(o_{l+1}|\tau_l)^2}}{\tau_l}\\
    =&
    \frac{1}{2K}\sum_{o\in\cO_o} \frac{\PP_{\theta,\mu}(o_{l+1}=o|\tau_l)\PP_{\theta,\mu'}(o_{l+1}=o|\tau_l)}{\PP_{0}(o_{l+1}=o|\tau_l)^2}\\
    =&
    \frac{1}{2K} \sum_{i=1}^K (1+c_{\tau_l}\epsilon\sigma\mu_i)(1+c_{\tau_l}\epsilon\sigma\mu_i')+(1-c_{\tau_l}\epsilon\sigma\mu_i)(1-c_{\tau_l}\epsilon\sigma\mu_i')\\
    =&1+\frac{c_{\tau_l}^2\epsilon^2\sigma^2}{K}\sum_{i=1}^K \mu_i\mu_i' \leq 1+\frac{C\epsilon^2\sigma^2}{K}\abs{\iprod{\mu}{\mu'}}.
\end{align*}
Similarly, when $l=H-1$, we can compute 
$$
I(\tau_{H-1})
=
\EE_0\brcond{\frac{\PP_{\theta,\mu}(o_{H}|\tau_{H-1})\PP_{\theta,\mu'}(o_{H}|\tau_{H-1})}{\PP_{0}(o_{H}|\tau_{H-1})^2}}{\tau_{H-1}}
=
1+\frac{4}{3}c_{\tau_{H-1}}^2\epsilon^2\leq 1+\frac{4}{3}C\epsilon^2.
$$

Therefore, combining all these facts above, we can conclude that
\begin{align*}
    \begin{cases}
        I(\tau_l)=1, & l\leq \hs, \\
        I(\tau_l)\leq 1+\IIc{\tau_l\in E_{\rev,l}^\theta}\cdot \frac{C\epsilon^2\sigma^2}{K}\abs{\iprod{\mu}{\mu'}}, & \hs<l<H-1,\\
        I(\tau_{H-1})\leq 1+\IIc{\tau_{H-1}\in \Ecor^\theta}\cdot \frac{4}{3}C\epsilon^2, & l=H-1,
    \end{cases}
\end{align*}
where we use the fact that $\Ecor^\theta=E_{\rev,H-1}^\theta$ by definition. Hence, using the fact $\log(1+x)\leq x$, we have
\begin{align*}
    \sum_{t=1}^\stoptime\sum_{l=0}^{H-1} \log I(\taut_{l})
    =&~ \sum_{t=1}^\stoptime\sum_{l=\hs+1}^{H-1} \log I(\taut_{l})\\
    \leq&~ \sum_{t=1}^\stoptime \sum_{l=\hs+1}^{H-2}\IIc{\taut_l\in E_{\rev,l}^\theta}\cdot \frac{C\epsilon^2\sigma^2}{K}\abs{\iprod{\mu}{\mu'}} + \IIc{\taut_{H-1}\in \Ecor^\theta}\cdot \frac{4}{3}C\epsilon^2\\
    =&~
    \sum_{l=\hs+1}^{H-2}\Nc{ E_{\rev,l}^\theta}\cdot \frac{C\epsilon^2\sigma^2}{K}\abs{\iprod{\mu}{\mu'}} + \Nc{\Ecor^\theta}\cdot \frac{4}{3}C\epsilon^2\\
    \leq&~
    \oN_o\cdot \frac{C\epsilon^2\sigma^2}{K}\abs{\iprod{\mu}{\mu'}} + \oN_r\cdot \frac{4}{3}C\epsilon^2.
\end{align*}
Plugging the above inequality into \eqref{eqn:1-step-pac-ingster-prod-to-i} completes the proof of \cref{lemma:1-step-pac-MGF}.
\qed

\section{Proof of Theorem \ref{thm:no-regret-demo}}
\label{appdx:no-regret}

We first construct a family of hard instances in~\cref{appendix:construction-no-regret}. We state the regret lower bound of this family of hard instances in~\cref{prop:no-regret-prop}.~\cref{thm:no-regret-demo} then follows from~\cref{prop:no-regret-prop} as a direct corollary. \cref{prop:no-regret-prop} also implies a part of the PAC lower bound stated in \cref{thm:multi-step-pac-demo}. %

\subsection{Construction of hard instances and proof of Theorem~\ref{thm:no-regret-demo}}
\label{appendix:construction-no-regret}

We consider the following family of $m$-step revealing POMDPs $\cM$ that admits a tuple of hyperparameters $(\epsilon,\sigma,n,m,K,H)$. All POMDPs in $\cM$ share the state space $\cS$, action space $\cA$, observation space $\cO$, and horizon length $H$, defined as following.
\begin{itemize}[topsep=0pt, itemsep=0pt]
    \item The state space $\cS=\Stree \bigsqcup \set{\sp, \sq, \ep, \eq, \termin}$, where $\Stree$ is a binary tree with level $n$ (so that $\abs{\Stree}=2^{n}-1$).  
    Let $s_0$ be the root of $\Stree$, and $\Sl$ be the set of leaves of $\Stree$, with $\abs{\Sl}=2^{n-1}$.
    \item The observation space $\cO=\Stree \bigsqcup \set{o_1^+,o_1^-,\cdots,o_K^+,o_K^-}\bigsqcup\set{\olock,\oba,\obb,\termin}$. \footnote{
    Similarly to \cref{appdx:1-step-pac}, here we slightly abuse notation to reuse $\Stree$ to denote both a set of states and a corresponding set of observations, in the sense that each state $s\in\Stree\subset \cS$ corresponds to a unique observation $o_s\in\Stree\subset \cO$, which we also denote as $s$ when it is clear from the context.}
    \item The action space $\cA=\set{0,1,\cdots,A-1}$.
\end{itemize}
We further define $\Arev=\set{0,1,\cdots,A_1-1}, \Atr=\set{A_1,\cdots,A-1},$ with $A_1=1+\floor{A/6}$.

\paragraph{Model parameters} 
Each non-null POMDP model $M=M_{\theta,\mu}\in\cM \setminus \{ M_0\}$ is specified by two parameters $(\theta, \mu)$. Here $\mu\in\set{-1,+1}^{K}$, and $\theta=(\hs,\ss,\acs,\revs,\as)$, where 
\begin{itemize}[topsep=0pt, itemsep=0pt]
    \item $\ss\in\Sl$, $\acs\in \Ac\defeq\set{1,\cdots,A-1}$, $\revs \in \Arev$. %
    \item $\hs\in\cH\defeq \set{h=n+lm: h<H, l \in \Z_{\ge 0}}$.
    \item $\as=(\as_{\hs+1},\dots,\as_{H-1})\in\cA^{H-\hs-1}$ is an action sequence indexed by $\hs+1,\cdots,H-1$, such that when $h\in\cH$, we have $\as_h\in\Atr$. We use $\Apass$ to denote the set of all such $\as$.
\end{itemize}
Our construction ensures that, only at steps $h\in\cH$ and states $s_h\in\set{\sp,\sq}$, the agent can take actions in $\Arev$ and transits to $\set{\ep,\eq}$. %

For any POMDP $M_{\theta, \mu}$, its system dynamics $\PP_{\theta,\mu}\defeq \PP_{M_{\theta, \mu}}$ is defined as follows.

\paragraph{Emission dynamics} At state $s\in\Stree\cup \set{\termin}$, the agent always receives (the unique observation corresponding to) $s$ itself as the observation.
\begin{itemize}[topsep=0pt, itemsep=0pt]
    \item At state $\ep$, the emission dynamic is given by
    \begin{align*}
        \O_{\mu}(o_i^+|\ep)=\frac{1+\sigma\mu_{i}}{2K}, \qquad 
        \O_{\mu}(o_i^-|\ep)=\frac{1-\sigma\mu_{i}}{2K}, \qquad
        \forall i\in[K],
    \end{align*}
    where we omit the subscript $h$ because the emission distribution does not depend on $h$.
    \item At state $\eq$, the observation is uniformly drawn from $\cO_o\defeq \set{o_1^+,o_1^-,\cdots,o_K^+,o_K^-}$, i.e. $\O(\cdot|\eq)=\Unif(\cO_o)$.
    \item At states $s\in\set{\sp,\sq}$ and steps $h\in[H-1]$, the agent always receives $\olock$ as the observation; At step $H$, the emission dynamics at $\set{\sp,\sq}$ is given by
    \begin{align*}
        &\O_H(\oba|\sp)=\frac34, \qquad \O_H(\obb|\sp)=\frac14, \\
        &\O_H(\oba|\sq)=\frac14, \qquad \O_H(\obb|\sq)=\frac34.
    \end{align*}
\end{itemize}

\paragraph{Transition dynamics} In each episode, the agent always starts at state $s_0$. 
\begin{itemize}[topsep=0pt, itemsep=0pt]
    \item At any node $s\in\Stree\setminus \Sl$, there are three types of available actions: $\await=0$, $\aleft=1$ and $\aright=2$, such that the agent can take $\await$ to stay at $s$, $\aleft$ to transit to the left child of $s$ and $\aright$ to transit to the right child of $s$.
    \item At any $s\in\Sl$, the agent can take action $\await=0$ to stay at $s$ (i.e. $\PP(s|s,\await)=1$); otherwise, for $s\in\Sl$, $h\in[H-1]$, $a\neq \await$,
    \begin{align*}
        \PP_{h;\theta}(\sp|s,a)&=\epsilon\cdot \II(h=\hs,s=\ss,a=\acs), \\
        \PP_{h;\theta}(\sq|s,a)&=1-\epsilon\cdot \II(h=\hs,s=\ss,a=\acs).
    \end{align*}
    where we use subscript $\theta$ to emphasize the dependence on $\theta$. In words, at step $h^\star$, at any leaf node taking any action, the agent will transit to one of $\set{\sp,\sq}$; only by taking $a^\star$ at $s^\star$, the agent can transit to state $\sp$ with a small probability $\eps$; in any other case the agent will transit to state $\sq$ with probability one.

    \item The state $s\in\set{\ep,\eq}$ always transits to $\termin$, regardless of the action taken.
    \item The $\termin$ state is an absorbing state.

\item At state $\sq$:
\begin{itemize}
    \item For steps $h\in\cH$ and $a\in\Arev$, we set $\PP_{h;\theta}(\eq|\sq,a)=1$, i.e. taking $a\in\Arev$ always transits to $\eq$.
    \item For steps $h\not\in\cH$ or $a\in\Atr$, we set $\PP_{h;\theta}(\sq|\sq,a)=1$, i.e. taking such action always stays at $\sq$.
\end{itemize}
 \item At state $\sp$, we only need to specify the transition dynamics for steps $h\geq\hs+1$: %
\begin{itemize}
    \item For steps $h\in\cH_{>\hs}=\cH\cap \set{h>\hs}$ and $a\in\Arev$, we set
    \begin{align*}
        \PP_{h;\theta}(\ep|\sp,a)=\II(a=\revs), \qquad \PP_{h;\theta}(\eq|\sp,a)&=\II(a\neq\revs).
    \end{align*}
    In words, at steps $h\in\cH_{> \hs}$ and states $s_h\in\set{\sp,\sq}$ (corresponding to $o_h=\olock$), the agent can take actions $a_h\in\Arev$ to transit to $\set{\ep,\eq}$; but only by taking $a_h=\revs$ ``correctly'' at $\sp$ the agent can transit to $\ep$; in any other case the agent will transit to state $\eq$ with probability one. Note that $\cH = \set{h=n+lm: h<H, l \in \Z_{\ge 0}}$, so we only allow the agent to take the reveal action $\revs$ every $m$ steps, which ensures that our construction is $(m+1)$-step revealing.
    \item For steps $h\not\in\cH$ or $a\in\Atr$, we set
    \begin{align*}
        \PP_{h;\theta}(\sp|\sp,a)=\II(a=\as_h), \qquad \PP_{h;\theta}(\sq|\sp,a)&=\II(a\neq\as_h).
    \end{align*}
\end{itemize}
\end{itemize}

\paragraph{Reward} The reward function is known (and only depends on the observation): at the first $H-1$ steps, no reward is given; at step $H$, we set $r_H(\oba)=1$, $r_H(\obb)=0$, 
$r_H(s_0)=(1+\epsilon)/4$, and $r_H(o)=0$ for any other $o\in\cO$. 

\paragraph{Reference model}
We use $M_0$ (or simply $0$) to refer to the null model (reference model). The null model $M_0$ has transition and emission the same as any non-null model, except that the agent always arrives at $\sq$ by taking any action $a\neq\await$ at $s \in\Sl$ and $h \in [H - 1]$ (i.e., $\PP_{h;M_0}(\sq|s,a) = 1$ for any $s \in \Sl$, $a \in \Ac$, $h \in [H-1]$). In this model, $\sp$ is not reachable (and so does $\ep$), and hence we do not need to specify the transition and emission dynamics at $\sp, \ep$.

We present the expected regret lower bound and PAC-learning sample complexity lower bound of the above POMDP model class $\cM$ in the following proposition, which we prove in~\cref{appendix:proof-no-regret-prop}.

\begin{proposition}\label{prop:no-regret-prop}
For given $\epsilon \in(0,0.1], \sigma\in(0,1]$, $m,n\geq 1$, $K\geq 2$, $H\geq 8n+m+1$, the above model class $\cM$ satisfies the following properties.

\begin{enumerate}[topsep=0pt, itemsep=0pt]
\item $\nS=2^n+4$, $\nO=2^n+2K+3$, $\nA=A$.

\item For each $M\in\cM$, $M$ is $(m+1)$-step revealing with $\arevmp(M)^{-1}\leq 1+\frac2\sigma$. %

\item $\log\abs{\cM}\leq K\log 2 + H\log A + \log(SAH)$.

\item Suppose algorithm $\fA$ interacts with the environment for $T$ episodes, then
\begin{align*}
    \max_{M\in\cM}\EE^{\fA}_M\brac{\Regret}\geq \frac{1}{120000}\min\set{
        \frac{\abs{\Sl}K^{1/2}A^{m+1}H}{m\sigma^2\epsilon^2}, \frac{\abs{\Sl}A^{H/2}H}{m\epsilon}, \epsilon T
    },
\end{align*}
where we recall that $\abs{\Sl}=2^{n-1}$.

\item Suppose algorithm $\fA$ interacts with the environment for $T$ episodes and returns $\piout$ such that
$$
\PP^{\fA}_M\paren{V_M^\star-V_M(\piout)<\frac{\epsilon}{8}}\geq \frac{3}{4}.
$$
for any $M\in\cM$, then it must hold that
\begin{align*}
    T\geq \frac{1}{60000}\min\set{
        \frac{\abs{\Sl}K^{1/2}A^{m+1}H}{\sigma^2\epsilon^2}, \frac{\abs{\Sl}A^{H/2}H}{\epsilon^2}
    }.
\end{align*}
\end{enumerate}
\end{proposition}

\begin{proofof}[thm:no-regret-demo]
We only need to suitably choose parameters when applying \cref{prop:no-regret-prop}. More specifically, given $(S,O,A,H,\arev,m)$, we can let $m'=m-1$, and take $n\geq 1$ to be the largest integer such that $2^n\leq \min\set{S-4,(O-5)/2}$, and take $K=\floor{\frac{O-2^n-3}{2}}\geq \frac{O-5}{4}$, $\epsilon'=\epsilon/8$, and $\sigma=\frac{2}{\arev^{-1}-1}\leq 1$. For any fixed $\epsilon\in(0,0.1]$, applying \cref{prop:no-regret-prop} to the parameters $(\epsilon,\sigma,n,m',K,H)$, we obtain a model class $\cM_\epsilon$ such that for any algorithm $\fA$, 
\begin{align*}
    \max_{M\in\cM_\epsilon}\EE^{\fA}_M\brac{\Regret}\geq c_0\min\set{
        \frac{SO^{1/2}A^{m}H}{m\arev^2\epsilon^2}, \frac{SA^{H/2}H}{m\epsilon}, \epsilon T
    },
\end{align*}
where $c_0$ is a universal constant. We can then take the $\epsilon\in(0,0.1]$ that maximizes the RHS of the above inequality, and applying \cref{lemma:balance-eps} completes the proof of \cref{thm:no-regret-demo}.
\end{proofof}

\begin{remark}
The requirement $S\leq O$ in \cref{thm:no-regret-demo} (and \cref{thm:multi-step-pac-demo}) can actually be relaxed to $S\leq O^m$. The reason why we require $S\leq O$ in the current construction is that we directly embed $\Stree$ directly into the observation space $\cO$, i.e. for each state $s\in\Stree$ it emits the corresponding $o_s\in\cO$. However, when $O^m\geq \abs{\Stree}\gg O$, we can alternatively take an embedding $\Stree\to\cO^m$, i.e. for each state $s\in\Stree$ such that $s\mapsto (o_s^{(1)},\cdots,o_s^{(m)})$, it emits $o_s^{(h~\mathrm{mod}~m)}\in\cO$ at step $h$. %
\end{remark}

\subsection{Proof of Proposition \ref{prop:no-regret-prop}}
\label{appendix:proof-no-regret-prop}

All propositions and lemmas stated in this section are proved in~\cref{appendix:proof-no-regret-revealing}-\ref{appendix:proof-no-regret-MGF}.

Claim 1 follows directly by counting the number of states, observations, and actions in models in $\cM$. Claim 3 follows as we have $\abs{\cM}=\abs{\set{(h^\star,s^\star,a^\star,\revs,\a^\star)}}\times \abs{\set{\pm 1}^K}+1\le HSA^H\times 2^K$. Taking logarithm yields the claim.

Claim 2 follows from this lemma, which is proved in \cref{appendix:proof-no-regret-revealing}.
\begin{lemma}\label{lemma:regret-rev-2-step}
    For each $M\in\cM$, it holds that $\arevmp(M)^{-1}\leq \frac{2}{\sigma}+1$.
\end{lemma}

We now prove Claim 4 \& 5. Similar to the proof of \cref{prop:1-step-pac-prop}, we begin by relating the learning problem to a testing problem. 
Recall that $\PP^{\fA}_M$ is the law of $(\tau^{(1)},\tau^{(2)},\cdots,\tau^{(\stoptime)})$ induced by algorithm $\fA$ and model $M$. For any event $E\subseteq (\cO\times\cA)^H$, we denote the visitation count of $E$ as
\begin{align*}
    N(E)\defeq \sum_{t=1}^T \II(\taut\in E).
\end{align*}
Since $N(E)$ is a function of $\tau^{(1:T)}$, we can talk about its expectation under the distribution $\PP^{\fA}_M$ for any $M\in\cM$. We first relate the expected regret to the expected visitation count of some ``bad'' events, giving the following lemma whose proof is contained in \cref{app:proof-no-regret-value-func}. 

\begin{lemma}[Relating regret to visitation counts]\label{lemma:no-regret-value-func}
    For any $M\in\cM$ such that $M\neq 0$, it holds that 
    \begin{align}\label{eqn:lemma:no-regret-value-func-regret1}
        \EE^{\fA}_M\brac{\Regret}\geq \frac{\epsilon}{4}\EE^{\fA}_M\brac{N(o_H=s_0)}.
    \end{align}
    On the other hand, for the reference model $0$, we have
    \begin{align}\label{eqn:lemma:no-regret-value-func-regret2}
        \EE^{\fA}_0\brac{\Regret}\geq \frac{\epsilon}{4}\EE^{\fA}_0\brac{N(o_H\neq s_0)}+\frac14\EE^{\fA}_0\brac{N(E_{\rev})}.
    \end{align}
    where we define $E_{\rev}\defeq\set{\tau: \text{for some }h\in\cH, o_h=\olock, a_h\in\Arev}$.

On the other hand, for any policy $\pi$, we have    \begin{align}\label{eqn:no-regret-PAC-value-func}
        V_M^\star-V_M(\pi)\geq \frac{\epsilon}{4}\PP^{\pi}_M\paren{o_H=s_0} \ \ \forall M\neq 0, \qquad \text{and}\qquad V_0^\star-V_0(\pi)\geq \frac{\epsilon}{4}\PP^{\pi}_0\paren{o_H\neq s_0}.
    \end{align}
\end{lemma}
Therefore, we can relate the regret (or sub-optimality of the output policy) to the TV distance (under $\mu\sim\Unif(\set{-1,+1}^K)$ the prior distribution of parameter $\mu$), by an argument similar to the one in \cref{appendix:proof-1-step-pac-prop}, giving the following lemma whose proof is contained in \cref{app:proof-no-regret-TV-lower}. 
\begin{lemma}\label{lemma:no-regret-TV-lower}
    Suppose that either statement below holds for the algorithm $\fA$:
    
    (a) For any model $M\in\cM$, $\EE^{\fA}_M\brac{\Regret} \leq T\epsilon/32$.

    (b) For any model $M\in\cM$, the algorithm $\fA$ outputs a policy $\piout$ such that $\PP^{\fA}_M\paren{V_M^\star-V_M(\piout)<\frac{\epsilon}{8}}\geq \frac{3}{4}$.

Then we have
\begin{align}\label{eqn:regret-TV-lower}
    \DTV{ \PP^{\fA}_0, \Emu\brac{\PP^{\fA}_{\theta,\mu}} } \geq \frac12, \qquad \forall \theta.
\end{align}
\end{lemma}

By our assumptions in Claim 4 (or 5), in the following we only need to consider the case that \eqref{eqn:regret-TV-lower} holds for all $\theta$.
We will use \eqref{eqn:regret-TV-lower} to derive lower bounds of $\EE^{\fA}_0\brac{N(o_H\neq s_0)}$ and $\EE^{\fA}_0\brac{N(\Erev)}$, giving the following lemma whose proof is contained in \cref{appendix:proof-no-regret-alternative}. 

\begin{lemma}\label{lemma:no-regret-alternative}
    Fix a $\theta=(\hs,\ss,\acs,\revs,\as)$. We consider events
    \begin{align*}
        E_{\rev}^\theta&\defeq \set{o_{\hs}=\ss, a_{\hs:h}=(\acs,\as_{\hs+1:h-1},\revs) \text{ for some }h\in\cH_{>\hs}},\\
        E_{\corr}^\theta&\defeq \set{o_{\hs}=\ss, a_{\hs:H-1}=(\acs,\as)}.
    \end{align*}
    Then for any algorithm $\fA$ with $\delta\defeq \DTV{\PP_{0}^{\fA},\Emu\brac{\PP_{\theta,\mu}^{\fA}}}>0$, we have
    \begin{align*}
        \text{either }\EE^{\fA}_0\brac{N(\Erev^\theta)} \geq \frac{\delta^3\sqrt{K}}{18\epsilon^2\sigma^2}-\frac\delta6, \text{  or }\EE^{\fA}_0\brac{N(\Ecor^\theta)}\geq \frac{\delta^3}{18\epsilon^2}-\frac\delta6.
    \end{align*}
\end{lemma}
Applying \cref{lemma:no-regret-alternative} for any parameter tuple $\theta=(\hs,\ss,\acs,\revs,\as)$ with $\delta=\frac12$, we obtain 
\begin{align}\label{eqn:no-regret-alternative}
    \text{either }\EE^{\fA}_0\brac{\Nc{\Erev^{(\hs,\ss,\acs,\revs,\as)}}} \geq \frac{\sqrt{K}}{300\epsilon^2\sigma^2}, \quad \text{  or } \quad \EE^{\fA}_0\brac{\Nc{\Ecor^{(\hs,\ss,\acs,\revs,\as)}}}\geq \frac{1}{300\epsilon^2},
\end{align}
as we choose $\epsilon\in(0,0.1]$. %

Fix a tuple $(\hs,\ss,\acs)$ such that $\hs\in\cH$ and $\hs\leq n+m\floor{H/10m}$, $\ss\in\Sl$, $\acs\in\Ac$. By \eqref{eqn:no-regret-alternative}, we know that for all $\as\in\Apass$, $\revs\in\Arev$, $\theta=(\hs,\ss,\acs,\revs,\as)$, real constant $r\geq0$, it holds that
\begin{align*}
    &~\abs{\Arev}A^{m-1}\cdot\EE^{\fA}_0\brac{N\paren{\Erev^{(\hs,\ss,\acs,\revs,\as)}}} + r\abs{\Apass}\cdot\EE^{\fA}_0\brac{N\paren{\Ecor^{(\hs,\ss,\acs,\revs,\as)}}} \\
    \geq&~ \frac1{300}\min\set{\frac{\abs{\Arev}A^{m-1}\sqrt{K}}{\epsilon^2\sigma^2}, \frac{r\abs{\Apass}}{\epsilon^2}}
    \geq \frac1{300}\min\set{\frac{\abs{\Arev}A^{m-1}\sqrt{K}}{\epsilon^2\sigma^2}, \frac{rA^{H/2-1}}{\epsilon^2}}=:\omega_{r},
\end{align*}
where the last inequality follows from a direct calculation (see \cref{lemma:password-number}). Notice that
\begin{align*}
    &~\sum_{\revs\in\Arev,\as\in\Apass} \EE^{\fA}_0\brac{N\paren{\Erev^{(\hs,\ss,\acs,\revs,\as)}}}\\
    =&~
    \sum_{\revs\in\Arev,\as\in\Apass} \EE^{\fA}_0\brac{N\paren{o_{\hs}=\ss, a_{\hs:h}=(\acs,\as_{\hs+1:h-1},\revs) \text{ for some }h\in\cH_{>\hs}}}\\
    \leq&~
    \sum_{\revs\in\Arev,\as\in\Apass} \EE^{\fA}_0\brac{N\paren{o_{\hs}=\ss,a_{\hs:\hs+m-1}=(\acs,\as_{\hs+1:\hs+m-1}), a_h=\revs \text{ for some }h\in\cH_{>\hs}}}\\
    =&~
    \sum_{\revs\in\Arev,\a\in\cA^{m-1}} \EE^{\fA}_0\brac{N\paren{o_{\hs}=\ss,a_{\hs:\hs+m-1}=(\acs,\a), a_h=\revs \text{ for some }h\in\cH_{>\hs}}}\cdot \sum_{\substack{\as\in\Apass\\
    \as\text{ begins with }\a}} 1\\
    =&~
    \sum_{\revs\in\Arev,\a\in\cA^{m-1}} \EE^{\fA}_0\brac{N\paren{o_{\hs}=\ss,a_{\hs:\hs+m-1}=(\acs,\a), a_h=\revs \text{ for some }h\in\cH_{>\hs}}}\cdot \frac{\abs{\Apass}}{A^{m-1}}\\
    =&~
    \frac{\abs{\Apass}}{A^{m-1}}\cdot \EE^{\fA}_0\brac{N\paren{o_{\hs}=\ss,a_{\hs}=\acs, a_h\in\Arev \text{ for some }h\in\cH_{>\hs}}},
\end{align*}
where the second line is due to the inclusion of events, the fourth line follows from our definition of $\Apass$, and the last line is because the events $\set{o_{\hs}=\ss,a_{\hs:\hs+m-1}=(\acs,\a), a_h=\revs \text{ for some }h\in\cH_{>\hs}}$ are disjoint and their union is simply $\set{o_{\hs}=\ss,a_{\hs}=\acs, a_h\in\Arev \text{ for some }h\in\cH_{>\hs}}$.
Similarly we have
\begin{align*}
    \sum_{\revs\in\Arev,\as\in\Apass} \EE^{\fA}_0\brac{N\paren{\Ecor^{(\hs,\ss,\acs,\revs,\as)}}}
    =&~
    \sum_{\revs\in\Arev,\as\in\Apass} \EE^{\fA}_0\brac{N\paren{o_{\hs}=\ss, a_{\hs:H-1}=(\acs,\as)}}\\
    =&~
    \abs{\Arev}\cdot\sum_{\as\in\Apass} \EE^{\fA}_0\brac{N\paren{o_{\hs}=\ss, a_{\hs:H-1}=(\acs,\as)}}\\
    =&~
    \abs{\Arev}\cdot\EE^{\fA}_0\brac{N\paren{o_{\hs}=\ss, a_{\hs}=\acs,a_{\hs+1:H-1}\in\Apass}}.
\end{align*}
Combining all these facts, we obtain
\begin{align*}
    \omega_{r}
    \leq&~ \frac{1}{\abs{\Arev}\abs{\Apass}}
    \sum_{\revs\in\Arev}\sum_{\as\in\Apass} \paren{ \abs{\Arev}A^{m-1}\cdot\EE^{\fA}_0\brac{N\paren{\Erev^{(\hs,\ss,\acs,\revs,\as)}}} + r\abs{\Apass}\cdot\EE^{\fA}_0\brac{N\paren{\Ecor^{(\hs,\ss,\acs,\revs,\as)}}} }\\
    \leq&~
    \EE^{\fA}_0\brac{N\paren{o_{\hs}=\ss,a_{\hs}=\acs, a_h\in\Arev \text{ for some }h\in\cH_{>\hs}}}
    +r\EE^{\fA}_0\brac{N\paren{o_{\hs}=\ss, a_{\hs}=\acs}}
\end{align*}
Notice that the above inequality holds for any given $\ss\in\Sl, \acs\in\Ac, \hs\in\cH$ such that $\hs\leq n+m\floor{H/10m}$, and any $r\geq 0$. Therefore, we can take summation over all $\ss\in\Sl, \acs\in\Ac, \hs=n+lm\in\cH$ with $0\leq l\leq \floor{H/10m}$, and obtain
\begin{align*}
\MoveEqLeft
    \abs{\Sl}\abs{\Ac}(\floor{H/10m}+1)\cdot \min\set{\frac{\abs{\Arev}A^{m-1}\sqrt{K}}{300\epsilon^2\sigma^2}, \frac{rA^{H/2-1}}{300\epsilon^2}}
    =\sum_{\ss\in\Sl}\sum_{\acs\in\Ac}\sum_{\substack{\hs=n+lm:\\0\leq l\leq \floor{H/10m}}} \omega\\
    \leq&~
    \sum_{\ss\in\Sl}\sum_{\acs\in\Ac}\sum_{\substack{\hs=n+lm:\\0\leq l\leq \floor{H/10m}}}
    \EE^{\fA}_0\brac{N\paren{o_{\hs}=\ss,a_{\hs}=\acs, a_h\in\Arev \text{ for some }h\in\cH_{>\hs}}}
    +r\EE^{\fA}_0\brac{N\paren{o_{\hs}=\ss, a_{\hs}=\acs}}\\
    \leq&~\EE^{\fA}_0\brac{N\paren{\Erev}}+r \EE^{\fA}_0\brac{N\paren{o_H\neq s_0}},
\end{align*}
where the last inequality is because
\begin{align*}
    \bigsqcup_{\ss\in\Sl, \acs\in\Ac, \hs\in\cH} \set{o_{\hs}=\ss,a_{\hs}=\acs, a_h\in\Arev \text{ for some }h\in\cH_{>\hs}} \subseteq \set{\text{for some }h\in\cH, o_h=\olock, a_h\in\Arev}=\Erev,
\end{align*}
and $\bigsqcup_{\ss\in\Sl, \acs\in\Ac, \hs\in\cH}\set{o_{\hs}=\ss, a_{\hs}=\acs}\subseteq \set{o_H\neq s_0}$. Plugging in our choice $\abs{\Ac}=A-1\geq \frac{2}{3}A$, $\abs{\Arev}=1+\floor{A/6}\geq A/6$ and $\floor{H/10m}+1\geq H/10m$, we conclude the proof of the following claim:

\textbf{Claim:} as long as \eqref{eqn:regret-TV-lower} holds, we have
\begin{align}\label{eqn:no-regret-visitation-count}
    \EE^{\fA}_0\brac{N\paren{\Erev}}+r \EE^{\fA}_0\brac{N\paren{o_H\neq s_0}}\geq \frac{\abs{\Sl}H}{30000m}\cdot \min\set{\frac{A^{m+1}\sqrt{K}}{\epsilon^2\sigma^2}, \frac{rA^{H/2}}{\epsilon^2}}, \qquad \forall r\geq 0.
\end{align}

To deduce Claim 4 from the above fact, we notice that either (1) $\EE^{\fA}_M\brac{\Regret}>T\epsilon/32$ for some $M\in\cM$, or (2) $\EE^{\fA}_M\brac{\Regret}\leq T\epsilon/32$ for any $M\in\cM$, and then by \cref{lemma:no-regret-TV-lower}, \eqref{eqn:regret-TV-lower} holds, and hence we have
\begin{align*}
    \EE^{\fA}_0\brac{\Regret}\geq \frac14\EE^{\fA}_0\brac{N\paren{\Erev}}+\frac\epsilon4 \EE^{\fA}_0\brac{N\paren{o_H\neq s_0}}\geq \frac{\abs{\Sl}H}{120000m}\cdot \min\set{\frac{A^{m+1}\sqrt{K}}{\epsilon^2\sigma^2}, \frac{A^{H/2}}{\epsilon}}
\end{align*}
by setting $r=\epsilon$ in \eqref{eqn:no-regret-visitation-count}. Combining these two cases, we complete the proof of Claim 4 in \cref{prop:no-regret-prop}.

Similarly, suppose that the condition in Claim 5 holds, which implies \eqref{eqn:regret-TV-lower} (by \cref{lemma:no-regret-TV-lower}). Then we can set $r=1$ in \eqref{eqn:no-regret-visitation-count} to obtain
\begin{align*}
    2T\geq \EE^{\fA}_0\brac{N\paren{\Erev}}+\EE^{\fA}_0\brac{N\paren{o_H\neq s_0}}\geq \frac{\abs{\Sl}H}{30000m}\cdot \min\set{\frac{A^{m+1}\sqrt{K}}{\epsilon^2\sigma^2}, \frac{A^{H/2}}{\epsilon^2}},
\end{align*}
and hence complete the proof of Claim 4. This completes the proof of \cref{appendix:proof-no-regret-prop}. 
\qed

\begin{lemma}\label{lemma:password-number}
    As long as $\abs{\Arev}=A_1\leq 1+\floor{A/6}$, we have $\abs{\Apass}\geq A^{H/2-1}$ for $\hs\in\cH$ such that $\hs\leq n+m\floor{H/10m}$.
\end{lemma}
\begin{proof}
We denote $H_0=\floor{(H-n)/m}$, and assume that $\hs=n+ml$. Recall that
\begin{align*}
    \Apass\defeq \set{ \as=(\as_{\hs+1},\dots,\as_{H-1})\in\cA^{H-\hs-1}: \as_h\in\Atr, \forall h\in\cH_{>\hs}}.
\end{align*}
Hence, noticing that $\abs{\cH_{>\hs}}=H_0-l$, $\abs{\cA}=A$, $\abs{\Atr}=A-A_1$, we have
\begin{align*}
    \abs{\Apass}=A^{H-\hs-1-(H_0-l)}\times (A-A_1)^{H_0-l}.
\end{align*}
Thus, we only need to prove that
\begin{align}\label{eqn:a-password-calc}
    H-\hs-1-(H_0-l)+\frac{\log (A-A_1)}{\log A}(H_0-l) \geq \frac{H}{2}-1.
\end{align}
Notice that as long as $A_1\leq 1+\floor{A/6}$, it holds that $\frac{\log (A-A_1)}{\log A}\geq \frac{\log 2}{\log 3}=:w$. Using this fact and rearranging, we can see \eqref{eqn:a-password-calc} holds if 
$$
l\leq \frac{\frac{H}{2}-n-(1-w)H_0}{m-1+w}=:l_0.
$$
Now, using our assumption that $H\geq 10n$, we have
\begin{align*}
    l_0\geq \frac{\frac{H}{2}-n-(1-w)(H-n)}{mw}=\frac{(w-0.5)H-wn}{mw}\geq \frac{(w-0.5)H-0.1wH}{mw}\geq \frac{H}{10m},
\end{align*}
where the last inequality uses $w>\frac{5}{8}$.  Therefore, as long as $l\leq \floor{H/10m}$ (i.e. $\hs\leq n+m\floor{H/10m}$), we have $l\leq l_0$, which implies \eqref{eqn:a-password-calc} and hence completes the proof. %
\end{proof}

\subsection{Proof of Lemma \ref{lemma:regret-rev-2-step}}\label{appendix:proof-no-regret-revealing}

The idea here is similar to the proof of \cref{prop:1-step-pac-rev}, but as our construction is more involved, the direct description of $\M_{h,m+1}$ can be very complicated (even though actually only a few of its entries are non-zero). Therefore, in order to upper bound $\M_{h,m+1}$, we invoke the following lemmas, which will make our discussion cleaner.
\begin{lemma}\label{lemma:rev-star-to-1}
    For $m\geq 1$, $h\in[H-m]$, $\a\in\cA^{m}$, we consider 
    \begin{align*}
        \M_{h,\a}\defeq \brac{ \PP(o_{h:h+m}=\o|s_h=s,a_{h:h+m-1}=\a) }_{\o\in\cO^{m+1},s\in\cS}\in\R^{\cO^{m+1}\times\cS}.
    \end{align*}
    Then it holds that
    \begin{align*}
        \min_{\M_{h,m+1}^+}\nrmop{\M_{h,m+1}^+}\leq \min_{\M_{h,\a}^+} \loneone{\M_{h,\a}^+},
    \end{align*}
    where $\min_{\M_{h,\a}^+}$ is taken over all $\M_{h,\a}^+$ such that $\M_{h,\a}^+\M_{h,\a}\T_{h-1}=\T_{h-1}$ (cf.~\cref{definition:m-step-revealing}).
\end{lemma}

\begin{proof}[Proof of \cref{lemma:rev-star-to-1}]
    Notice that given a $\a\in\cA^m$, $\M_{h,\a}^+$ such that $\M_{h,\a}^+\M_{h,\a}\T_{h-1}=\T_{h-1}$, we can construct a generalized left inverse of $\M_{h,m}$ as follows:
    \begin{align*}
        \M_{h,m}^{\linv}=\begin{bmatrix}
            \vdots\\
            \II(\a'=\a)\M_{h,\a}^+\\
            \vdots
        \end{bmatrix}_{\a'\in\cA^m},
    \end{align*}
    and clearly $\nrmop{\M_{h,m}^{\linv}}\leq \loneone{\M_{h,\a}^+}$.
\end{proof}

In the following, for any matrix $M$, we write
\begin{align*}
    \gamma(M)\defeq \min_{M^+: M^+M=\id} \loneone{M^+}.
\end{align*}

\begin{lemma}\label{lemma:rev-sep-obs}
    Fix a step $h$ and a set of states $\cS_h$. Suppose that $\cS_h$ contains all $s\in\cS$ such that $\exists (s',a)\in\cS\times\cA$,  $\T_{h-1}(s|s',a)>0$. Further, suppose that $\cS_h$ can be partitioned as $\cS_h=\bigsqcup_{i=1}^n\cS_h^i$, such that for each $i\neq j$, $s\in\cS_h^i$, $s'\in\cS_h^j$,
    $$
    \supp(\M_{h,\a}(\cdot|s))\bigcap \supp(\M_{h,\a}(\cdot|s'))=\emptyset,
    $$
    i.e. the observations emitted from different $\cS_h^i$ are different.\footnote{
        In particular, this condition is fulfilled if for each $i\neq j$, $s\in\cS_h^i$, $s'\in\cS_h^j$, we have $\supp(\O_h(\cdot|s))\bigcap \supp(\O_h(\cdot|s'))=\emptyset$.
    }
    Then it holds that 
    $$
    \min_{\M_{h,\a}^+} \loneone{\M_{h,\a}^+}\leq \max\set{ \gamma\paren{\M_{h,\a}(\cS^1_h)}, \cdots, \gamma\paren{\M_{h,\a}(\cS^n_h)}},
    $$
    where
    \begin{align*}
        \M_{h,\a}(\cS')\defeq \brac{ \PP(o_{h:h+m}=\o|s_h=s,a_{h:h+m-1}=\a) }_{\o\in\cO^{m+1},s\in\cS'}\in\R^{\cO^{m+1}\times\cS'}, \qquad \text{for }\cS'\subset \cS_h.
    \end{align*}
\end{lemma}

\begin{proof}[Proof of \cref{lemma:rev-sep-obs}]
We first note that $\min_{\M_{h,\a}^+} \loneone{\M_{h,\a}^+}\leq \gamma\paren{\M_{h,\a}(\cS_h)}$, because the matrix $\M_{h,\a}(\cS_h)$ directly gives a generalized left inverse of $\M_{h,\a}$ (because $\cS_h$ contains all $s\in\cS$ such that $\exists (s',a)\in\cS\times\cA$,  $\T_{h-1}(s|s',a)>0$).

Next, as each $\cS_h^i$ has the disjoint set of possible observation, the matrix $\M_{h,\a}(\cS_h)$ can be written as (up to permutation of rows and columns, and any empty entry is zero)
\begin{align*}
    \M_{h,\a}(\cS_h)
    =\begin{bmatrix}
        \M_{h,\a}(\cS^1_h) & & & \\
        & \M_{h,\a}(\cS^2_h) & & \\
        & & \ddots & \\
        & & & \M_{h,\a}(\cS^n_h) \\
        & & &
    \end{bmatrix}.
\end{align*}
Therefore, suppose that for each $i$ we have a left inverse $\M_{h,\a}(\cS^i_h)^+$ of $\M_{h,\a}(\cS^i_h)$, then we can form a left inverse of $\M_{h,\a}(\cS_h)$ as
\begin{align*}
    \M_{h,\a}(\cS_h)^+
    =\begin{bmatrix}
        \M_{h,\a}(\cS^1_h)^+ & & & \\
        & \M_{h,\a}(\cS^2_h)^+ & & \\
        & & \ddots & \\
        & & & \M_{h,\a}(\cS^n_h)^+ \\
        & & &
    \end{bmatrix},
\end{align*}
and hence we derive that $\gamma\paren{\M_{h,\a}(\cS_h)}\leq \max\set{ \gamma\paren{\M_{h,\a}(\cS^1_h)}, \cdots, \gamma\paren{\M_{h,\a}(\cS^n_h)}}$.
\end{proof}

An important observation is that, for matrix $M\in\R^{m\times 1}$, we have $\gamma(M)\leq \frac{1}{\lone{M}}$. Thus, when the sum of entries of $M$ equals $1$, then $\gamma(M)\leq 1$. With the lemmas above, we now provide the proof of \cref{lemma:regret-rev-2-step}.

\begin{proof}[Proof of \cref{lemma:regret-rev-2-step}]
We first show that the null model $0$ is 1-step 1-revealing. In this model, the state $\sp$ and $\ep$ are not reachable, and hence for each step $h$, we consider the set $\cS'=\Stree\bigsqcup\set{\sq,\eq,\termin}$. For different states $s,s'\in\cS'$, the support of $\O_h(\cdot|s)$ and $\O_h(\cdot|s')$ are disjoint by our construction, and hence applying \cref{lemma:rev-sep-obs} gives
\begin{align*}
    \min_{\O_{h}^+}\loneone{\O_{h}^+}\leq \max_{s\in\cS'} \gamma\paren{\O_{h}(s)}\leq 1.
\end{align*}
Applying \cref{prop:m-step-mp1-step} completes the proof for null model $0$.

We next consider the non-null model $M=M_{\theta,\mu}\in\cM \setminus \{ M_0\}$. By our construction, for $h\leq \hs$, state $\sp$ and $\ep$ are not reachable, and hence by the same argument as in the null model, we obtain that $\min_{\M_{h,m+1}^+}\nrmop{\M_{h,m+1}^+} \leq \min_{\O_{h}^+}\loneone{\O_{h}^+}\leq 1$. 

Hence, we only need to bound the quantity $\min_{\M_{h,m+1}^+}\nrmop{\M_{h,m+1}^+}$ for a fixed step $h>\hs$. In this case, there exists a $l\in\cH$ such that $h\leq l\leq h+m-1$, and we write $r=l-h+1$. By \cref{lemma:rev-inc-step}, we only need to bound $\min_{\M_{h,r+1}^+}\nrmop{\M_{h,r+1}^+}$. Consider the action sequence $\a=(\as_{h:l-1},\revs)\in\cA^r$, and we partition $\cS$ as
\begin{align*}
    \cS=\bigsqcup_{s\in\Stree}\set{s} \sqcup \{\sp,\sq\} \sqcup \{\ep,\eq\} \sqcup \set{\termin}.
\end{align*}
It is direct to verify that, in $M_{\theta,\mu}$, for states $s,s'$ come from different subsets in the above partition, the support of $\M_{h,\a}(\cdot|s)$ and $\M_{h,\a}(\cdot|s')$ are disjoint.
Then, we can apply \cref{lemma:rev-star-to-1} and \cref{lemma:rev-sep-obs}, and obtain
\begin{align*}
    \min_{\M_{h,r+1}^+}\nrmop{\M_{h,r+1}^+}\leq \min_{\M_{h,\a}^+} \loneone{\M_{h,\a}^+}
    \leq \max\set{1,  \gamma\paren{\M_{h,\a}(\{\sp,\sq\})}, \gamma\paren{\M_{h,\a}(\{\ep,\eq\})} }.
\end{align*}
Therefore, in the following we only need to consider left inverses of the matrix $\M_{h,\a}(\{\sp,\sq\})$ and $\M_{h,\a}(\{\ep,\eq\})$.

(1) The matrix $\M_{h,\a}(\{\sp,\sq\})$. By our construction, taking $\a$ at $s_h=\sp$ will lead to $o_{h:l}=\olock$ and $o_{l+1}\sim \O_{\mu}(\cdot|\ep)$;  taking $\a$ at $s_h=\sq$ will lead to $o_{h:l}=\olock$ and $o_{l+1}\sim \O_{\mu}(\cdot|\eq)$. Hence, $\M_{h,\a}(\{\sp,\sq\})$ can be written as (up to permutation of rows)
\begin{align*}
    \M_{h,\a}(\{\sp,\sq\})=\begin{bmatrix}
        \frac{\II_{2K}+\sigma \tmu}{2K} & \frac{\II_{2K}}{2K}\\
        \mathbf{0} & \mathbf{0}
\end{bmatrix}\in\R^{\cO^{r+1}\times2},
\end{align*}
where $\tmu=[\mu;-\mu]\in\set{-1,1}^{2K}$, $\II=\II_{2K}$ is the vector in $\R^{2K}$ with all entries being one. Similar to \cref{prop:1-step-pac-rev}, we can directly verify that $\gamma\paren{\M_{h,\a}(\set{\sp,\sq})}\leq \frac{2}{\sigma}+1$.

(2) The matrix $\M_{h,\a}(\{\ep,\eq\})$. By our construction, at $s_h=\ep$, we have $o_h\sim \O_{\mu}(\cdot|\ep)$ and $o_{h+1:l+1}=\termin$; at $s_h=\eq$, we have $o_h\sim \O_{\mu}(\cdot|\eq)$ and $o_{h+1:l+1}=\termin$. Thus, $\M_{h,\a}(\{\ep,\eq\})$ can also be written as (up to permutation of rows)
\begin{align*}
    \M_{h,\a}(\{\ep,\eq\})=\begin{bmatrix}
        \frac{\II_{2K}+\sigma \tmu}{2K} & \frac{\II_{2K}}{2K}\\
        \mathbf{0} & \mathbf{0}
\end{bmatrix}\in\R^{\cO^{r+1}\times2},
\end{align*}
and hence we also have $\gamma\paren{\M_{h,\a}(\{\ep,\eq\})}\leq \frac{1}{\sigma}+2$.

Combining the two cases above gives
\begin{align*}
    \min_{\M_{h,m+1}^+}\nrmop{\M_{h,m+1}^+}
    \leq \min_{\M_{h,r+1}^+}\nrmop{\M_{h,r+1}^+}\leq \min_{\M_{h,\a}^+} \loneone{\M_{h,\a}^+}
    \leq \frac{2}{\sigma}+1,
\end{align*}
and hence completes the proof of \cref{lemma:regret-rev-2-step}.
\end{proof}

\begin{remark}\label{remark:mp-to-m-rev}
From the proof above, it is not easy to see the POMDP $M=M_{\theta,\mu}$ is not $m$-step revealing for any parameters $(\theta,\mu)$. Actually, for $\theta=(\hs,\ss,\acs,\revs,\as)$, we can show that the matrix $\M_{\hs+1,m}$ \emph{does not} admit a generalized left inverse. This is because for any $\a\in\cA^{m-1}$, we have
\begin{align*}
    \PP_{\theta,\mu}(o_{\hs+1:\hs+m}=\cdot|s_{\hs+1}=\sp,a_{\hs+1:\hs+m-1}=\a)
    =
    \PP_{\theta,\mu}(o_{\hs+1:\hs+m}=\cdot|s_{\hs+1}=\sq,a_{\hs+1:\hs+m-1}=\a),
\end{align*}
because both of the distributions are supported on the dummy observation $\olock^{\otimes m}$.
However, it is clear that $\e_{\sp}, \e_{\sq}\in\colspan{\T_{\hs}}$, and hence if $\M_{\hs+1,m}$ admits a generalized left inverse $\M_{\hs+1,m}^+$, then $\e_{\sp}=\M_{\hs+1,m}^+\M_{\hs+1,m}\e_{\sp}=\M_{\hs+1,m}^+\M_{\hs+1,m}\e_{\sq}=\e_{\sq}$, a contradiction! Therefore, we can conclude that $\M_{\hs+1,m}$ does not admit a generalized left inverse, and hence $M$ is not $m$-step revealing.
\end{remark}

\subsection{Proof of Lemma \ref{lemma:no-regret-value-func}}\label{app:proof-no-regret-value-func}

In the following, we prove \cref{eqn:lemma:no-regret-value-func-regret1} and \cref{eqn:lemma:no-regret-value-func-regret2}. This proof is very similar to the proof of \cref{lemma:1-step-pac-value-func}. The proof of \cref{eqn:no-regret-PAC-value-func} is very similar and hence omitted for succinctness.

Notice that by the definition of $\Regret$ and our construction of reward function, we have
\begin{align*}
    \EE_M^{\fA}\brac{\Regret}
    =&~ T\cdot V_M^\star - \EE_M^{\fA}\brac{\sum_{t=1}^T r_H(o_H^{(t)})}
    = T\cdot V_M^\star - \EE_M^{\fA}\brac{\frac{1+\epsilon}{4}\cdot N(o_H=s_0) + N(o_H=\oba)}\\
    =&~ \paren{V_M^\star-\frac{1+\epsilon}{4}}\EE_M^{\fA}\brac{N(o_H=s_0)} +  V_M^\star \EE_M^{\fA}\brac{N(o_H\neq s_0)}-\EE_M^{\fA}\brac{N(o_H=\oba)}
\end{align*}
and
\begin{align*}
    \EE_M^{\fA}\brac{N(o_H=\oba)}=&~\EE_M^{\fA}\brac{\sum_{t=1}^T \EE_M\cdbrac{\II(o_H^{(t)}=\oba)}{\taut_{H-1}}}\\
    =&~
    \EE_M^{\fA}\brac{\sum_{t=1}^T \sum_{\tau_{H-1}} \PP_M(o_H=\oba|\tau_{H-1}) \cdot\II(\taut_{H-1}=\tau_{H-1})}\\
    =&~\sum_{\tau_{H-1}} \EE_M^{\fA}\brac{N(\tau_{H-1})}\cdot \PP_M(o_H=\oba|\tau_{H-1}).
\end{align*}
We prove the result for the $M\neq 0$ and the case $M=0$ separately.

\textbf{Case 1:} $M=(\theta,\mu)\neq 0$. In this case, we have
\begin{align*}
    \PP_M(o_H=\oba|\tau_{H-1})=\frac34\PP_M(s_H=\sp|\tau_{H-1})+\frac14\PP_M(s_H=\sq|\tau_{H-1})\leq \frac14+\frac12\PP_M(s_H=\sp|\tau_{H-1})\leq \frac14+\frac12\epsilon,
\end{align*}
because $\PP_M(s_H=\sp|\tau_{H-1})\leq \epsilon$ by our construction. Thus, we have shown that
$$
\EE_M^{\fA}\brac{N(o_H=\oba)}\leq \paren{\frac14+\frac12\epsilon}\EE_M^{\fA}\brac{N(o_H\neq s_0)}.
$$

Notice that by this way we can also show that $V_M^\star=\frac{1+2\epsilon}{4}$. Therefore, combining the equations above, we conclude that
\begin{align*}
    \EE^{\fA}_M\brac{\Regret}\geq \frac{\epsilon}{4}\EE^{\fA}_M\brac{N(o_H=s_0)}.
\end{align*}

\textbf{Case 2:} $M=0$. In this case, $\sp$ is not reachable, and hence we have 
$$
\PP_0(o_H=\oba|\tau_{H-1})=\frac14\PP_0(s_H=\sq|\tau_{H-1})\leq \frac14.
$$
Also notice that, for any trajectory $\tau\in \Erev$, we have $\PP_0(o_H=\oba|\tau_{H-1})=0$. Thus, we have shown that
$$
\EE_0^{\fA}\brac{N(o_H=\oba)}\leq \frac14 \EE_0^{\fA}\brac{N(\set{o_H\neq s_0}-\Erev)}=\frac14 \EE_0^{\fA}\brac{N(o_H\neq s_0)}-\frac14 \EE_0^{\fA}\brac{N(\Erev)}.
$$

By this way we can also show that $V_0^\star=\frac{1+\epsilon}{4}$. Therefore, we can conclude that
\begin{align*}
    \EE_0^{\fA}\brac{\Regret}
    =&~ \frac{1+\epsilon}{4}\EE_0^{\fA}\brac{N(o_H\neq s_0)}-\EE_0^{\fA}\brac{N(o_H=\oba)}\\
    \geq&~ \frac{\epsilon}{4}\EE_0^{\fA}\brac{N(o_H\neq s_0)}+\frac14 \EE_0^{\fA}\brac{N(\Erev)}.
\end{align*}
This completes the proof of \cref{lemma:no-regret-value-func}. 
\qed

\subsection{Proof of Lemma \ref{lemma:no-regret-TV-lower}}\label{app:proof-no-regret-TV-lower}

We first consider case (a), i.e. suppose that $\EE^{\fA}_M\brac{\Regret} \leq T\epsilon/32$ for all $M\in\cM$. By Markov's inequality and \cref{eqn:lemma:no-regret-value-func-regret1} and \cref{eqn:lemma:no-regret-value-func-regret2}, it holds that
\begin{align*}
    &\PP^{\fA}_0\paren{N(o_H\neq s_0)\geq T/2}\leq \frac14,\\
    &\PP^{\fA}_M\paren{N(o_H=s_0)\geq T/2}\leq \frac14, \qquad \forall M \neq 0.
\end{align*}
In particular, for any fixed $\theta$, we consider the prior distribution of $M=(\theta,\mu)$ with $\mu\sim \Unif(\set{-1,1}^K)$, then
\begin{align*}
    \Emu\brac{\PP^{\fA}_{\theta,\mu}}\paren{N(o_H=s_0)\geq T/2}\leq \frac14.
\end{align*}
However, we also have
$$
\PP^{\fA}_0\paren{N(o_H=s_0)\geq T/2}=\PP^{\fA}_0\paren{N(o_H\neq s_0)\leq T/2}=1-\PP^{\fA}_0\paren{N(o_H\neq s_0)> T/2}\geq \frac34,
$$
and then by the definition of TV distance it holds
\begin{align*}
    \DTV{ \PP^{\fA}_0, \Emu\brac{\PP^{\fA}_{\theta,\mu}} } \geq \frac12.
\end{align*}
The proof of case (b) follows from an argument which is the same as the proof of \eqref{eqn:1-step-pac-TV-lower}, and hence omitted.
\qed

\subsection{Proof of Lemma \ref{lemma:no-regret-alternative}}\label{appendix:proof-no-regret-alternative}

We first prove the following version of \cref{lemma:no-regret-alternative} with an additional condition that the visitation counts are almost surely bounded under $\PP_0^{\fA}$, and then prove \cref{lemma:no-regret-alternative} by reducing to this case using a truncation argument.

To upper bound the above quantity, we invoke the following lemma, which serves a key step for bounding the above ``$\chi^2$-inner product'' \citep[Section 3.1]{canonne2022topics} between $\P_{\theta, \mu}/\P_0$ and $\P_{\theta,\mu'}/\P_0$ (proof in~\cref{appendix:proof-no-regret-MGF}).
\begin{lemma}[Bound on the $\chi^2$-inner product]\label{lemma:no-regret-alternative-weak}
   Suppose that algorithm $\fA$ (with possibly random stopping time $\stoptime$) satisfies  $N(\Erev^\theta)\leq \oN_o$ and $N(\Ecor^\theta)\leq\oN_r$ almost surely, for some fixed $\oN_o,\oN_r$. Then 
    \begin{align*}
        \text{either }\oN_o \geq \frac34\frac{\delta^2\sqrt{K}}{\epsilon^2\sigma^2}, \text{  or }\oN_r\geq \frac34\frac{\delta^2}{\epsilon^2},
    \end{align*}
    where $\delta=\DTV{\PP_{0}^{\fA},\Emu\brac{\PP_{\theta,\mu}^{\fA}}}$.
\end{lemma}

\begin{proof}[Proof of \cref{lemma:no-regret-alternative-weak}]
By \cref{lemma:ingster}, it holds that
\begin{align*}
    1+\chis{\Emu\brac{\PP_{\theta,\mu}^{\fA}}}{\PP_{0}^{\fA}}
    =
    \Emumu\EE_{\tau^{(1)},\cdots,\tau^{(\stoptime)}\sim\PP_{0}^{\fA}}\brac{
        \prod_{t=1}^\stoptime\frac{\PP_{\theta,\mu}(\taut)\PP_{\theta,\mu'}(\taut)}{\PP_{0}(\taut)^2}
    }.
\end{align*}
In the following lemma (proof in \cref{appendix:proof-no-regret-MGF}), we bound the LHS of the equality above.
\begin{lemma}
\label{lemma:no-regret-MGF}
Under the conditions of \cref{lemma:no-regret-alternative-weak}, it holds that for any $\mu,\mu'\in\set{-1,1}^K$,
\begin{align}\label{eqn:no-regret-MGF}
    \EE_0^\fA\brac{
        \prod_{t=1}^\stoptime\frac{\P_{\theta,\mu}(\taut)\P_{\theta,\mu'}(\taut)}{\P_{0}(\taut)^2}
    }
    \leq \exp\paren{ \oN_o\cdot \frac{\sigma^2\epsilon^2}{K}\abs{\iprod{\mu}{\mu'}}+\frac43\epsilon^2\oN_r }.
\end{align}
\end{lemma}
With \cref{lemma:no-regret-MGF}, we can take expectation of \eqref{eqn:no-regret-MGF} over $\mu,\mu'\sim \unif(\set{-1,+1}^K)$, and then
\begin{align*}
    1+\chis{\Emu\brac{\P_{\theta,\mu}^{\fA}}}{\P_{0}^{\fA}}
    =&
    \Emumu\EE_{\tau^{(1)},\cdots,\tau^{(\stoptime)}\sim\P_{0}^{\fA}}\brac{
        \prod_{t=1}^{\stoptime}\frac{\P_{\theta,\mu}(\taut)\P_{\theta,\mu'}(\taut)}{\P_{0}(\taut)^2}
    }\\
    \leq & \Emumu \brac{\exp\paren{ \oN_o\cdot \frac{\sigma^2\epsilon^2}{K}\abs{\iprod{\mu}{\mu'}}+\frac43\epsilon^2\oN_r }}.
\end{align*}
Notice that $\mu_i,\mu_i'$ are i.i.d. $\Unif(\{\pm 1\})$, and hence $\mu_1\mu_1',\cdots,\mu_K\mu_K'$ are i.i.d. $\Unif(\{\pm 1\})$. Then by Hoeffding's lemma, it holds that $\Emumu\brac{\exp\paren{x\sum_{i=1}^K\mu_i\mu_i'}}\leq \exp\paren{Kx^2/2}$ for all $x\in\R$, and thus by \cref{lemma:MGF-abs}, we have 
\begin{align*}
    \Emumu\brac{\exp\paren{\frac{\oN_o\sigma^2\epsilon^2}{K}\abs{\iprod{\mu}{\mu'}}}}\leq \exp\paren{\max\set{\frac{\sigma^4\epsilon^4\oN_o^2}{K}, \frac43\frac{\sigma^2\epsilon^2\oN_o}{\sqrt{K}}}}.
\end{align*}
Therefore, combining the above inequalities with \cref{lemma:TV-to-chi}, we obtain
\begin{align*}
    2\delta^2=2 \DTV{\Emu\brac{\PP_{\theta,\mu}^{\fA}}, \PP_{0}^{\fA}}^2
    \leq \log\paren{1+\chis{\Emu\brac{\PP_{\theta,\mu}^{\fA}}}{\PP_{0}^{\fA}}}
    \leq \max\set{ \frac43\frac{\oN_o\sigma^2\epsilon^2}{\sqrt{K}}, \frac{\oN_o^2\sigma^4\epsilon^4}{K}} +\frac43\epsilon^2\oN_r.
\end{align*}
Then, we either have $\oN_r \geq \frac{3\delta^2}{4\epsilon^2}$, or it holds
$$
\max\set{ \frac43\frac{\oN_o\sigma^2\epsilon^2}{\sqrt{K}}, \frac{\oN_o^2\sigma^4\epsilon^4}{K}} \geq \delta^2,
$$
which implies that $\frac{\oN_o\sigma^2\epsilon^2}{\sqrt{K}}\geq \min\set{\frac43,\frac34\delta^2}=\frac34\delta^2$ (as $\delta\leq 1$). The proof of \cref{lemma:no-regret-alternative-weak} is completed by rearranging.
\end{proof}

\begin{proof}[Proof of Lemma \ref{lemma:no-regret-alternative}]
We perform a truncation type argument to reduce Lemma \ref{lemma:no-regret-alternative} to Lemma \ref{lemma:no-regret-alternative-weak}, which is similar to the proof of \cref{lemma:1-step-pac-alternative}.

Let us take $\oN_o=\ceil{6\delta^{-1}\E_0^\fA \brac{N(\Erev^\theta)}}$ and $\oN_r=\ceil{6\delta^{-1}\E_0^\fA \brac{N(\Ecor^\theta)}}$.
By Markov's inequality, we have
\begin{align*}
    \PP_{0}^{\fA}\paren{N(\Erev^\theta)\geq \oN_o}\leq \frac\delta6, \qquad 
    \PP_{0}^{\fA}\paren{N(\Ecor^\theta)\geq \oN_r}\leq \frac\delta6.
\end{align*}
Therefore, we can consider the following exit criterion $\cexit$ for the algorithm $\fA$:
\[
{\cexit(\tau^{(1:T')})}=\TRUE \quad \textrm{iff} \quad \sum_{t=1}^{T'} {{\mathbb{I}}\paren{\taut\in\Erev^\theta}\geq \oN_o \text{ or } \sum_{t=1}^{T'} {{\mathbb{I}}\paren{\taut\in \Ecor^\theta}}\geq \oN_r.}
\]
The criterion $\cexit$ induces a stopping time $T_{\cexit}$, and we have
\begin{align*}
    \PP_0^\fA(\exists t < T, \cexit(\tau^{(1:t)})=\TRUE)
    \leq \PP_{0}^{\fA}\paren{N(\Erev^\theta)\geq \oN_o \text{ or } N(\Ecor^\theta)\geq \oN_r}\leq \frac\delta6+\frac\delta6\leq \frac\delta3.
\end{align*}
Therefore, we can consider the early stopped algorithm $\fA(\cexit)$ with exit criterion $\cexit$ (cf. \cref{appendix:ingster-method}), and by \cref{lemma:alg-stop} we have
\begin{align*}
\DTV{\PP_{0}^{\fA(\cexit)},\Emu\brac{\PP_{\theta,\mu}^{\fA(\cexit)}}}
\geq \DTV{\PP_{0}^{\fA},\Emu\brac{\PP_{\theta,\mu}^{\fA}}} - \PP_0^\fA(\exists t < T, \cexit(\tau^{(1:t)})=\TRUE) \geq \frac{2\delta}{3}.
\end{align*}

Notice that by our definition of $\cexit$ and stopping time $T_{\cexit}$, in the execution of $\fA(\cexit)$, we also have
\begin{align*}
    \sum_{t=1}^{T_{\cexit}-1} \IIc{\taut\in\Erev^\theta}< \oN_o, \qquad \sum_{t=1}^{T_{\cexit}-1} \IIc{\taut\in \Ecor^\theta}< \oN_r.
\end{align*}
Therefore, algorithm $\fA(\cexit)$ ensures that
\begin{align*}
    N(\Erevh^\theta)=\sum_{t=1}^{T_{\cexit}} \IIc{\taut\in\Erev^\theta} \leq \oN_o, \qquad N(\Ecor^\theta)=\sum_{t=1}^{T_{\cexit}} \IIc{\taut\in \Ecor^\theta}\leq \oN_r.
\end{align*}
Applying \cref{lemma:no-regret-alternative-weak} to the algorithm $\fA(\cexit)$ (and $\delta'=\frac23\delta$), we can obtain
\begin{align*}
    \text{either }~  \frac{\delta^2\sqrt{K}}{3\epsilon^2\sigma^2}
    \leq \oN_o \leq 6\delta^{-1}\E_0^\fA \brac{N(\Erev^\theta)} + 1, \qquad
    \text{  or }~ \frac{\delta^2}{3\epsilon^2} \leq \oN_r \leq 6\delta^{-1}\E_0^\fA \brac{N(\Ecor^\theta)} +1,
\end{align*}
and rearranging gives the desired result of \cref{lemma:no-regret-alternative}.
\end{proof}

\subsection{Proof of Lemma \ref{lemma:no-regret-MGF}}\label{appendix:proof-no-regret-MGF}
Throughout the proof, the parameters $\theta,\mu,\mu'$ are fixed.

By our discussion in \cref{appendix:ingster-method}, 
using~\cref{eqn:ingster-prod-to-i}, we have
\begin{align}\label{eqn:no-regret-ingster-prod-to-i}
    \EE_{\tau^{(1)},\cdots,\tau^{(\stoptime)}\sim\PP_{0}^{\fA}}\brac{
        \prod_{t=1}^\stoptime\frac{\PP_{M}(\taut)\PP_{M'}(\taut)}{\PP_{0}(\taut)^2}\cdot
        \exp\paren{-\sum_{t=1}^\stoptime\sum_{h=1}^H \log I(\taut_{h-1})}
    }=1,
\end{align}
where we have defined $I(\tau_l)$ for any partial trajectory $\tau_l$ up to step $l\in[H]$ as
\begin{align*}
    I(\tau_l)\defeq \EE_0\brcond{\frac{\PP_{\theta,\mu}(o_{l+1}|\tau_l)\PP_{\theta,\mu'}(o_{l+1}|\tau_l)}{\PP_{0}(o_{l+1}|\tau_l)^2}}{\tau_l}.
\end{align*}

Notice that the model $\PP_{\theta,\mu}$ and $\PP_0$ are different only at the transition from $s_{\hs}=\ss,a_{\hs}=\acs$ to $\sp$ and the dynamic at the component $\set{\sp,\ep}$.
Therefore, for any (reachable) trajectory $\tau_l=(o_1,a_1,\cdots,o_l,a_l)$, we can consider the implication of $\PP_{\theta,\mu}(o_{l+1}=\cdot|\tau_l)\neq \PP_{0}(o_{l+1}=\cdot|\tau_l)$: \\
1. Clearly, $o_{\hs}=\ss$, $a_{\hs}=\acs$ (i.e. $l\geq \hs+1$ and taking action $a_{1:\hs-1}$ from $s_0$ will result in $\ss$ at step $\hs$). \\
2. Either $a_{\hs+1:l}=(\as_{\hs+1:l-1},\revs)$ for some $l\in\cH_{>\hs}$, or $l=H-1$ and $a_{\hs+1:H-1}=\as$.

Hence, for $l\in\cH_{>\hs}$, we define
$$
E_{\rev,l}^\theta\defeq \set{
o_{\hs}=\ss, a_{\hs:l}=(\acs,\as_{\hs+1:l-1},\revs)
}.
$$
Also recall that we define $E_{\corr}^\theta\defeq \set{o_{\hs}=\ss, a_{\hs:H-1}=(\acs,\as)}$.
Then if $\PP_{\theta,\mu}(\cdot|\tau_l)\neq \PP_{0}(\cdot|\tau_l)$, then either $l\in\cH_{>\hs}, \tau_l\in E_{\rev,l}^\theta$, or $l=H-1, \tau_{H-1}\in\Ecor^\theta $. %
In other words, for any $\tau_l$ (that is reachable under $\PP_0$), we have $I(\tau_l)=1$ except for these two cases, and it remains to compute $I(\tau_l)$ for these two cases.

\textbf{Case 1:} $l\in\cH_{>\hs}, \tau_l\in E_{\rev,l}^\theta$. In this case, we have
\begin{align*}
    \P_{\theta,\mu}(o_{l+1}=o|\tau_l)
    =&~
    \P_{\theta,\mu}(o_{l+1}=o|s_{l+1}=\ep )\P_{\theta,\mu}(s_{l+1}=\ep|\tau_l)
    +
    \P_{\theta,\mu}(o_{l+1}=o|s_{l+1}=\eq)\P_{\theta,\mu}(s_{l+1}=\eq|\tau_l)\\
    =&~
    \paren{\O_\mu(o|\ep)-\O(o|\eq)}\cdot \P_{\theta,\mu}(s_{l+1}=\ep|\tau_l)+\O(o|\eq),
\end{align*}
where the second equality is because conditional on $\tau_l$, we have $s_{l+1}\in\set{\ep,\eq}$. Now, we have
\begin{align*}
    \P_{\theta,\mu}(s_{l+1}=\ep|\tau_l)
    =&~
    \P_{\theta,\mu}(s_{l+1}=\ep|o_{\hs}=\ss,a_{\hs:l}=(\acs,\as_{\hs+1:l-1},\revs))\\
    =&~
    \P_{\theta,\mu}(s_{\hs+1}=\sp|o_{\hs}=\ss,a_{\hs}=\acs)=\epsilon.
\end{align*}
Hence, by the definition of $\O_\mu(\cdot|\ep)$ and $\O(\cdot|\eq)$, we can conclude that
\begin{align*}
    \P_{\theta,\mu}(o_{l+1}=o_i^+|\tau_l)=\frac{1+\epsilon\sigma\mu_{i}}{2K}, \qquad 
    \P_{\theta,\mu}(o_{l+1}=o_i^-|\tau_l)=\frac{1-\epsilon\sigma\mu_{i}}{2K}, \qquad \forall i\in[K].
\end{align*}
On the other hand, clearly $\PP_{0}(o_{l+1}=\cdot|\tau_l)=\Unif(\{o_1^+,o_1^-,\cdots,o_K^+,o_K^-\})$.
Hence, it holds that
\begin{align*}
    I(\tau_l)
    =&
    \EE_0\brcond{\frac{\PP_{\theta,\mu}(o_{l+1}|\tau_l)\PP_{\theta,\mu'}(o_{l+1}|\tau_l)}{\PP_{0}(o_{l+1}|\tau_l)^2}}{\tau_l}\\
    =&
    \frac{1}{2K}\sum_{o\in\cO_o} \frac{\PP_{\theta,\mu}(o_{l+1}=o|\tau_l)\PP_{\theta,\mu'}(o_{l+1}=o|\tau_l)}{\PP_{0}(o_{l+1}=o|\tau_l)^2}\\
    =&
    \frac{1}{2K} \sum_{i=1}^K (1+\epsilon\sigma\mu_i)(1+\epsilon\sigma\mu_i')+(1-\epsilon\sigma\mu_i)(1-\epsilon\sigma\mu_i')\\
    =&1+\frac{\epsilon^2\sigma^2}{K}\sum_{i=1}^K \mu_i\mu_i'=1+\frac{\epsilon^2\sigma^2}{K}\iprod{\mu}{\mu'}.
\end{align*}

\textbf{Case 2:} $l=H-1. \tau_{H-1}\in\Ecor^\theta$. In this case, the distribution $\PP(o_{H}=\cdot|\tau_{H-1})$ is supported on $\set{\oba,\obb}$. Similar to case 1, we have
\begin{align*}
    \P_{\theta,\mu}(o_{H}=\cdot|\tau_l)
    =&~
    \P_{\theta,\mu}(o_{H}=\cdot|s_{H}=\sp )\P_{\theta,\mu}(s_{H}=\sp|\tau_{H-1})
    +
    \P_{\theta,\mu}(o_{H}=\cdot|s_{H}=\sq )\P_{\theta,\mu}(s_{H}=\sq|\tau_{H-1})\\
    =&~
    \paren{\O_H(\cdot|\sp)-\O_H(\cdot|\sq)}\cdot \P_{\theta,\mu}(s_{H}=\sp|\tau_{H-1})+\O_H(\cdot|\sq),
\end{align*}
where the second equality is because conditional on $\tau_{H-1}\in\Ecor^\theta$, we have $s_{H}\in\set{\sp,\sq}$. Now, we have
\begin{align*}
    \P_{\theta,\mu}(s_{H}=\sp|\tau_{H-1})
    =
    \P_{\theta,\mu}(s_{H}=\sp|o_{\hs}=\ss,a_{\hs:H-1}=(\acs,\as))
    =
    \P_{\theta,\mu}(s_{\hs+1}=\sp|o_{\hs}=\ss,a_{\hs}=\acs)=\epsilon.
\end{align*}
Hence, by the definition of $\O_H(\cdot|\sp)$ and $\O_H(\cdot|\sq)$, we have
\begin{align*}
    \PP_{\theta,\mu}(o_{H}=o|\tau_{H-1})
    =&
    \begin{cases}
        \frac{1+2\epsilon}{4}, & o=\oba,\\
        \frac{3-2\epsilon}{4}, & o=\obb.\\
    \end{cases}
\end{align*}
On the other hand, clearly $\PP_{0}(o_{H}=\oba|\tau_{H-1})=\frac14, \PP_{0}(o_{H}=\obb|\tau_{H-1})=\frac34$.
Therefore, in this case, we have
\begin{align*}
    I(\tau_{H-1})
    =
    \EE_0\brcond{\frac{\PP_{\theta,\mu}(o_{H}|\tau_{H-1})\PP_{\theta,\mu'}(o_{H}|\tau_{H-1})}{\PP_{0}(o_{H}|\tau_{H-1})^2}}{\tau_{H-1}}
    =
    \frac14\times \paren{ \frac{(1+2\epsilon)/4}{1/4}}^2+\frac34\times \paren{ \frac{(3-2\epsilon)/4}{3/4}}^2=1+\frac43 \epsilon^2.
\end{align*}

Combining the two cases above, we obtain
\begin{align*}
    I(\tau_l)=\begin{cases}
        1+\frac{\epsilon^2\sigma^2}{K}\iprod{\mu}{\mu'}, & l\in\cH_{>\hs}, \tau_l\in E_{\rev,l}^\theta,\\
        1+\frac43 \epsilon^2, & l=H-1, \tau_{H-1}\in\Ecor^\theta,\\
        1, & \text{otherwise}.
    \end{cases}
\end{align*}
Hence, for each $t\in[\stoptime]$,
\begin{align*}
    \sum_{l=0}^{H-1} \log I(\taut_{l})
    =&~
    \sum_{l=0}^{H-1} \IIc{l\in\cH_{>\hs}, \taut_l\in E_{\rev,l}^\theta}\cdot \log\paren{1+\frac{\epsilon^2\sigma^2}{K}\iprod{\mu}{\mu'}}+\IIc{\taut_{H-1}\in\Ecor^\theta}\cdot \log\paren{1+\frac43 \epsilon^2}\\
    \leq&~ 
    \sum_{l\in\cH_{>\hs}} \IIc{\taut_l\in E_{\rev,l}^\theta}\cdot \frac{\epsilon^2\sigma^2}{K}\iprod{\mu}{\mu'}
    +\IIc{\taut_{H-1}\in\Ecor^\theta}\cdot \frac43 \epsilon^2\\
    =&~
    \IIc{\taut_H\in E_{\rev}^\theta}\cdot \frac{\epsilon^2\sigma^2}{K}\iprod{\mu}{\mu'}+\IIc{ \taut_H\in\Ecor^\theta} \cdot \frac43 \epsilon^2,
\end{align*}
where the last equality is because $\Erev^\theta=\bigsqcup_{l:l\in\cH_{>\hs}} E_{\rev,l}^\theta$. 
Taking summation over $t\in[\stoptime]$, we obtain 
\begin{align*}
    \sum_{t=1}^\stoptime \sum_{l=0}^{H-1} \log I(\taut_{l})
    \leq&~ 
    \sum_{t=1}^\stoptime \IIc{ \taut_H\in\Ecor^\theta} \cdot \frac43 \epsilon^2+ \IIc{\taut_H\in E_{\rev}^\theta}\cdot \frac{\epsilon^2\sigma^2}{K}\iprod{\mu}{\mu'}\\
    =&~ 
    N(\Ecor^\theta)\cdot \frac43 \epsilon^2+ N(\Erev^\theta)\cdot \frac{\epsilon^2\sigma^2}{K}\iprod{\mu}{\mu'}\\
    \leq&~
    \oN_r\cdot \frac43 \epsilon^2+ \oN_o\cdot \frac{\epsilon^2\sigma^2}{K}\abs{\iprod{\mu}{\mu'}}.
\end{align*}
Plugging the above inequality into \eqref{eqn:no-regret-ingster-prod-to-i} completes the proof of \cref{lemma:no-regret-MGF}. 
\qed

\subsection{Regret calculation for hard instance in Section~\ref{section:proof-overview}}
\label{appdx:heuristic_regret}

For the hard instance presented in~\cref{section:proof-overview}, we notice that any algorithm either incurs a $\Omega(\epsilon T)$ regret, or must have successfully identified $(\hs,\ss,\acs)$ within $T$ episodes of play, which requires either at least $\Omega(SAH\times\sqrt{O}/(\sigma^2\eps^2))$ episodes of taking revealing actions, each being $\Theta(1)$-suboptimal, or at least $\Omega(SAH\times A^{\Theta(H)}/\eps^2)$ episodes of trying out all possible action sequences, each being $\Theta(\eps)$-suboptimal. This yields a regret lower bound 
\begin{align*}\textstyle
\Omega\Big(SAH\times \min \Big\{ \frac{\sqrt{O}}{\sigma^2\epsilon^2}, \frac{A^{\ThO{H}}}{\epsilon^2}\cdot \epsilon \Big\} \Big) \wedge \Omega(\eps T).
\end{align*}
Optimizing over $\eps>0$, we obtain a $\Omega(T^{2/3})$-type regret lower bound (for $T\ll A^{\cO(H)}$) similar as (though slightly worse rate than)~\cref{thm:no-regret-demo}.

\section{Proof of Theorem \ref{thm:multi-step-pac-demo}}\label{appdx:multi-step-pac}

We first construct a family of hard instances in~\cref{appendix:construction-multi-step-pac}. We then state the PAC lower bound of this family of hard instances in~\cref{prop:multi-step-prop}.~\cref{thm:multi-step-pac-demo} then follows from combining~\cref{prop:multi-step-prop} with \cref{prop:no-regret-prop}.

\subsection{Construction of hard instances and proof of Theorem~\ref{thm:multi-step-pac-demo}}
\label{appendix:construction-multi-step-pac}

We consider the following family of $m$-step revealing POMDPs $\cM$ that admits a tuple of hyperparameters $(\epsilon,\sigma,n,m,K,L,H)$. All POMDPs in $\cM$ share the state space $\cS$, action space $\cA$, observation space $\cO$, and horizon length $H$, defined as following.
\begin{itemize}[topsep=0pt, itemsep=0pt]
    \item The state space $\cS=\Stree \bigsqcup_{j=1}^L \set{\sp^j, \sq^j, \ep^j, \eq^j, \termin^j}$, where $\Stree$ is a binary tree with level $n$ (so that $\abs{\Stree}=2^{n}-1$).  
    Let $s_0$ be the root of $\Stree$, and $\Sl$ be the set of leaves of $\Stree$, with $\abs{\Sl}=2^{n-1}$.
    \item The observation space $\cO=\Stree \bigsqcup \set{o_1^+,o_1^-,\cdots,o_K^+,o_K^-}\bigsqcup\set{\olock,\oba,\obb}\bigsqcup_{j=1}^L \set{\olock^j, \termin^j}$.
    \item The action space $\cA=\set{0,1,\cdots,A-1}$.
\end{itemize}
We further define $\acrev=0\in\cA, \Ac=\set{1,\cdots,A-1}.$

\paragraph{Model parameters} 
Each non-null POMDP model $M=M_{\theta,\mu}\in\cM \setminus \{ M_0\}$ is specified by parameters $(\theta, \mu)$, where $\mu\in\set{-1,+1}^{L\times K}$, and $\theta=(\hs,\ss,\acs,\as)$, where 
\begin{itemize}[topsep=0pt, itemsep=0pt]
    \item $\ss\in\Sl$, $\acs\in \Ac\defeq\set{1,\cdots,A-1}$.
    \item $\hs\in\cH\defeq \set{h=n+lm: l\in\Z_{\geq 0}, h<H}$.
    \item $\as=(\as_{\hs+1},\dots,\as_{H-1})\in\cA^{H-\hs-1}$ is an action sequence indexed by $\hs+1,\cdots,H-1$, such that when $h\in\cH$, we have $\as_h\neq\acrev$. We use $\Apass$ to denote the set of all such $\as$.
\end{itemize}
Our construction will ensure that, only at steps $h\in\cH$ and states $s_h\in\set{\sp,\sq}$, the agent can observe $\olock^j$ and take action $\acrev$ to transit to $\set{\ep^j,\eq^j}$.

For any POMDP $M_{\theta, \mu}$, its system dynamics $\PP_{\theta,\mu}\defeq \PP_{M_{\theta, \mu}}$ is defined as follows.

\paragraph{Emission dynamics} At state $s\in\Stree\cup \set{\termin}$, the agent always receives (the unique observation corresponding to) $s$ itself as the observation.
\begin{itemize}[topsep=0pt, itemsep=0pt]
    \item At state $\ep^j$, the emission dynamics is given by
    \begin{align*}
        \O_{\mu}(o_i^+|\ep^j)=\frac{1+\sigma\mu_{j,i}}{2K}, \qquad 
        \O_{\mu}(o_i^-|\ep^j)=\frac{1-\sigma\mu_{j,i}}{2K}, \qquad
        \forall i\in[K],
    \end{align*}
    where we omit the subscript $h$ because the emission distribution does not depend on $h$.
    \item At state $\eq^j$, the observation is uniformly drawn from $\cO_o\defeq \set{o_1^+,o_1^-,\cdots,o_K^+,o_K^-}$, i.e. $\O(\cdot|\eq^j)=\Unif(\cO_o)$.
    \item At states $s\in\{\sp^j,\sq^j\}$:
    \begin{itemize}
        \item For steps $h\in\cH$, the agent always receives $\olock^j$ as the observation.
        \item For steps $h\leq H-1$ that does not belong to $\cH$, the agent always receives $\olock$ as the observation.
        \item At step $H$, the emission dynamics at $\{\sp^j,\sq^j\}$ is given by
        \begin{align*}
            &\O_H(\oba|\sp^j)=\frac34, \qquad \O_H(\obb|\sp^j)=\frac14, \\
            &\O_H(\oba|\sq^j)=\frac14, \qquad \O_H(\obb|\sq^j)=\frac34.
        \end{align*}
    \end{itemize}
\end{itemize}

\paragraph{Transition dynamics} In each episode, the agent always starts at $s_0$. 
\begin{itemize}[topsep=0pt, itemsep=0pt]
    \item At any node $s\in\Stree\setminus \Sl$, there are three types of available actions: $\await=0$, $\aleft=1$ and $\aright=2$, such that the agent can take $\await$ to stay at $s$, $\aleft$ to transit to the left child of $s$ and $\aright$ to transit to the right child of $s$.
    \item At any $s\in\Sl$, the agent can take action $\await=0$ to stay at $s$ (i.e. $\PP(s|s,\await)=1$); otherwise, for $s\in\Sl$, $h\in[H-1]$, $a\neq \await$,
    \begin{align*}
        \PP_{h;\theta}(\sp^j|s,a)&=\frac{\epsilon}{L}\cdot \II(h=\hs,s=\ss,a=\acs), \\
        \PP_{h;\theta}(\sq^j|s,a)&=\frac{1}{L}-\frac{\epsilon}{L}\cdot \II(h=\hs,s=\ss,a=\acs).
    \end{align*}

    \item The states $s\in\{\ep^j,\eq^j\}$ always transit to $\termin^j$, regardless of the action taken.
    
    \item The states $\termin^1,\cdots,\termin^L$ are absorbing states.

    \item At states $s\in\{\sp^j,\sq^j\}$: %
\begin{itemize}
    \item For the step $h\in\cH$, we set
    \begin{align*}
        \PP_{h;\theta}(\ep^j|\sp^j,\acrev)=1, \qquad
        \PP_{h;\theta}(\eq^j|\sq^j,\acrev)=1.
    \end{align*}
    In words, at steps $h\in\cH$ and states $s\in\{\sp^j,\sq^j\}$ (corresponding to $o=\olock^j$), the agent can take action $\acrev$ to transit to $\{\ep^j,\eq^j\}$, respectively. Note that $\cH = \set{h=n+lm: h<H, l \in \Z_{\ge 0}}$, so we only allow the agent to take the reveal action $\acrev$ every $m$ steps, which ensures that our construction is $(m+1)$-step revealing.
    \item For $h\not\in\cH$ or $a\neq\acrev$, we set
    \begin{align*}
        \PP_{h;\theta}(\sp^j|\sp^j,a)=\II(a=\as_h), \qquad \PP_{h;\theta}(\sq^j|\sp^j,a)&=\II(a\neq\as_h),\\
        \PP_{h;\theta}(\sq^j|\sq^j,a)&=1.
    \end{align*}
\end{itemize}
\end{itemize}

\paragraph{Reward} The reward function is known (and only depends on the observation): at the first $H-1$ steps, no reward is given; at step $H$, we set $r_H(\oba)=1$, $r_H(\obb)=0$, 
$r_H(s_0)=(1+\epsilon)/4$, and $r_H(o)=0$ for any other $o\in\cO$. 

\paragraph{Reference model}
We use $M_0$ (or simply $0$) to refer to the null model (reference model). The null model $M_0$ has transition and emission the same as any non-null model, except that the agent always arrives at $\sq^j$ (with $j\sim\Unif([L])$) by taking any action $a\neq\await$ at $s \in\Sl$ and $h \in [H - 1]$ (i.e., $\PP_{h;M_0}(\sq^j|s,a) = \frac{1}{L}$ for any $s \in \Sl$, $a \in \Ac$, $h \in [H-1]$). In this model, states in $\{\sp^1,\ep^1,\cdots,\sp^L,\ep^L\}$ are all not reachable, and hence we do not need to specify the transition and emission dynamics at these states.

We summarize the results of the hard instances we construct in the following proposition, which we prove in~\cref{appendix:proof-multi-step-prop}.

\begin{proposition}\label{prop:multi-step-prop}
For given $\epsilon \in(0,0.1], \sigma\in(0,1]$, $m,n\geq 1$, $K,L\geq 1$, $H\geq 8n+m+1$, the above model class $\cM$ satisfies the following properties.

\begin{enumerate}[topsep=0pt, itemsep=0pt]
\item $\nS=2^n+5L$, $\nO=2^n+2K+2L+3$, $\nA=A$.

\item For each $M\in\cM$, $M$ is $(m+1)$-step revealing with $\arevmp(M)^{-1}\leq 1+\frac2\sigma$.

\item $\log\abs{\cM}\leq LK\log 2 + H\log A + \log(SAH)$.

\item Suppose algorithm $\fA$ interacts with the environment for $T$ episodes and returns $\piout$ such that
$$
\PP^{\fA}_M\paren{V_M^\star-V_M(\piout)<\frac{\epsilon}{8}}\geq \frac{3}{4}.
$$
for any $M\in\cM$. Then it must hold that
\begin{align*}
    T\geq \frac{1}{10000m}\min\set{
        \frac{\abs{\Sl}\sqrt{LK}A^{m}H}{\sigma^2\epsilon^2}, \frac{\abs{\Sl}A^{H/2}H}{\epsilon^2}
    }.
\end{align*}
\end{enumerate}
\end{proposition}

\begin{proofof}[thm:multi-step-pac-demo]
We have to suitably choose parameters when applying \cref{prop:multi-step-prop}. More specifically, given $(S,O,A,H,\arev,m)$, we can let $m'=m-1$, and take $n\geq 1$ to be the largest integer such that $2^n\leq S/4$, and take $L=\floor{(S-2^n)/5}$, $K=\floor{\frac{O-2^n-2L-3}{2}}\geqsim O$ (because $O\geq S\geq 10$), $\epsilon'=\epsilon/8$, and $\sigma=\frac{2}{\arev^{-1}-1}\leq 1$. Applying \cref{prop:multi-step-prop} to the parameters $(\epsilon,\sigma,n,m',K,L,H)$, we obtain a model class $\cM$ of $m$-step $\arev$-revealing POMDPs, such that if there exists an algorithm $\fA$ that interacts with the environment for $T$ episodes and returns a $\piout$ such that 
$V_M^\star-V_M (\piout) <\epsilon$ with probability at least $3/4$ for all $M\in\cM$, then
\begin{align*}
    T\geq \frac{c_0}{m}\min\set{
        \frac{S^{3/2}O^{1/2}A^{m-1}H}{\arev^2\epsilon^2}, \frac{SA^{H/2}H}{\epsilon^2}
    },
\end{align*}
where $c_0$ is a universal constant. 

Furthermore, we can apply \cref{prop:no-regret-prop} (claim 5) instead, and similarly obtain a model class $\cM'$ of $m$-step $\arev$-revealing POMDPs, such that if there exists an algorithm $\fA$ that interacts with the environment for $T$ episodes and returns a $\piout$ such that 
$V_M^\star-V_M (\piout) <\epsilon$ with probability at least $3/4$ for all $M\in\cM'$, then
\begin{align*}
    T\geq \frac{c_0'}{m}\min\set{
        \frac{SO^{1/2}A^{m}H}{\arev^2\epsilon^2}, \frac{SA^{H/2}H}{\epsilon^2}
    },
\end{align*}
where $c_0'$ is a universal constant.

Combining these two cases completes the proof of \cref{thm:multi-step-pac-demo}.
\end{proofof}

\subsection{Proof of Proposition \ref{prop:multi-step-prop}}
\label{appendix:proof-multi-step-prop}

All propositions and lemmas stated in this section are proved in~\cref{appendix:proof-multi-step-revealing}-\ref{appendix:proof-multi-step-alternative}. %

Claim 1 follows directly by counting the number of states, observations, and actions in models in $\cM$. Claim 3 follows as we have $\abs{\cM}=\abs{\set{(h^\star,s^\star,a^\star,\a^\star)}}\times \abs{\set{\pm 1}^{L\times K}}+1\le HSA^H\times 2^{LK}$. Taking logarithm yields the claim.

Claim 2 follows from this lemma, which is proved in \cref{appendix:proof-multi-step-revealing}.
\begin{lemma}\label{lemma:multi-step-rev}
    For each $M\in\cM$, it holds that $\arevmp(M)^{-1}\leq \frac{2}{\sigma}+1$.
\end{lemma}

By our construction, we can relate the sub-optimality of the output policy to the TV distance between models (under the prior distribution of parameter $\mu\sim\Unif(\set{-1,+1}^{L\times K})$), by an argument similar to the one in \cref{appendix:proof-1-step-pac-prop}. We summarize the results in the following lemma, whose proof is omitted for succinctness. 
\begin{lemma}[Relating learning to testing]
In holds that
\begin{align*}
    V_M^\star-V_M(\pi)\geq \frac{\epsilon}{4}\PP^{\pi}_M\paren{o_H=s_0} \ \ \forall M\neq 0, \qquad \text{and}\qquad V_0^\star-V_0(\pi)\geq \frac{\epsilon}{4}\PP^{\pi}_0\paren{o_H\neq s_0}.
\end{align*}
Therefore, suppose that the algorithm $\fA$ outputs a policy $\piout$ such that $\PP^{\fA}_M\paren{V_M^\star-V_M(\piout)<\frac{\epsilon}{8}}\geq \frac{3}{4}$ for any model $M\in\cM$, then we have
\begin{align}\label{eqn:multi-step-TV-lower}
    \DTV{ \PP^{\fA}_0, \Emu\brac{\PP^{\fA}_{\theta,\mu}} } \geq \frac12, \qquad \forall \theta.
\end{align}
\end{lemma}

In the following, we use \eqref{eqn:multi-step-TV-lower} to derive lower bounds of the expected visitation count of some good events, and then deduce a lower bound of $T$, giving the following lemma whose proof is contained in \cref{appendix:proof-multi-step-alternative}.
\begin{lemma}\label{lemma:multi-step-alternative}
    Fix a $\theta=(\hs,\ss,\acs,\as)$. We consider events
    \begin{align*}
        \Ereach^\theta&\defeq \set{o_{\hs}=\ss, a_{\hs:\hs+m-1}=(\acs,\as_{\hs+1:\hs+m-1})}, \\
        E_{\corr}^\theta&\defeq \set{o_{\hs}=\ss, a_{\hs:H-1}=(\acs,\as)}.
    \end{align*}
    Then for any algorithm $\fA$ with $\delta\defeq\DTV{\PP_{0}^{\fA},\Emu\brac{\PP_{\theta,\mu}^{\fA}}}>0$, we have
    \begin{align*}
        \text{either }\EE^{\fA}_0\brac{N(\Ereach^\theta)} \geq \frac{\delta^3\sqrt{LK}}{18\epsilon^2\sigma^2}-\frac\delta6, \text{  or }\EE^{\fA}_0\brac{N(\Ecor^\theta)}\geq \frac{\delta^3}{18\epsilon^2}-\frac\delta6.
    \end{align*}
\end{lemma}

Applying \cref{lemma:multi-step-alternative} for any parameter tuple $\theta=(\hs,\ss,\acs,\as)$ with $\delta=\frac12$, we obtain 
\begin{align}\label{eqn:multi-step-pac-alternative}
    \text{either }\EE^{\fA}_0\brac{\Nc{\Ereach^{(\hs,\ss,\acs,\as)}}} \geq \frac{\sqrt{LK}}{300\epsilon^2\sigma^2}, \quad \text{  or } \quad \EE^{\fA}_0\brac{\Nc{\Ecor^{(\hs,\ss,\acs,\as)}}}\geq \frac{1}{300\epsilon^2},
\end{align}
by our choice that $\epsilon\in(0,0.1]$.

Fix a tuple $(\hs,\ss,\acs)$ such that $\hs\in\cH$ and $\hs\leq n+m\floor{H/10m}$, $\ss\in\Sl$, $\acs\in\Ac$. By \eqref{eqn:multi-step-pac-alternative}, we know that for all $\as\in\Apass$, it holds that
\begin{align}\label{eqn:multi-step-pac-alternative-sum}
\begin{aligned}
    &~A^{m-1}\cdot\EE^{\fA}_0\brac{N\paren{\Ereach^{(\hs,\ss,\acs,\as)}}} + \abs{\Apass}\cdot\EE^{\fA}_0\brac{N\paren{\Ecor^{(\hs,\ss,\acs,\as)}}} \\
    \geq&~ \frac1{300}\min\set{\frac{A^{m-1}\sqrt{LK}}{\epsilon^2\sigma^2}, \frac{\abs{\Apass}}{\epsilon^2}}
    \geq \frac1{300}\min\set{\frac{A^{m-1}\sqrt{LK}}{\epsilon^2\sigma^2}, \frac{A^{H/2-1}}{\epsilon^2}}=:\omega,
\end{aligned}
\end{align}
where the last inequality uses the fact that $\abs{\Apass}\geq A^{H/2-1}$ for $\hs\leq n+m\floor{H/10m}$, which follows from a direct calculation (\cref{lemma:password-number}).
Notice that by our definition of $\Ereach$,
\begin{align*}
    \sum_{\as\in\cA^{H-\hs-1}}\EE^{\fA}_0\brac{N\paren{\Ereach^{(\hs,\ss,\acs,\as)}}}
    =&
    \sum_{\as\in\Apass} \EE^{\fA}_0\brac{\Nc{o_{\hs}=\ss, a_{\hs:\hs+m-1}=(\acs,\as_{\hs+1:\hs+m-1})}}\\
    =&
    \sum_{\a\in\cA^{m-1}} \EE^{\fA}_0\brac{\Nc{o_{\hs}=\ss, a_{\hs:\hs+m-1}=(\acs,\a)}}\cdot \sum_{\substack{\as\in\Apass\\
    \as\text{ begins with }\a}} 1\\
    =&
    \sum_{\a\in\cA^{m-1}} \EE^{\fA}_0\brac{\Nc{o_{\hs}=\ss, a_{\hs:\hs+m-1}=(\acs,\a)}}\cdot\frac{\abs{\Apass}}{A^{m-1}}\\
    =&
    \EE^{\fA}_0\brac{\Nc{o_{\hs}=\ss, a_{\hs}=\acs}}\cdot\frac{\abs{\Apass}}{A^{m-1}}.
\end{align*}
Similarly, by our definition of $\Ecor$, we have
\begin{align*}
    \sum_{\as\in\Apass} \EE^{\fA}_0\brac{N\paren{\Ecor^{(\hs,\ss,\acs,\as)}}}
    =&~
    \sum_{\as\in\Apass} \EE^{\fA}_0\brac{N\paren{o_{\hs}=\ss, a_{\hs:H-1}=(\acs,\as)}}\\
    =&~
    \EE^{\fA}_0\brac{N\paren{o_{\hs}=\ss, a_{\hs}=\acs,a_{\hs+1:H-1}\in\Apass}}
    \leq \EE^{\fA}_0\brac{\Nc{o_{\hs}=\ss, a_{\hs}=\acs}}.
\end{align*}
Therefore, taking average of~\cref{eqn:multi-step-pac-alternative-sum} over all $\a\in\Apass$ and using the equations above, we get
\begin{align*}
    \omega
    \leq& 
    \frac{1}{\abs{\Apass}}\sum_{\a\in\Apass} \brac{ A^{m-1}\cdot\EE^{\fA}_0\brac{N\paren{\Ereach^{(\hs,\ss,\acs,\as)}}} + \abs{\Apass}\cdot\EE^{\fA}_0\brac{N\paren{\Ecor^{(\hs,\ss,\acs,\as)}}} }\\
    \leq&
    2~\EE^{\fA}_0\brac{N\paren{o_{\hs}=\ss, a_{\hs}=\acs}}.
\end{align*}

Now, we have shown that $\EE^{\fA}_0\brac{N\paren{o_{\hs}=\ss, a_{\hs}=\acs}}\geq \frac{\omega}{2}$ for each $\ss\in\Sl, \acs\in\Ac$, $\hs\in\cH$ such that $\hs\leq n+m\floor{H/10m}$. Taking summation over all such $(\hs,\ss,\acs)$, we derive that
\begin{align*}
    \frac{\abs{\Sl}\abs{\Ac} \paren{\floor{H/10m}+1}}{600}\min\set{\frac{A^{m-1}\sqrt{LK}}{\epsilon^2\sigma^2}, \frac{A^{H/2-1}}{\epsilon^2}} \leq \sum_{\ss\in\Sl}\sum_{\acs\in\Ac}\sum_{\substack{\hs=n+lm:\\0\leq l\leq \floor{H/10m}}} \EE^{\fA}_0\brac{N\paren{o_{\hs}=\ss, a_{\hs}=\acs}}\leq T,
\end{align*}
where the second inequality is because events $\paren{\set{o_{\hs}=\ss, a_{\hs}=\acs}}_{\hs,\ss,\acs}$ are disjoint. Plugging in $\abs{\Ac}=A-1\geq \frac23A$, $\floor{H/10m}+1\geq H/10m$ completes the proof of~\cref{prop:multi-step-prop}.
\qed

\subsection{Proof of Lemma \ref{lemma:multi-step-rev}}\label{appendix:proof-multi-step-revealing}

The proof is very similar to the proof of \cref{lemma:regret-rev-2-step}, with only slight modification.

\textbf{Case 1:} We first show that the null model $0$ is 1-step 1-revealing. In this model, the states in $\set{\sp^1,\ep^1,\cdots,\sp^L,\ep^L}$ are all not reachable, and hence for each step $h$, we consider the set $\cS'=\Stree\sqcup\bigsqcup_{j=1}^L\set{\sq^j,\eq^j,\termin^j}$. For different states $s,s'\in\cS'$, the support of $\O_h(\cdot|s)$ and $\O_h(\cdot|s')$ are disjoint by our construction, and hence applying \cref{lemma:rev-sep-obs} gives
\begin{align*}
    \min_{\O_{h}^+}\loneone{\O_{h}^+}\leq \max_{s\in\cS'} \gamma\paren{\O_{h}(s)}\leq 1.
\end{align*}
Applying \cref{prop:m-step-mp1-step} completes the proof for null model $0$.

\textbf{Case 2:} We next consider the model $M=M_{\theta,\mu}\in\cM$. By our construction, for $h\leq \hs$, the states in $\set{\sp^1,\ep^1,\cdots,\sp^L,\ep^L}$ are all not reachable, and hence by the same argument as in the null model, we obtain
$$
\min_{\M_{h,m+1}^+}\nrmop{\M_{h,m+1}^+} \leq \min_{\O_{h}^+}\loneone{\O_{h}^+}\leq 1.
$$

Hence, we only need to bound the quantity $\min_{\M_{h,m+1}^+}\nrmop{\M_{h,m+1}^+}$ for a fixed step $h>\hs$. In this case, there exists a $l\in\cH$ such that $h\leq l\leq h+m-1$, and we write $r=l-h+1$. By \cref{lemma:rev-inc-step}, we only need to bound $\min_{\M_{h,r+1}^+}\nrmop{\M_{h,r+1}^+}$. Consider the action sequence $\a=(\as_{h:l-1},\acrev)\in\cA^r$, and we partition $\cS$ as
\begin{align*}
    \cS=\bigsqcup_{s\in\Stree}\set{s} \sqcup \bigsqcup_{j=1}^L \{\sp^j,\sq^j\} \sqcup \{\ep^j,\eq^j\} \sqcup \set{\termin^j}.
\end{align*}
It is direct to verify that, in $M_{\theta,\mu}$, for states $s,s'$ come from different subsets in the above partition, the support of $\M_{h,\a}(\cdot|s)$ and $\M_{h,\a}(\cdot|s')$ are disjoint.
Then, we can apply \cref{lemma:rev-star-to-1} and \cref{lemma:rev-sep-obs}, and obtain
\begin{align*}
    \min_{\M_{h,r+1}^+}\nrmop{\M_{h,r+1}^+}\leq \min_{\M_{h,\a}^+} \loneone{\M_{h,\a}^+}
    \leq \max_j\set{1,  \gamma\paren{\M_{h,\a}(\{\sp^j,\sq^j\})}, \gamma\paren{\M_{h,\a}(\{\ep^j,\eq^j\})} }.
\end{align*}
Therefore, in the following we only need to consider left inverses of the matrix $\M_{h,\a}(\{\sp^j,\sq^j\})$ and $\M_{h,\a}(\{\ep^j,\eq^j\})$ for each $j\in[L]$.

(1) The matrix $\M_{h,\a}(\{\sp^j,\sq^j\})$. By our construction, taking $\a$ at $s_h=\sp^j$ will lead to $o_{h:l-1}=\olock$, $o_l=\olock^j$ and $o_{l+1}\sim \O_{\mu}(\cdot|\ep^j)$;  taking $\a$ at $s_h=\sq^j$ will lead to $o_{h:l-1}=\olock$, $o_l=\olock^j$ and $o_{l+1}\sim \O_{\mu}(\cdot|\eq^j)$. Hence, $\M_{h,\a}(\{\sp^j,\sq^j\})$ can be written as (up to permutation of rows)
\begin{align*}
    \M_{h,\a}(\{\sp^j,\sq^j\})=\begin{bmatrix}
        \frac{\II_{2K}+\sigma \tmu_j}{2K} & \frac{\II_{2K}}{2K}\\
        \mathbf{0} & \mathbf{0}
\end{bmatrix}\in\R^{\cO^{r+1}\times2},
\end{align*}
where $\tmu_j=[\mu_j;-\mu_j]\in\set{-1,1}^{2K}$, $\II=\II_{2K}$ is the vector in $\R^{2K}$ with all entry being one. Similar to \cref{prop:1-step-pac-rev}, we can directly verify that $\gamma\paren{\M_{h,\a}(\set{\sp,\sq})}\leq \frac{2}{\sigma}+1$.

(2) The matrix $\M_{h,\a}(\{\ep^j,\eq^j\})$. By our construction, at $s_h=\ep^j$, we have $o_h\sim \O_{\mu}(\cdot|\ep^j)$ and $o_{h+1:l+1}=\termin^j$; at $s_h=\eq$, we have $o_h\sim \O_{\mu}(\cdot|\eq^j)$ and $o_{h+1:l+1}=\termin^j$. Thus, $\M_{h,\a}(\{\sp^j,\sq^j\})$ can also be written as
\begin{align*}
    \M_{h,\a}(\{\ep^j,\eq^j\})=\begin{bmatrix}
        \frac{\II_{2K}+\sigma \tmu_j}{2K} & \frac{\II_{2K}}{2K}\\
        \mathbf{0} & \mathbf{0}
\end{bmatrix}\in\R^{\cO^{r+1}\times2},
\end{align*}
and hence we also have $\gamma\paren{\M_{h,\a}(\{\ep^j,\eq^j\})}\leq \frac{2}{\sigma}+1$.

Combining the two cases above gives
\begin{align*}
    \min_{\M_{h,m+1}^+}\nrmop{\M_{h,m+1}^+}
    \leq \min_{\M_{h,r+1}^+}\nrmop{\M_{h,r+1}^+}\leq \min_{\M_{h,\a}^+} \loneone{\M_{h,\a}^+}
    \leq \frac{2}{\sigma}+1,
\end{align*}
and hence completes the proof of \cref{lemma:multi-step-rev}.
\qed

\subsection{Proof of Lemma \ref{lemma:multi-step-alternative}}\label{appendix:proof-multi-step-alternative}

Similar to the proof of \cref{lemma:no-regret-alternative}, we only need to show the following lemma, and the proof of \cref{lemma:multi-step-alternative} follows by a reduction argument (see \cref{appendix:proof-no-regret-alternative}).

\begin{lemma}\label{lemma:multi-step-alternative-weak}
    Suppose that algorithm $\fA$ (with possibly random stopping time $\stoptime$) satisfies $N(\Ereach^\theta)\leq \oN_o$ and $N(\Ecor^\theta)\leq\oN_r$ almost surely, for some fixed $\oN_o,\oN_r$. Then 
    \begin{align*}
        \text{either }\oN_o \geq \frac34\frac{\delta^2\sqrt{LK}}{\epsilon^2\sigma^2}, \text{  or }\oN_r\geq \frac34\frac{\delta^2}{\epsilon^2},
    \end{align*}
    where $\delta=\DTV{\PP_{0}^{\fA},\Emu\brac{\PP_{\theta,\mu}^{\fA}}}$.
\end{lemma}
\begin{proof}
Fix a $\theta=(\hs,\ss,\acs,\as)$. Recall that we define $\Ereach^\theta\defeq \set{o_{\hs}=\ss, a_{\hs:\hs+m-1}=(\acs,\as_{\hs+1:\hs+m-1})}$, and we further define
\begin{align*}
    E_{\rev}^\theta&\defeq \set{o_{\hs}=\ss, a_{\hs:h}=(\acs,\as_{\hs+1:h-1},\acrev) \text{ for some }h\in\cH_{>\hs}}.
\end{align*}
For any model $M\in\cM_\theta\defeq\big\{M_{\theta,\mu}:\mu\in\set{-1,+1}^{L\times K}\big\}\cup\set{0}$, we consider the following ``augmented'' system dynamics $\oP_M$: \\
1. For each episode, after the interaction $\tau_H\sim \P_M$ is finished, the environment generated an extra observation $o_{H+1}=\oaug$.\\
2. If $\tau_H\not\in \Ereach^\theta$ or $\tau_H\in\Erev^\theta$, then $\oaug=\odum$.\\
3. If $\tau_H\in \Ereach^\theta-\Erev^\theta$, then in $\tau_H=(o_1,a_1,\cdots,o_H,a_H)$ we have $o_{\hs+m}=\olock^j$ for some $j\in[L]$, and then the environment generates $\oaug$ as
\begin{align*}
    M=M_{\theta,\mu}:\qquad 
    \oP_{\theta,\mu}(\oaug=o_i^+|\tau_H)=\frac{1+\epsilon\sigma\mu_{j,i}}{2K}, \qquad 
    \oP_{\theta,\mu}(\oaug=o_i^+|\tau_H)=\frac{1-\epsilon\sigma\mu_{j,i}}{2K}, \qquad
    \forall i\in[K],
\end{align*}
and for $M=0$, $\oP_{0}(\oaug=\cdot|\tau_H)=\Unif(\set{o_1^+,o_1^-,\cdots,o_K^+,o_K^-})$. %

Clearly, for each $M\in\cM_\theta$, $\oP_M$ is still a sequential decision process. %
Under such construction, each policy $\pi$ induces a distribution of $\otau=(\tau_H,\oaug)\sim \oP^{\pi}_M$, and the algorithm $\fA$ induce a distribution of  $\otau^{(1)},\cdots,\otau^{(\stoptime)}\sim\oP_{0}^{\fA}$.
By data-processing inequality, we have
\begin{align*}
    \DTV{\PP_{0}^{\fA},\Emu\brac{\PP_{\theta,\mu}^{\fA}}} \leq \DTV{\oP_{0}^{\fA}, \Emu\brac{\oP_{\theta,\mu}^{\fA}}}.
\end{align*}
Hence, by \cref{lemma:ingster}, we only need to bound
\begin{align*}
    1+\chis{\Emu\brac{\oP_{\theta,\mu}^{\fA}}}{\oP_{0}^{\fA}}
    =
    \Emumu\EE_{\otau^{(1)},\cdots,\otau^{(\stoptime)}\sim\oP_{0}^{\fA}}\brac{
        \prod_{t=1}^\stoptime\frac{\oP_{\theta,\mu}(\otaut)\oP_{\theta,\mu'}(\otaut)}{\oP_{0}(\otaut)^2}
    }.
\end{align*}
To upper bound the above quantity, we invoke the following lemma (proof in \cref{appendix:proof-multi-step-MGF}).
\begin{lemma}[Bound on the $\chi^2$-inner product]
\label{lemma:multi-step-MGF}
Under the conditions of \cref{lemma:multi-step-alternative-weak} (for a fixed $\theta$), it holds that for any $\mu,\mu'\in\set{-1,1}^K$,
\begin{align}\label{eqn:multi-step-MGF}
    \oE_0^\fA\brac{
        \prod_{t=1}^\stoptime\frac{\oP_{\theta,\mu}(\otaut)\oP_{\theta,\mu'}(\otaut)}{\oP_{0}(\otaut)^2}
    }
    \leq \exp\paren{ \oN_o\cdot \frac{\sigma^2\epsilon^2}{LK}\abs{\iprod{\mu}{\mu'}}+\frac43\epsilon^2\oN_r }.
\end{align}
\end{lemma}
Given \cref{lemma:multi-step-MGF}, the desired result follows from a standard argument (see e.g. the proof of \cref{lemma:no-regret-alternative-weak}).
\end{proof}

\subsection{Proof of Lemma~\ref{lemma:multi-step-MGF}}\label{appendix:proof-multi-step-MGF}

We first show the following lemma, which is a single-episode version of \cref{lemma:multi-step-MGF}.
\begin{lemma}\label{lemma:multi-step-single-episode}
For any policy $\pi$ and parameter $\theta,\mu,\mu'$, it holds that
\begin{align}\label{eqn:multi-step-ingster-single-step}
    \EE_{\otau\sim\oP_0^\pi}\brac{\frac{\oP_{\theta,\mu}(\otau)\oP_{\theta,\mu'}(\otau)}{\oP_{0}(\otau)^2} \exp\paren{-\II(\otau\in\Ereach^\theta)\cdot \frac{\epsilon^2\sigma^2}{LK}\iprod{\mu}{\mu'}-\II(\otau\in\Ecor^\theta)\cdot\frac43 \epsilon^2}}\leq 1.
\end{align}
\end{lemma}

\begin{proof}[Proof of \cref{lemma:multi-step-single-episode}]
In the following, all expectation and conditional expectation is taken with respect to $\otau=(\tau_H,\oaug)\sim\oP_0^\pi$.

Similar to the proof of \cref{lemma:no-regret-MGF} (in \cref{appendix:proof-no-regret-MGF}), the core of our analysis is still computing the quantity $I(\tau_l)$, defined as
\begin{align}
    I(\tau_l)
    \defeq&~ \EE_{o_{l+1}\sim \oP_0(\tau_l)} \brac{\frac{\oP_{\theta,\mu}(o_{l+1}|\tau_l)\oP_{\theta,\mu'}(o_{l+1}|\tau_l)}{\oP_{0}(o_{l+1}|\tau_l)^2}}
    =
    \left\{
    \begin{aligned}
        &\cond{\frac{\P_{\theta,\mu}(o_{l+1}|\tau_l)\P_{\theta,\mu'}(o_{l+1}|\tau_l)}{\P_{0}(o_{l+1}|\tau_l)^2}}{\tau_l}, && l<H, \\
        \\
         &\EE_{\oaug\sim \P_0(\tau_H)} \brac{\frac{\oP_{\theta,\mu}(\oaug|\tau_H)\oP_{\theta,\mu'}(\oaug|\tau_H)}{\oP_{0}(\oaug|\tau_H)^2}},&& l=H.
    \end{aligned}\right.
\end{align}
Basically, by \cref{lemma:martingale-eqn}, we have
\begin{align}\label{eqn:proof-multi-step-trunc}
    1=\EE_{\otau\sim \oP^{\pi}_0}\cdbrac{\frac{\oP_{\theta,\mu}(\otau|\tau_{\hs+m})\oP_{\theta,\mu'}(\otau|\tau_{\hs+m})}{\oP_{0}(\otau|\tau_{\hs+m})^2}\cdot\exp\paren{ - \sum_{l=\hs+m}^H \log I(\tau_l) }}{\tau_{\hs+m}}.
\end{align}

In the following, we first compute $I(\tau_l)$ for each (reachable) $\tau_l$.

An important observation is that, for a trajectory $\tau_l=(o_1,a_1,\cdots,o_l,a_l)$ with $l<H$, if $\PP_{\theta,\mu}(o_{l+1}=\cdot|\tau_l)\neq \PP_{0}(o_{l+1}=\cdot|\tau_l)$, then \\
1. Clearly, $o_{\hs}=\ss$, $a_{\hs}=\acs$ (i.e. $l\geq \hs+1$ and taking action $a_{1:\hs-1}$ from $s_0$ will result in $\ss$ at step $\hs$). \\
2. Either $a_{\hs+1:l}=(\as_{\hs+1:l-1},\acrev)$ for some $l\in\cH_{>\hs}$, or $l=H-1$ and $a_{\hs+1:H-1}=\as$.

Therefore, for each $l\in\cH_{>\hs}$ we define
$$
E_{\rev,l}^\theta\defeq \set{
o_{\hs}=\ss, a_{\hs:l}=(\acs,\as_{\hs+1:l-1},\acrev)
}.
$$
Then, if $\PP_{\theta,\mu}(\cdot|\tau_l)\neq \PP_{0}(\cdot|\tau_l)$, either (case 1) $l\in\cH_{>\hs}, \tau_{H-1}\in E_{\rev,l}^\theta$, or (case 2) $l=H-1, \tau_l\in\Ecor^\theta$, or (case 3) $l=H, \tau_H\in \Ereach^\theta-\Erev^\theta$.

In the following, we compute $I(\tau_l)$ for these three cases separately. We consider the events $\cL_j\defeq \set{ o_{\hs+m}=\olock^j } (j\in[L])$ to simplify our discussion.

\textbf{Case 1:} $l\in\cH_{>\hs}, \tau_l\in E_{\rev,l}^\theta$. In this case, there exists a $j\in[L]$ such that $o_{\hs+m}=\olock^j$, i.e. $\tau_l\in\cL_j$. In other words, observing $\tau_l$ implies that $s_{h'}\in\{\sp^j,\sq^j\}$ for $h<h'\leq l$, and $s_{l+1}\in\{\ep^j,\eq^j\}$ (because $a_l=\acrev$). Therefore,
\begin{align*}
    \P_{\theta,\mu}(o_{l+1}=o|\tau_l)
    =&~
    \P_{\theta,\mu}(o_{l+1}=o|s_{l+1}=\ep^j)\P_{\theta,\mu}(s_{l+1}=\ep^j|\tau_l)
    +
    \P_{\theta,\mu}(o_{l+1}=o|s_{l+1}=\eq^j)\P_{\theta,\mu}(s_{l+1}=\eq^j|\tau_l)\\
    =&~
    \paren{\O_\mu(o|\ep^j)-\O(o|\eq^j)}\cdot \P_{\theta,\mu}(s_{l+1}=\ep^j|\tau_l)+\O(o|\eq^j),
\end{align*}
Notice that by our construction,
\begin{align*}
    \P_{\theta,\mu}(s_{l+1}=\ep^j|\tau_l)
    =&~
    \P_{\theta,\mu}(s_{l}=\sp^j|\tau_{l-1},o_l)\\
    =&~
    \P_{\theta,\mu}(s_{l}=\sp^j|o_{\hs}=\ss,a_{\hs:l-1}=(\acs,\as_{\hs+1:l-1}),\cL_j)\\
    =&~
    \P_{\theta,\mu}(s_{\hs+1}=\sp^j|o_{\hs}=\ss,a_{\hs:l-1}=(\acs,\as_{\hs+1:l-1}),\cL_j)\\
    =&~
    \P_{\theta,\mu}(s_{\hs+1}=\sp^j|o_{\hs}=\ss,a_{\hs:l-1}=\acs,s_{\hs+1}\in\{\sp^j,\sq^j\})\\
    =&~\epsilon,
\end{align*}
where the first equality is because $s_{l+1}=\ep^j$ if and only if $s_l=\sp^j, a_l=\acrev$, the second inequality is because there are only $\olock$ and $\olock^j$ in $o_{\hs+1:l}$ are the third equality is because $s_l=\sp^j$ if and only if $s_{\hs+1}=\sp^j, a_{\hs+1:l-1}=\as_{\hs+1:l-1}$. Combining the above equations with our definition of $\O_\mu(\cdot|\ep^j)$ and $\O(\cdot|\eq^j)$ gives
\begin{align*}
    \P_{\theta,\mu}(o_{l+1}=o_i^+|\tau_l)=\frac{1+\epsilon\sigma\mu_{j,i}}{2K}, \qquad 
    \P_{\theta,\mu}(o_{l+1}=o_i^-|\tau_l)=\frac{1-\epsilon\sigma\mu_{j,i}}{2K}, \qquad \forall i\in[K].
\end{align*}
On the other hand, clearly $\PP_{0}(o_{l+1}=\cdot|\tau_l)=\Unif(\{o_1^+,o_1^-,\cdots,o_K^+,o_K^-\})$.
Hence, it holds that
\begin{align}\label{eqn:proof-multi-step-MGF-case1}
    I(\tau_l)
    =
    \frac{1}{2K}\sum_{o\in\cO_o} \frac{\PP_{\theta,\mu}(o_{l+1}=o|\tau_l)\PP_{\theta,\mu'}(o_{l+1}=o|\tau_l)}{\PP_{0}(o_{l+1}=o|\tau_l)^2}
    =1+\frac{\epsilon^2\sigma^2}{K}\sum_{i=1}^K \mu_{j,i}\mu_{j,i}', \quad\text{for any } \tau_l\in E_{\rev,l}^\theta \cap \cL_j.
\end{align}

\textbf{Case 2:} $l=H-1, \tau_{H-1}\in\Ecor^\theta$. In this case, by a calculation exactly the same as the proof of \cref{lemma:no-regret-MGF} (\cref{appendix:proof-no-regret-MGF}, case 2), we can obtain
\begin{align}\label{eqn:proof-multi-step-MGF-case2}
    I(\tau_{H-1})
    =1+\frac43 \epsilon^2, \quad\text{for any } \tau_{H-1}\in \Ecor^\theta.
\end{align}

\textbf{Case 3:} $l=H, \tau_H\in \Ereach^\theta-\Erev^\theta$. Suppose that for some $j\in[L]$, $\tau_H\in\paren{\Ereach^\theta-\Erev^\theta}\cap \cL_j$, then by our construction of $\oP$, we have
\begin{align}\label{eqn:proof-multi-step-MGF-case3}
    I(\tau_H)
    =
    1+\frac{\epsilon^2\sigma^2}{K} \sum_{i=1}^K \mu_{j,i}\mu_{j,i}', \quad\text{for any }\tau_H\in\paren{\Ereach^\theta-\Erev^\theta}\cap \cL_j.
\end{align}

Combining \eqref{eqn:proof-multi-step-MGF-case1} \eqref{eqn:proof-multi-step-MGF-case2} \eqref{eqn:proof-multi-step-MGF-case3} together, we have shown that for any $\tau_H$ that begins with $\tau_{\hs+m}\in\Ereach^\theta \cap \cL_j$,
\begin{align*}
    \sum_{l=\hs+m}^H \log I(\tau_l)
    =&~\sum_{l\in\cH_{>\hs}} \IIc{\tau_l\in E_{\rev,l}^\theta} \cdot \log\paren{1+\frac{\epsilon^2\sigma^2}{K}\sum_{i=1}^K \mu_{j,i}\mu_{j,i}'}
    + \IIc{\tau_{H-1}\in\Ecor^\theta}\cdot\log\paren{1+\frac43 \epsilon^2} \\
    &~+ \IIc{\tau_H \in \Ereach^\theta-\Erev^\theta} \cdot \log\paren{1+\frac{\epsilon^2\sigma^2}{K}\sum_{i=1}^K \mu_{j,i}\mu_{j,i}'}\\
    =&~ 
    \IIc{\tau_H\in E_{\rev}^\theta} \cdot \log\paren{1+\frac{\epsilon^2\sigma^2}{K}\sum_{i=1}^K \mu_{j,i}\mu_{j,i}'}
    + \IIc{\tau_{H}\in\Ecor^\theta}\cdot\log\paren{1+\frac43 \epsilon^2} \\
    &~+ \IIc{\tau_H \in \Ereach^\theta-\Erev^\theta} \cdot \log\paren{1+\frac{\epsilon^2\sigma^2}{K}\sum_{i=1}^K \mu_{j,i}\mu_{j,i}'}\\
    =&~ 
    \IIc{\tau_H\in \Ereach^\theta} \cdot \log\paren{1+\frac{\epsilon^2\sigma^2}{K}\sum_{i=1}^K \mu_{j,i}\mu_{j,i}'}
    + \IIc{\tau_{H}\in\Ecor^\theta}\cdot\log\paren{1+\frac43 \epsilon^2},
\end{align*}
where the second equality is because $\Erev^\theta=\bigsqcup_l E_{\rev,l}^\theta$. We also have $\sum_{l=\hs+m}^H \log I(\tau_l)=0$ for $\tau_{\hs+m}\not\in\Ereach^\theta$. Plugging the value of
$\sum_{l=\hs+m}^H \log I(\tau_l)$ into \eqref{eqn:proof-multi-step-trunc} and using the fact that $\PP_{\theta,\mu}(\tau_{\hs+m})=\PP_{0}(\tau_{\hs+m})$ by our construction, we have
\begin{align*}
    \EE_{\otau\sim \oP^{\pi}_0}\cdbrac{\frac{\oP_{\theta,\mu}(\otau)\oP_{\theta,\mu'}(\otau)}{\oP_{0}(\otau)^2}\cdot \paren{ 1+\II(\tau_H\in\Ecor^\theta)\cdot\frac43 \epsilon^2 }^{-1} }{\tau_{\hs+m}}
    =
    \begin{cases}
    1+\frac{\epsilon^2\sigma^2}{K} \sum_{i=1}^K \mu_{j,i}\mu_{j,i}',  & \text{for }\tau_{\hs+m}\in\Ereach^\theta \cap \cL_j, \\
    1, & \text{if }\tau_{\hs+m}\not\in\Ereach^\theta.
    \end{cases}
\end{align*}
Notice that in $\PP_0$, conditional on $\tau_{\hs+m-1}\in\Ereach^\theta$, $o_{\hs+m}$ is uniformly distributed over $\set{\olock^1,\cdots,\olock^L}$, and hence
\begin{align*}
    &~\cond{\frac{\oP_{\theta,\mu}(\otau)\oP_{\theta,\mu'}(\otau)}{\oP_{0}(\otau)^2}\cdot \paren{ 1+\II(\tau_H\in\Ecor^\theta)\cdot\frac43 \epsilon^2 }^{-1} }{\tau_{\hs+m-1}}\\
    =&~
    \cond{\cond{\frac{\oP_{\theta,\mu}(\otau)\oP_{\theta,\mu'}(\otau)}{\oP_{0}(\otau)^2}\cdot \paren{ 1+\II(\tau_H\in\Ecor^\theta)\cdot\frac43 \epsilon^2 }^{-1} }{\tau_{\hs+m}}}{\tau_{\hs+m-1}}\\
    =&~
    \frac1L\sum_{j=1}^L \paren{ 1+\frac{\epsilon^2\sigma^2}{K} \sum_{i=1}^K \mu_{j,i}\mu_{j,i}' }
    =1+\frac{\epsilon^2\sigma^2}{LK}\iprod{\mu}{\mu'}.
\end{align*}
On the other hand, for $\tau_{\hs+m-1}\not\in\Ereach^\theta$,
\begin{align*}
    \EE_{\otau\sim \oP^{\pi}_0}\cdbrac{\frac{\oP_{\theta,\mu}(\otau)\oP_{\theta,\mu'}(\otau)}{\oP_{0}(\otau)^2}\cdot \paren{ 1+\II(\otau\in\Ecor^\theta)\cdot\frac43 \epsilon^2 }^{-1} }{\tau_{\hs+m-1}}=1.
\end{align*}
Hence, taking expectation over $\tau_{\hs+m-1}$ gives
\begin{align*}
    \EE_{\otau\sim \oP^{\pi}_0}\brac{\frac{\oP_{\theta,\mu}(\otau)\oP_{\theta,\mu'}(\otau)}{\oP_{0}(\otau)^2}\cdot \paren{ 1+\II(\otau\in\Ecor^\theta)\cdot\frac43 \epsilon^2 }^{-1}\cdot \paren{ 1+\II(\otau\in\Ereach^\theta)\cdot\frac{\epsilon^2\sigma^2}{LK}\iprod{\mu}{\mu'} }^{-1} }=1.
\end{align*}
Using the fact $(1+x)^{-1}\geq \exp(-x)$ completes the proof.
\end{proof}

With \cref{lemma:multi-step-single-episode} proven, we continue to prove \cref{lemma:multi-step-MGF}. Applying \cref{lemma:multi-step-single-episode} to algorithm $\fA$, we obtain that for each $t\in[\stoptime]$,
\begin{align*}
    \EE_{\otau^{(1)},\cdots,\otau^{(\stoptime)}\sim\oP_{0}^{\fA}}\brcond{
        \frac{\oP_{\theta,\mu}(\otaut)\oP_{\theta,\mu'}(\otaut)}{\oP_{0}(\otaut)^2} 
        \cdot
        \exp\paren{
            -\II(\otaut\in\Ereach^\theta)\cdot\frac{\epsilon^2\sigma^2}{LK}\iprod{\mu}{\mu'}
            -\II(\otaut\in\Ecor^\theta)\cdot\frac43 \epsilon^2
        }
    }{\otau^{(1:t-1)}}\leq 1.
\end{align*}
Therefore, by the martingale property, it holds that
\begin{align*}
    1\geq&~ \EE_{\otau^{(1)},\cdots,\otau^{(\stoptime)}\sim\oP_{0}^{\fA}}\brac{
        \prod_{t=1}^\stoptime \frac{\oP_{\theta,\mu}(\otaut)\oP_{\theta,\mu'}(\otaut)}{\oP_{0}(\otaut)^2} 
        \cdot
        \exp\paren{
            -\II(\otaut\in\Ereach^\theta)\cdot\frac{\epsilon^2\sigma^2}{LK}\iprod{\mu}{\mu'}
            -\II(\otaut\in\Ecor^\theta)\cdot\frac43 \epsilon^2
        }
    }\\
    =&~
    \EE_{\otau^{(1)},\cdots,\otau^{(\stoptime)}\sim\oP_{0}^{\fA}}\brac{
        \prod_{t=1}^\stoptime \frac{\oP_{\theta,\mu}(\otaut)\oP_{\theta,\mu'}(\otaut)}{\oP_{0}(\otaut)^2} 
        \times
        \exp\paren{
            -\sum_{t=1}^T\II(\otaut\in\Ereach^\theta)\cdot\frac{\epsilon^2\sigma^2}{LK}\iprod{\mu}{\mu'}
            -\sum_{t=1}^T\II(\otaut\in\Ecor^\theta)\cdot\frac43 \epsilon^2
        }
    }\\
    =&~
    \EE_{\otau^{(1)},\cdots,\otau^{(\stoptime)}\sim\oP_{0}^{\fA}}\brac{
        \prod_{t=1}^\stoptime \frac{\oP_{\theta,\mu}(\otaut)\oP_{\theta,\mu'}(\otaut)}{\oP_{0}(\otaut)^2} 
        \times
        \exp\paren{
            -N(\Ereach^\theta)\cdot\frac{\epsilon^2\sigma^2}{LK}\iprod{\mu}{\mu'}
            -N(\Ecor^\theta)\cdot\frac43 \epsilon^2
        }
    }\\
    \geq&~
    \EE_{\otau^{(1)},\cdots,\otau^{(\stoptime)}\sim\oP_{0}^{\fA}}\brac{
        \prod_{t=1}^\stoptime \frac{\oP_{\theta,\mu}(\otaut)\oP_{\theta,\mu'}(\otaut)}{\oP_{0}(\otaut)^2} 
        \times
        \exp\paren{
            -\oN_o\cdot\frac{\epsilon^2\sigma^2}{LK}\abs{\iprod{\mu}{\mu'}}
            -\oN_r\cdot\frac43 \epsilon^2
        }
    }.
\end{align*}
Multiplying both sides by $\exp\paren{ \oN_o\cdot \frac{\sigma^2\epsilon^2}{LK}\abs{\iprod{\mu}{\mu'}}+\frac43\epsilon^2\oN_r }$ completes the proof of \cref{lemma:multi-step-MGF}.
\qed

\section{Regret for single-step revealing POMDPs}
\label{appdx:1-step-regret}

In this section, we establish~\cref{thm:regret-upper} on a broader class of sequential decision problems termed as \emph{strongly B-stable PSRs}, and then deduce the guarantee for single-step revealing POMDPs as a special case. The proof is largely parallel to the analysis of PAC learning for B-stable PSRs~\citep{chen2022partially}, and we follow the notations there: in the following we use $\theta$ to refer to the PSR model, and $\Theta$ to refer to the class of PSR models.

\subsection{Strongly B-stable PSRs}
\label{appdx:strong-b-stable}

\newcommand{\cesthth}{\cE_{\theta,h}^{\otheta}(\pi,\tau_{h-1})}
\newcommand{\cesththz}{\cE_{\theta,0}^{\otheta}(\pi)}
\newcommand{\omh}{\omega_{h}}
\newcommand{\omhp}{\omega_{h+1}}
\newcommand{\Omh}{\Omega_{h}}
\newcommand{\Omhp}{\Omega_{h+1}}
\newcommand{\Pihpi}{\Pi_h(\pi)}

We recall the definition of PSRs and B-stability in~\cref{appdx:psr}. To establish $\sqrt{T}$-regret upper bound for learning PSRs, we introduce the following structural condition.

\begin{definition}[Strong B-stability]
\label{def:B-stable-strong}
A PSR is \emph{strongly B-stable with parameter $\stab\ge 1$} (henceforth also \emph{$\stab$-strongly-stable}) if it admits a {\Bpara} such that for all step $h\in[H]$, policy $\pi$, $x\in\R^{\Uh}$, 
\begin{equation}\label{eqn:strong-B-stable}
\sum_{\tau_{h:H}} \pi(\tau_{h:H})\times \abs{\B_{H:h}(\tau_{h:H}) x} \le \stab\sum_{t_h\in\Uh}\pi(t_h)\times \abs{x(t_h)}.
\end{equation}
\end{definition}

For notational simplicity, from now on we assume that for each step $h$, $\Uh=(\cO\times\cA)^{m_h-1}\times\cO$ for some $m_h\in\Z_{\ge 1}$, and we define $\Omh\defeq (\cO\times\cA)^{m_h-1}$; our results also hold for any general $\Uh$ using slightly more involved notation.

\begin{proposition}[Error decomposition for strongly B-stable PSRs]
\label{prop:psr-err-decomp}
Suppose that two PSR models $\theta,\otheta$ admit $\{ \{ \BB_h^\theta(o_h, a_h) \}_{h, o_h, a_h} , \bq_0^\theta \}$ and $\{ \{ \BB_h^{\otheta}(o_h, a_h) \}_{h, o_h, a_h} , \bq_0^{\otheta} \}$ as {\Bpara} respectively. Define
\begin{align*}
    \cesthth\defeq&
    \frac12 \max_{\pi'\in\Pihpi}\sum_{\tau_{h:H}}\pi'(\tau_{h:H}|\tauhm)\times \abs{\B_{H:h+1}^\theta(\tau_{h+1:H}) \left(\B^\theta_h(o_h,a_h) - \B^{\otheta}_h(o_h,a_h)\right)\bd^{\otheta}(\tau_{h-1})},\\
    \cesththz\defeq&
    \frac12 \max_{\pi'\in\Pi_0(\pi)}\sum_{\tau_{1:H}}\pi'(\tau_{1:H})\times \abs{\B_{H:1}^\theta(\tau_{1:H}) \paren{\bd_0^{\theta}-\bd_0^{\otheta}}},
\end{align*}
where we define
\begin{align*}
    \Pihpi\defeq\set{\pi': \pi'|_{\cO\times\cA\times\Omhp}=\pi|_{\cO\times\cA\times\Omhp}}, \qquad
    \Pi_0(\pi)\defeq\set{\pi': \pi'|_{\Omega_1}=\pi|_{\Omega_1}},
\end{align*}
i.e. $\Pihpi$ is the set of all policy $\pi'$ such that for all $(o_h,a_h,\omhp)\in \cO\times\cA\times\Omhp$, $\pi'(o_h,a_h,\omhp|\tauhm)=\pi(o_h,a_h,\omhp|\tauhm)$.

Then the following claims hold.

1. (Performance decomposition) It holds that
\begin{align*}
    \DTV{ \PP_{\theta}^{\pi}, \PP_{\otheta}^{\pi} } \leq \cesththz + \sum_{h=1}^H \EE_{\otheta}^\pi\brac{\cesthth},
\end{align*}
where for $h\in[H]$, the expectation $\EE_{\otheta}^\pi$ is taking over $\tau_{h-1}$ under model $\otheta$ and policy $\pi$.

2. (Bounding errors by Hellinger distance) Suppose that $\theta$ is $\stab$-strong-stable and $\{ \{ \BB_h^\theta(o_h, a_h) \}_{h, o_h, a_h} , \bq_0^\theta \}$ satisfies the stability condition \eqref{eqn:strong-B-stable}. For any step $h$, policy $\pi$, it holds that
\begin{align*}
    \EE_{\otheta}^\pi\brac{\cesthth^2}\leq 2\stab^2 \DH{\PP_\theta^\pi,\PP_\otheta^\pi}.
\end{align*}
and $(\cesththz)^2\leq \stab^2\DH{ \PP_\theta^\pi, \PP_\otheta^\pi}$.
\end{proposition}

\subsection{Algorithms and guarantees}

In this section, we state the $\sqrt{T}$-regret guarantee of the algorithm \omle~(\cref{alg:OMLE}, \citep{liu2022partially, chen2022partially}). Its proof is in presented \cref{appdx:proof-OMLE}, which is adapted from the analysis of the (explorative) \omle~algorithm in \citet{chen2022partially}. We also remark that the regret upper bound of \omle~in \cref{thm:regret-upper} can also be shown directly for single-step revealing POMDP, by strengthening the analysis in \citet{liu2022partially} using the ideas of \citet{chen2022partially}.

\begin{algorithm}[t]
	\caption{\textsc{Optimistic Maximum Likelihood Estimation (OMLE) \citep{liu2022partially, chen2022partially}}} \begin{algorithmic}[1]\label{alg:OMLE}
	\STATE \textbf{Input:} Model class $\Theta$, parameter $\beta>0$.
   \STATE \textbf{Initialize:} $\Theta^1=\Theta$, $\cD=\{\}$. 
	  \FOR{iteration $k=1,\ldots,T$}
	  \STATE Set $(\theta^k,\pi^k) = \argmax_{\theta\in\Theta^k,\pi} V_\theta(\pi)$. 
	  \label{line:greedy}
	\STATE Execute $\pi^k$ to collect a trajectory $\tau^{k}$, and add  $(\pi^{k},\tau^{k})$ into $\cD$. 
	\label{line:execute-policy}
	  \STATE Update confidence set
	  \begin{equation*}
	  \textstyle
	  \Theta^{k+1} = \bigg\{\hat\theta \in \Theta: \sum_{(\pi,\tau)\in\cD} \log \P_{{\hat\theta}}^{\pi} (\tau)
	  \ge \max_{ \theta \in\Theta} \sum_{(\pi,\tau)\in\cD} \log \P^{\pi}_{{\theta}}(\tau) -\beta \bigg\}. 
	  \end{equation*} 
	  \label{line:conf-set}
	  \ENDFOR
   \end{algorithmic}
\end{algorithm}

\begin{theorem}
\label{thm:OMLE-regret}
Suppose every $\theta\in\Theta$ is $\stab$-strongly stable (\cref{def:B-stable-strong}), and the true model $\theta^\star\in\Theta$ with rank $\dPSR\le d$. Then, choosing $\beta=C\log(\Nt(1/T)) /\delta)$ for some absolute constant $C>0$, with probability at least $1-\delta$,~\cref{alg:OMLE} achieves
\begin{equation}
     \sum_{t=1}^T V^\star-V_\ths(\pi^t)\leq \bigO{\sqrt{\stab^2 OAU_{\cT}dH^2 \clog \beta T}}
\end{equation}
where $U_{\cT}\defeq\max_h\abs{\Omh}$, $\clog\defeq \log\paren{1+ T d OAU_{\cT} \stab \rb}$ with $\rb\defeq 1+\max_{h,o,a} \lone{\B_h(o,a)}$.
\end{theorem}

Using analysis entirely parallel to \citet[Appendix G]{chen2022partially}, we can show that \etod~\citep{chen2022unified} and \mops~\citep[Algorithm 4]{chen2022partially} both achieve the same regret guarantees as \cref{thm:regret-upper}.

\begin{theorem}\label{thm:e2d-psr}
Suppose $\Theta$ is a PSR class with the same core test sets $\{ \Uh\}_{h \in [H]}$, and each $\theta\in\Theta$ admits a {\Bpara} that is $\stab$-strongly-stable~(cf. \cref{def:B-stable-strong}), and has PSR rank $\dPSR \le d$. Then for the coefficients $\dec$ and $\psc$ introduced in \citet{chen2022unified}, it holds that
\begin{align*}
    \odec_{\gamma}(\Theta)\leq \bigO{ \frac{\stab^2 OAU_{\cT}dH^2}{\gamma} }, \qquad \psc_{\gamma}(\Theta)\leq \bigO{ \frac{\stab^2 OAU_{\cT}dH^2}{\gamma} }.
\end{align*}
Therefore, we can apply \citet[Theorem D.1]{chen2022unified} (for \mops) and \citet[Theorem C.7]{chen2022unified} (for \etod) to show that, with suitably chosen parameters, \mops~and \etod~both achieve a regret of
\begin{equation}
     \Regret\leq \bigO{\sqrt{\stab^2 OAU_{\cT}dH^2 \log(\Nt(1/T)) /\delta) T}},
\end{equation}
with probability at least $1-\delta$.
\end{theorem}

\begin{proof}[Proof of \cref{thm:regret-upper}]
To apply \cref{thm:OMLE-regret}, we first notice that \cref{prop:rev-to-psr} readily implies that any single-step $\arev$-revealing is strongly B-stable PSR, with $\stab\leq\arev^{-1}$ and core test sets $\Uh=\cO$ for all $h$. Therefore, applying \cref{thm:OMLE-regret} shows that with a model class $\cM$ of single-step $\arev$-revealing POMDPs, \omle~achieves a regret of 
\begin{align*}
    \Regret\leq \tbO{\sqrt{\arev^{-2} SOAH^2 \log\cN_{\cM}(1/T)\cdot T}},
\end{align*}
as $U_{\cT}=1, d\leq S, \stab\leq \arev^{-1}, \rb\leq \arev^{-1}$ and $\clog=\tbO{1}$. Similarly, \etod{} and \mops{} also achieve the same regret upper bound. Noticing that $ \log\cN_{\cM}(1/T)=\tbO{H(S^2A+SO)}$ \citep{chen2022partially} completes the proof. 
\end{proof}

\subsection{Proof of Proposition~\ref{prop:psr-err-decomp}}
Claim 1 follows from the proof of \citet[Lemma D.1]{chen2022partially} directly. In the following, we show claim 2.

Fix a step $h\in[H]$. An important observation is that, by the strong $\stab$-stability of $\theta$ (\cref{def:B-stable-strong}), for any $\pi'\in\Pihpi$, we have $\forall x\in\R^{\Uh}$
\begin{align}\label{eqn:strong-B-stability-imply-1}
\begin{aligned}
    &\sum_{\tau_{h:H}}\pi'(\tau_{h:H}|\tauhm)\times \abs{\BB_{H:h}^\theta(\tau_{h:H})x }\leq \stab\sum_{t_h\in\Uh}\pi'(t_h|\tauhm)\times\abs{x(t_h)}=\stab\sum_{t_h\in\Uh}\pi(t_h|\tauhm)\times\abs{x(t_h)},
\end{aligned}
\end{align}
and similarly, for $\forall x\in\R^{\Uhp}$,
\begin{align}\label{eqn:strong-B-stability-imply-2}
\begin{aligned}
    &\sum_{\tau_{h+1:H}}\pi'(\tau_{h+1:H}|\tau_h)\times \abs{\BB_{H:h+1}^\theta(\tau_{h+1:H})x }\leq \stab \sum_{t_{h+1}\in\Uhp}\pi(t_{h+1}|\tau_h)\times\abs{x(t_{h+1})}.
\end{aligned}
\end{align}
Therefore, using use the following formula:
\begin{align*}
    \left(\B^\theta_h(o_h,a_h) - \B^{\otheta}_h(o_h,a_h)\right)\bd^{\otheta}(\tau_{h-1})
    =
    \B^\theta_h(o_h,a_h)\paren{\bd^{\otheta}(\tau_{h-1})-\bd^{\theta}(\tau_{h-1})}
    +
    \paren{\B^\theta_h(o_h,a_h)\bd^{\theta}(\tau_{h-1}) - \B^{\otheta}_h(o_h,a_h)\bd^{\otheta}(\tau_{h-1})},
\end{align*}
we have
\begin{align*}
2\cesthth
=&
\max_{\pi'\in\Pihpi} \sum_{\tau_{h:H}}\pi'(\tau_{h:H}|\tauhm)\times \abs{\BB_{H:h+1}^\theta(\tau_{h+1:H}) \left(\B^\theta_h(o_h,a_h) - \B^{\otheta}_h(o_h,a_h)\right)\bd^{\otheta}(\tau_{h-1})} \\
\le&
\max_{\pi'\in\Pihpi} \sum_{\tau_{h:H}}\pi'(\tau_{h:H}|\tauhm)\times \abs{\BB_{H:h}^\theta(\tau_{h:H}) \left(\bd^\theta(\tau_{h-1})-\bd^{\otheta}(\tau_{h-1})\right) }  \\
 &+ 
\max_{\pi'\in\Pihpi} \sum_{\tau_{h:H}}\pi'(\tau_{h:H}|\tauhm)\times \abs{\BB_{H:h+1}^\theta(\tau_{h+1:H})\left( \B^\theta_h(o_h,a_h)\bd^\theta(\tau_{h-1}) - \BB_h^{\otheta}(o_h,a_h)\bd^{\otheta}(\tau_{h-1}) \right)} \\
\leq &
\stab\sum_{t_h\in\Uh}\pi(t_h|\tauhm)\times \abs{\e_{t_h}^\top\paren{\bq^{\theta}(\tauhm)-\bq^{\otheta}(\tauhm)}}\\
&+
\stab\sum_{o_h,a_h}\sum_{t_{h+1}\in\Uhp}\pi(o_h,a_h,t_{h+1}|\tauhm)\times\abs{\e_{t_{h+1}}^\top\paren{  \B^\theta_h(o_h,a_h)\bd^\theta(\tau_{h-1}) - \BB_h^{\otheta}(o_h,a_h)\bd^{\otheta}(\tau_{h-1})} },
\end{align*}
where the last inequality uses \eqref{eqn:strong-B-stability-imply-1} and \eqref{eqn:strong-B-stability-imply-2}.
Notice that $\bq^{\theta}(\tauhm)=\brac{\PP_\theta(t_h|\tauhm)}_{t_h\in\Uh}$, and hence
\begin{align*}
&~\sum_{t_h\in\Uh}\pi(t_h|\tauhm)\times \abs{\e_{t_h}^\top\paren{\bq^{\theta}(\tauhm)-\bq^{\otheta}(\tauhm)}}\\
=&~
\sum_{t_h\in\Uh}\pi(t_h|\tauhm)\times \abs{\PP_\theta(t_h|\tauhm)-\PP_\otheta(t_h|\tauhm)}\\
\leq&~ \DTV{ \PP_\theta^\pi(\tau_{h:H}=\cdot|\tau_{h-1}), \PP_\otheta^\pi(\tau_{h:H}=\cdot|\tau_{h-1})}.
\end{align*}
Also, by the definition of \Bpara~(cf. \cref{def:Bpara}), we have
$$
\brac{\B_h^{\theta}(o,a)\bq^{\theta}(\tau_{h-1})}(t_{h+1})=\PP_{\theta}(t_{h+1}|\tau_{h-1},o,a)\times\PP_\theta(o|\tauhm)=\PP_{\theta}(o,a,t_{h+1}|\tau_{h-1}),
$$
and therefore
\begin{align*}
    &~\sum_{o_h,a_h}\sum_{t_{h+1}\in\Uhp}\pi(o_h,a_h,t_{h+1}|\tauhm)\times\abs{\e_{t_{h+1}}^\top\paren{  \B^\theta_h(o_h,a_h)\bd^\theta(\tau_{h-1}) - \BB_h^{\otheta}(o_h,a_h)\bd^{\otheta}(\tau_{h-1})} }\\
    =&~
    \sum_{o_h,a_h}\sum_{t_{h+1}\in\Uhp}\pi(o_h,a_h,t_{h+1}|\tauhm)\times\abs{\PP_{\theta}(o_h,a_h,t_{h+1}|\tau_{h-1})-\PP_{\otheta}(o_h,a_h,t_{h+1}|\tau_{h-1}) }\\
    =&~
    \sum_{o_h,a_h}\sum_{t_{h+1}\in\Uhp}\abs{\PP_{\theta}^\pi(o_h,a_h,t_{h+1}|\tau_{h-1})-\PP_{\otheta}^\pi(o_h,a_h,t_{h+1}|\tau_{h-1}) }\\
    \leq &~
    \DTV{ \PP_\theta^\pi(\tau_{h:H}=\cdot|\tau_{h-1}), \PP_\otheta^\pi(\tau_{h:H}=\cdot|\tau_{h-1})}.
\end{align*}
Combining the inequalities above, we have already shown that
\begin{align*}
    \cesthth\leq \stab\DTV{ \PP_\theta^\pi(\tau_{h:H}=\cdot|\tau_{h-1}), \PP_\otheta^\pi(\tau_{h:H}=\cdot|\tau_{h-1})}
\end{align*}
for any step $h\in[H]$.  
Therefore, we can use that fact that $\dTV\leq \dH$ and apply \cref{lemma:Hellinger-cond} to obtain
\begin{align*}
    \EE_{\otheta}^\pi\brac{\cesthth^2}\leq& 
    \stab^2\EE_{\otheta}^\pi\brac{\DTV{ \PP_\theta^\pi(\tau_{h:H}=\cdot|\tau_{h-1}), \PP_\otheta^\pi(\tau_{h:H}=\cdot|\tau_{h-1})}^2}\\
    \leq &\stab^2\EE_{\otheta}^\pi\brac{\DH{ \PP_\theta^\pi(\tau_{h:H}=\cdot|\tau_{h-1}), \PP_\otheta^\pi(\tau_{h:H}=\cdot|\tau_{h-1})}}\leq2\stab^2\DH{\PP_\theta^\pi,\PP_\otheta^\pi}.
\end{align*}

A similar argument can also show that $(\cesththz)^2\leq \stab^2\DTV{ \PP_\theta^\pi, \PP_\otheta^\pi}^2 \leq \stab^2\DH{ \PP_\theta^\pi, \PP_\otheta^\pi}$.
\qed

\subsection{Proof of Theorem~\ref{thm:OMLE-regret}}\label{appdx:proof-OMLE}

The proof of \cref{thm:OMLE-regret} uses the following fast rate guarantee for the OMLE algorithm, which is standard (e.g.~\citet{van2000empirical,agarwal2020flambe}, and a simple proof can be found in \citep[Appendix E]{chen2022partially}).
\begin{proposition}[Guarantee of MLE]\label{thm:MLE}
Suppose that we choose $\beta\geq 2\logNt(1/T)+2\log(1/\delta)+2$ in \cref{alg:OMLE}. Then with probability at least $1-\delta$, the following holds:
\begin{enumerate}[wide, label=(\alph*)]
\item For all $k\in[K]$, $\ths\in\Theta^k$;
\item For all $k\in[K]$ and any $\theta\in\Theta^k$, it holds that
\begin{align*}
    \sum_{t=1}^{k-1} \DH{ \PP^{\pi^t}_{\theta}, \PP^{\pi^t}_{\ths} } \leq 2\beta.
\end{align*}
\end{enumerate}
\end{proposition}

We next prove \cref{thm:OMLE-regret}. We adopt the definitions of $\cesthth$ as in \cref{prop:psr-err-decomp} and abbreviate $\cEs_{k,h}=\cE^{\ths}_{\theta^k,h}$. We also condition on the success of the event in \cref{thm:MLE}. %

\paragraph{Step 1.} By \cref{thm:MLE}, it holds that $\ths\in\Theta$. Therefore, $V_{\theta^k}(\pi^k)\geq \Vs$, and by \cref{prop:psr-err-decomp}, we have
\begin{equation}\label{eqn:performance-OMLE-proof}
\begin{aligned}
\sum_{t=1}^k \paren{ \Vs-V_{\ths}(\pi^t) }
\leq& 
\sum_{t=1}^k \paren{ V_{\theta^t}(\pi^t)-V_{\ths}(\pi^t) } \leq
\sum_{t=1}^k \DTV{ \PP_{\theta^t}^{\pi^t}, \PP_{\ths}^{\pi^t} } \\
\leq& 
\sum_{t=1}^k 1\wedge \paren{\cE_{t,0}^\star(\pi^t)+\sum_{h=1}^H \EE_{\pi^t}\brac{\cEs_{t,h}(\pi^t,\tauhm)} }\\
\leq& 
\sum_{t=1}^k \paren{1\wedge\cE_{t,0}^\star(\pi^t)+\sum_{h=1}^H 1\wedge \EE_{\pi^t}\brac{\cEs_{t,h}(\pi^t,\tauhm)} },
\end{aligned}
\end{equation}
where the expectation $\EE_{\pi^t}$ is taken over $\tauhm\sim\PP^{\pi^t}_\ths$.
On the other hand, by \cref{prop:psr-err-decomp}, we have
\begin{align*}
    \EE_{\pi^t}\brac{\cEs_{k,h}(\pi^t,\tauhm)^2}\leq 2\stab^2\DH{\PP_{\theta^k}^{\pi^t},\PP_\ths^{\pi^t}}, \qquad \cEs_{k,0}(\pi^t)^2\leq \stab^2\DH{\PP_{\theta^k}^{\pi^t},\PP_\ths^{\pi^t}}.
\end{align*}
Furthermore, by \cref{thm:MLE} we have $\sum_{t=1}^{k-1} \DH{ \PP^{\pi^t}_{\theta^k}, \PP^{\pi^t}_{\ths} } \leq 2\beta.$
Therefore, combining the two equations above gives
\begin{equation}\label{eqn:squared-B-bound-in-proof-OMLE}
\sum_{t<k} \EE_{\pi^t}[\cEs_{k,h}(\pi^t,\tauhm)^2] \leq 4\stab^2\beta, \qquad \forall k \in [K], 0\leq h\leq H.
\end{equation}

\paragraph{Step 2.} We would like to bridge the performance decomposition \eqref{eqn:performance-OMLE-proof} and the squared B-errors bound \eqref{eqn:squared-B-bound-in-proof-OMLE} using the generalized $\ell_2$-Eluder argument. We consider separately the case for $h \in [H]$ and $h = 0$. 

\textbf{Case 1: $h \in [H]$.} We denote $m=m_{h+1}$ such that $\Uhp=(\cO\times\cA)^{m-1}\times\cO$, $\Omhp=(\cO\times\cA)^{m-1}$. By definition,
\begin{align*}
\cEs_{k,h}(\pi^t,\tauhm)\defeq&
    \frac12 \max_{\pi'\in\Pi_h(\pi^t)}\sum_{\tau_{h:H}}\pi'(\tau_{h:H}|\tauhm)\times \abs{\B_{H:h+1}^k(\tau_{h+1:H}) \left(\B^k_h(o_h,a_h) - \Bs_h(o_h,a_h)\right)\bds(\tau_{h-1})},\\
    =&
    \frac12 \max_{\pi'}\sum_{\tau_{h:H}}\pi'(\tau_{h+m:H}|\tau_{h+m-1})\times\pi^t(\tau_{h:h+m-1}|\tauhm) \times \abs{\B_{H:h+1}^k(\tau_{h+1:H}) \left(\B^k_h(o_h,a_h) - \Bs_h(o_h,a_h)\right)\bds(\tau_{h-1})}\\
    =&
    \frac12 
    \sum_{o_h,a_h}\sum_{\omhp\in\Omhp} \pi^t(o_h,a_h,\omhp|\tauhm)\nrmpi{\cB^k_{H:h+m}\cdot\B^k_{h+m-1:h+1}(\omhp) \left(\B^k_h(o_h,a_h) - \Bs_h(o_h,a_h)\right)\bds(\tau_{h-1})},
\end{align*}
where in the last equality we adopt the notation introduced in \eqref{eqn:B-op}.

To bridge between \eqref{eqn:performance-OMLE-proof} and \eqref{eqn:squared-B-bound-in-proof-OMLE}, we invoke the following generalized $\ell_2$-Eluder lemma, which can be obtained directly by generalizing~\citet[Proposition C.1 \& Corollary C.2]{chen2022partially} (which correspond to the special case of the following result with $N=1$).

\begin{lemma}[Generalized $\ell_2$-Eluder argument]
\label{lemma:l2-eluder}
Suppose we have a sequence of functions $\{ \efunc_{k,l}:\R^n\to \R \}_{(k,l) \in [K]\times[N]}$:
\begin{align*}
    \efunc_{k,l}(x):=\max_{r \in \cR}\sum_{j=1}^J \abs{\iprod{x}{y_{k,l,j,r}}},
\end{align*}
which is given by the family of vectors $\set{y_{k,l,j,r}}_{(k,j,r)\in[K]\times[J]\times\cR}\subset\R^n$.
Further assume that there exists $L>0$ such that $\efunc_{k,l}(x)\leq L\nrm{x}_1$. 

Consider further a sequence of vector $(x_{t,l,i})_{(t,l,i)\in[K]\times[N]\times\cI}$, satisfying the following condition
\begin{align*}
    \sum_{t = 1}^{k-1} \EE_{i\sim q_t}\brac{ \paren{\sum_{l=1}^N \efunc_{k,l}(x_{t,l,i}) }^2 } \leq \beta_k, 
    \phantom{xxxx} \forall k\in[K],
\end{align*}
and the subspace spanned by  $(x_{t,l,i})$ has dimension at most $d$.
Then it holds that 
\begin{align*}
\sum_{t=1}^k 1\wedge \EE_{i\sim q_t} \brac{ \sum_{l=1}^N \efunc_{t,l}(x_{t,l,i}) }
\leq
\sqrt{4Nd\Big(k+\sum_{t=1}^k \beta_t\Big)\log\left(1+kdL\max_i \lone{x_i}\right)},\phantom{xxxx} \forall k\in[K].
\end{align*}
\end{lemma}

We have the following three preparation steps to apply \cref{lemma:l2-eluder}. 

1. We define
\begin{align*}
    x_{t,l,i}&:=\pi^t(o_{h}^l,a_{h}^l,\omhp^l|\tauhm^i)\times\bds(\tauhm^i)\in\R^{\Uh}, \\
    y_{k,l,j,\pi}&:=\frac12 \pi(\tau_{h+m:H}^j)\times \brac{ \B_{H:h+m}^k(\tau_{h+1:H}^j)\B_{h+m-1:h+1}^k(\omhp^l) \left(\B^k_h(o_h^l,a_h^l) - \Bs_h(o_h^l,a_h^l)\right)}^\top \in \R^{\Uh},
\end{align*}
where $\{\tau_{h-1}^i \}_i$ is an ordering of all possible $\tau_{h-1} \in (\cO \times \cA)^{h-1}$, $\{\tau_{h+m:H}^j=(o_{h+m},a_{h+m},\cdots,o_H,a_H) \}_{j=1}^{n}$ is an ordering of all possible $\tau_{h+m:H}$ (and hence $n=(OA)^{H-h-m+1}$), $\{(o_{h}^l,a_{h}^l,\omhp^l) \}_{l=1}^{N}$ is an ordering of $\cO\times\cA\times\Omhp$ (and hence $N=OA\abs{\Omhp}\leq OAU_{\cT}$), $\pi$ is any policy that starts at step $h$. We then define
\begin{align*}
    f_{k,l}(x)= \max_{\pi}\sum_{j} \abs{\iprod{y_{k,l,j,\pi}}{x}}, \qquad x\in\R^{\Uh}.
\end{align*}
It follows from definition that 
$$
\cEs_{k,h}(\pi^t,\tauhm^i)=\sum_{l=1}^N \pi^t(o_{h}^l,a_{h}^l,\omhp^l|\tauhm^i)\times f_{k,l}(\bds(\tauhm^i))
=\sum_{l=1}^N f_{k,l}(x_{t,l,i}).
$$

2. By the assumption that $\theta^\star$ has PSR rank less than or equal to $d$, we have $\dim\spa(x_{t,l,i})\leq d$. Furthermore, we have $\lone{x_{t,l,i}}\leq U_A\leq U_{\cT}$ by definition. 

3. It remains to verify that $f_k$ is Lipschitz with respect to $1$-norm. Clearly,
\begin{align*}
    f_{k,l}(\bq)
    \leq&~ \frac12\brac{\nrmpi{\cBHh^k \bq}+\max_{o,a} \nrmpi{\cBHhp^k \Bs_h(o,a) \bq} }\\
    \leq&~ \frac12\brac{  \stab\lone{\bq}+ \stab  \max_{o,a} \lone{\Bs_h(o,a) \bq} }
    \leq \frac12 \stab \rb\lone{\bq}.
\end{align*}
Hence we can take $L=\frac12\stab \rb$ to ensure that $f_{k,l}(x)\leq L\nrm{x}_1$.

Therefore, applying \cref{lemma:l2-eluder} yields
\begin{align}\label{eqn:omle-proof-case2}
    \sum_{t=1}^k 1\wedge \EE_{\pi^t}\brac{\cEs_{t,h}(\pi^t,\tauhm) } \leq \bigO{\sqrt{\stab^2Nd\clog \beta k }}\leq \bigO{\sqrt{\stab^2OAU_{\cT}d\clog \beta k }}.
\end{align}
This completes case 1.

\textbf{Case 2: $h=0$.} This case follows similarly as
\begin{align}\label{eqn:omle-proof-case1}
    \sum_{t=1}^k 1\wedge \cEs_{t,0}(\pi^t) \leq \bigO{\sqrt{\stab^2OAU_{\cT}\clog \beta k }}.
\end{align} 

Combining these two cases, we obtain
\begin{align*}
\sum_{t=1}^k \paren{ \Vs-V_{\ths}(\pi^t) }
\stackrel{(i)}{\leq}&
\sum_{t=1}^k 1\wedge\cE_{t,0}^\star(\pi^t)+\sum_{h=1}^H \sum_{t=1}^k 1\wedge \EE_{\pi^t}\brac{\cEs_{t,h}(\pi^t,\tauhm)} \\
\stackrel{(ii)}{\leq}& 
\bigO{\sqrt{\stab^2OAU_{\cT}\clog \beta k }}+H\cdot \bigO{\sqrt{\stab^2OAU_{\cT}d\clog \beta k }}
\leq\bigO{\sqrt{H^2\stab^2Nd\clog \beta k }},
\end{align*}
where (i) used \eqref{eqn:performance-OMLE-proof}; (ii) used the above two cases \cref{eqn:omle-proof-case1} and \cref{eqn:omle-proof-case2}.
This completes the proof of \cref{thm:OMLE-regret}
\qed
\section{Additional discussions}

\subsection{Impossibility of a generic sample complexity in DEC + log covering number of value/policy class}
\label{appdx:discussion-dec}

A typical guarantee of DEC theory \cite{foster2021statistical, chen2022unified} asserts that for any model class $\cM$ and policy class $\Pi$, the E2D algorithm achieves
\begin{align}\label{eqn:DEC-assert}
    \EE\brac{\Regret}\leq \bigO{1}\cdot  \min_{\gamma>0} \paren{ T\cdot \dec_{\gamma}^{\rm H}(\cM) + \gamma\log\abs{\cM} }.
\end{align}
\citet{foster2021statistical} also showed that, letting $\co(\cM)$ denote the convex hull of $\cM$ (the set of all mixture models of $M\in\cM$), there is a variant of E2D that achieves
\begin{align*}
\EE\brac{\Regret}\leq \bigO{1}\cdot  \min_{\gamma>0} \paren{ T\cdot \dec_{\gamma}^{\rm H}(\co(\cM)) + \gamma\log\abs{\Pi} }.
\end{align*}
However, $\dec_{\gamma}^{\rm H}(\co(\cM))$ is typically intractable large---For example, when $\cM$ is the class of all tabular MDPs, $\dec_{\gamma}^{\rm H}(\co(\cM))$ scales exponentially in $S,H$~\citep{foster2022complexity}. Therefore, it is natural to ask the following

\begin{quote}
\textbf{Question:} Is it possible to obtain a regret upper bound that replaces the term $\log\abs{\cM}$ in \eqref{eqn:DEC-assert} by $\log\abs{\Pi}$ or $\log\abs{\cF_{\cM}}$ (where $\cF_{\cM}$ is a certain class of value functions induced by $\cM$)?
\end{quote}

The question above is of particular interest when the model class $\cM$ itself is much larger than the value class (e.g. Q-function class), for example when $\cM$ is a class of linear MDPs~\citep{jin2020provably} with a known feature $\phi(s,a)$ but unknown $\mu(s')$. Also, replacing $\log\abs{\cM}$ in \eqref{eqn:DEC-assert} by $\log\abs{\Pi}$ could be a decent improvement for specific problem classes, such as tabular MDPs in which case we can take $\Pi$ to be the class of deterministic Markov policies with $\log\abs{\Pi}=\tbO{SH}$, which is smaller than $\log\abs{\cM} = \tbO{\log\cN_{\cM}} =\tbO{S^2AH}$ by a factor of $SA$.

However, our lower bounds for revealing POMDPs---specifically our hard instance construction in~\cref{appdx:no-regret}---provides a (partially) negative answer to this question. 
For simplicity, consider the $m=2$ case, and assume $A^H\gg \poly(S,O,A,\alpha^{-1},T)$)
We have the following basic facts about our model class $\cM$.
\begin{enumerate}
    \item The structure of $\cM$ ensures that any possibly optimal policy is a deterministic action sequence (that does not depend on the history), and hence we can take $\Pi=\set{\text{determinstic action sequences}}$, with $\log{\Pi}=\tbO{H}$.
    \item The general results in~\citet{chen2022partially} shows that as long as $\cM$ is a subclass of $2$-step $\alpha$-revealing POMDPs, it holds that $\edec_\gamma(\cM) \le \tO(SA^2H^2\arev^{-2}/\gamma)$---where $\edec$ is a PAC-learning analogue of the $\dec$---which implies that $\dec_\gamma(\cM) \le \tO(\sqrt{SA^2H^2\arev^{-2}/\gamma})$~\citep{chen2022unified}.
    \item \cref{prop:no-regret-prop} states that worst-case regret within family $\cM$ for any algorithm is lower bounded by $\Om{(S\sqrt{O}A^2H\alpha^{-2})^{1/3}T^{2/3}}$.
\end{enumerate}
Note that the regret lower bound involves a ${\rm poly}(O)$ factor, which does not appear in the upper bound for the dec. This leads to the following

\begin{quote}
\textbf{Fact: } Without further structural assumptions for the problem, a regret upper bound of the form
\begin{align}\label{eqn:DEC-policy}
    \EE\brac{\Regret}\leq \bigO{1}\cdot  \min_{\gamma>0} \paren{ T\cdot \dec_{\gamma}^{\rm H}(\cM) + \gamma \log\abs{\Pi} }
\end{align}
is not achievable. 
\end{quote}
The above fact is because that if~\eqref{eqn:DEC-policy} were achievable, then combining with the aformentioned dec upper bound would result in a regret upper bound that does not scale with ${\rm poly}(O)$, contradicting the lower bound.

Similarly, if we view each POMDP $M\in\cM$ as an MDP by viewing each history $\tau_h$ as a ``mega-state'', then naturally the Q-function of $M$ is given by
\begin{align*}
    Q_M^\star(\tau_{h})=\EE_M^{\pi_M^\star}\cdbrac{\sum_{h'=1}^H r_{h'}}{\tau_{h}}, \qquad\tau_h\in(\cO\times\cA)^h, 0\leq h\leq H,
\end{align*}
where $\pi_M^\star$ is the optimal policy for $M$. For our family $\cM$, it is straightforward to check that $\log{\cQ_{\cM}}=\tbO{H}$, where $\cQ_{\cM}=\set{Q_M^\star:M\in\cM}$. Therefore, the answer to the question above is also negative if we take the value class to be such a Q-function class.

\subsection{Algorithms for hard instances of Theorem \ref{thm:multi-step-pac-demo}}
\label{appdx:discussion-algorithm}

We propose a brute-force algorithm $\fA$ to learn the class of hard instances provided in \cref{appdx:multi-step-pac} (for proving \cref{thm:multi-step-pac-demo}), which admits a PAC sample complexity $\tbO{S^{3/2}O^{1/2}A^{m}H /(\alpha^2\epsilon^2)}$. Algorithm $\fA$ contains two stages: 
\begin{enumerate}
    \item Stage 1: For each $h\in\cH, s\in\Sl, a\in\Ac, \a\in\cA^{m-1}$, the algorithm spends $N_1$ episodes on visiting $o_{h}=s$, taking actions $a_{h:h+m}=(a,\a,\acrev)$, and observing $(o_{h+m},o_{h+m+1})$. The observed $(o_{h+m},o_{h+m+1})$ should then satisfy the joint distribution 
    \begin{align*}
        \PP(\olock^j, o_i^+)=\frac{1+\sigma\epsilon\mu_{j,i}}{2KL}, \qquad \PP(\olock^j, o_i^-)=\frac{1-\sigma\epsilon\mu_{j,i}}{2KL}, \qquad \forall (j,i)\in[L]\times[K]
    \end{align*}
    if $(h,s,a,\a)=(\hs,\ss,\acs,\as_{\hs+1:\hs+m-1})$, and satisfy distribution $\Unif(\{\olock^1,\cdots,\olock^L\}\times\cO_o)$ otherwise. Using the standard uniformity testing algorithm \citep{canonne2020survey}, we can distinguish between 
    \[
    \begin{aligned}
    H_0:&~ (h,s,a,\a) = (\hs,\ss,\acs,\as_{\hs+1:\hs+m-1}), \\
    H_1:&~ (h,s,a,\a) \neq (\hs,\ss,\acs,\as_{\hs+1:\hs+m-1})
    \end{aligned}
    \]
    with high probability using $N_1=\tbO{\sqrt{KL}/(\sigma^2\epsilon^2)}$ samples for every fixed $(h,s,a,\a)$. The total sample size needed in Stage 1 is thus $\abs{\Sl}H\abs{\Ac}A^{m-1}\times \tbO{\sqrt{KL}/(\sigma^2\epsilon^2)}$. 
    \item Stage 2: Once Stage 1 is completed, the algorithm can correctly identify the parameter $(\hs,\ss,\acs,\as_{\hs+1:\hs+m-1})$ (if $M\neq 0$) or find out $M=0$. In the latter case, the algorithm can directly terminate and output the optimal policy of $M = 0$. In the former case, the algorithm needs to continue to learn the password $\as_{\hs+m:H-1}$: 
    \begin{itemize}
        \item For each $h=\hs+m,\hs+2m,\cdots$: 
        \begin{itemize}
            \item For each $\a\in\cA^{m}$, test whether $\a=\as_{h:h+m-1}$ by spending $N_1$ episodes on visiting $o_{\hs}=\ss$, taking actions $a_{\hs:h+m}=(a,\as_{\hs+1:h-1},\a,\acrev)$, and observing $(o_{h+m},o_{h+m+1})$.
            \item By the same reason as in Stage 1 and by our choice that $N_1 = \tbO{\sqrt{KL}/(\sigma^2\epsilon^2)}$, we can learn $\as_{h:h+m-1}$ with high probability, using the standard uniformity testing algorithm.
        \end{itemize}
    \end{itemize}
    Once the algorithm learns the $M = (\hs,\ss,\acs,\as)$, it terminates and outputs the optimal policy of $M$. 
    The total sample size needed in Stage 2 is at most $A^{m}H\times\tbO{\sqrt{KL}/(\sigma^2\epsilon^2)}$ many samples. 
\end{enumerate}
To summarize, the brute-force algorithm $\fA$ we construct above can learn $\cM$ with sample size
\begin{align*}
    \abs{\Sl}H\abs{\Ac}A^{m-1}\times \tbO{\frac{\sqrt{KL}}{\sigma^2\epsilon^2}}+A^{m}H\times \tbO{\frac{\sqrt{KL}}{\sigma^2\epsilon^2}}\leq \tbO{\frac{S^{3/2}O^{1/2}A^{m}H}{\alpha^2\epsilon^2}},
\end{align*}
where the bound is by our choice of $\sigma, \Sl, \Ac, K, L$.

\end{document}